\documentclass[conference]{IEEEtran}
\usepackage{times}

\usepackage[numbers]{natbib}
\usepackage{multicol}
\usepackage[bookmarks=true]{hyperref}
\usepackage{multirow}
\usepackage{etoc}
\usepackage{microtype}
\usepackage{graphicx}
\usepackage{subfigure}
\usepackage{booktabs} 
\usepackage{float}
\usepackage{hyperref}
\usepackage{amsmath}
\usepackage{amssymb}
\usepackage{amsthm}
\usepackage{graphicx}
\usepackage{float}
\newtheorem{assumption}{Assumption}
\usepackage{xcolor}
\usepackage{colortbl}
\usepackage{algorithm}
\usepackage{algpseudocode}
\let\STATE\State
\let\FOR\For
\let\ENDFOR\EndFor
\let\IF\If
\let\ELSE\Else

\let\ENDIF\EndIf

\usepackage{microtype}
\usepackage{newtxtext,newtxmath}

\usepackage{booktabs}
\usepackage{multirow}
\usepackage{enumerate}
\usepackage{makecell}
\usepackage[table]{xcolor}   
\usepackage{threeparttable}
\usepackage{amsmath,amssymb}
\newenvironment{aligned*}{\begin{aligned}}{\end{aligned}}
\usepackage{caption}
\usepackage{array}
\usepackage{tabularx}

\usepackage{amsmath}
\makeatletter
\def\maketag@@@#1{\hbox{\m@th\normalsize#1}}
\makeatother

\definecolor{citeTeal}{HTML}{1976D2} 

\hypersetup{
  colorlinks=true,
  linkcolor=citeTeal, 
  citecolor=citeTeal, 
  urlcolor=citeTeal,  
  filecolor=citeTeal  
}
\definecolor{HeadGray}{HTML}{4B5563} 

\newlength{\benchw}
\newlength{\envimgh}
\newcommand{\benchcell}[2]{%
  \begin{minipage}[t]{\benchw}
    {\raggedright #1\par}
    \vspace{0.25em}
    \centering
    \includegraphics[width=\linewidth,height=\envimgh,keepaspectratio]{#2}
  \end{minipage}
}
\definecolor{MuninnBlue}{HTML}{2563EB}
\definecolor{nord10}{HTML}{1F3A5F} 

\newcommand{\nordy}[1]{#1}
\newcolumntype{L}[1]{>{\raggedright\arraybackslash}p{#1}}
\newcolumntype{Y}{>{\raggedright\arraybackslash}X} 
\newcolumntype{C}[1]{>{\centering\arraybackslash}p{#1}}
\newcolumntype{D}[1]{>{\columncolor{MuninnBlue!6}\centering\arraybackslash}p{#1}}

\newcommand{\theadb}[1]{\textcolor{MuninnBlue}{\bfseries #1}}

\newcommand{\upar}{\textcolor{MuninnBlue}{\(\uparrow\)}}
\newcommand{\dnar}{\textcolor{MuninnBlue}{\(\downarrow\)}}

\renewcommand{\theadb}[1]{\textcolor{HeadGray}{\bfseries #1}}
\renewcommand{\upar}{\textcolor{HeadGray}{\(\uparrow\)}}
\renewcommand{\dnar}{\textcolor{HeadGray}{\(\downarrow\)}}

\DeclareRobustCommand{\iconNo}{\textcolor{HeadGray}{\ensuremath{\times}}}            
\DeclareRobustCommand{\iconYes}{\textcolor{HeadGray}{\ensuremath{\checkmark}}}       
\DeclareRobustCommand{\iconOften}{\textcolor{HeadGray}{\ensuremath{\circlearrowright}}} 
\DeclareRobustCommand{\iconMod}{\textcolor{HeadGray}{\ensuremath{\triangle}}}         
\DeclareRobustCommand{\iconHigh}{\textcolor{HeadGray}{\ensuremath{\blacktriangle}}}   
               
\newcommand{\plat}[1]{\rotatebox[origin=c]{90}{\textcolor{HeadGray}{\bfseries #1}}}

\newcommand{\sectionrow}[1]{%
  \addlinespace[0.35em]
  \rowcolor{MuninnBlue!7}\multicolumn{10}{@{}l@{}}{\theadb{#1}}\\
  \addlinespace[0.15em]
}

\definecolor{Denim}{HTML}{1F4E79} 
\colorlet{tbmcolor}{Denim}
\newcommand{\tbm}[1]{\textcolor{tbmcolor}{#1}}

\newcommand{\Muninnlogo}{%
  \raisebox{-0.15em}{\includegraphics[height=1.1em]{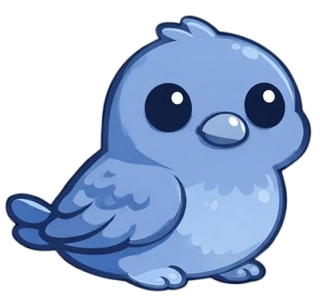}}%
  \hspace{4pt}%
}
\newcommand{\Muninlogo}{\Muninnlogo}

\newlength{\taskimgh}
\newcommand{\taskhead}[3]{%
  \begin{minipage}[t]{\linewidth}
    \centering
    {\theadb{#1}\par}
    \vspace{0.15em}
    {\textcolor{HeadGray}{\scriptsize #2}\par}
    \vspace{0.35em}
    \includegraphics[width=0.95\linewidth,height=\taskimgh,keepaspectratio]{#3}
  \end{minipage}
}

\algrenewcommand\algorithmicindent{1.0em}
\algrenewcommand\algorithmiccomment[1]{\hfill{\scriptsize\textcolor{HeadGray}{// #1}}}
\algrenewcommand\algorithmicrequire{\textcolor{HeadGray}{\bfseries Input:}}
\algrenewcommand\algorithmicensure{\textcolor{HeadGray}{\bfseries Output:}}
\algrenewcommand\algorithmicforall{\textcolor{HeadGray}{\bfseries for all}}
\algrenewcommand\algorithmicdo{\textcolor{HeadGray}{\bfseries do}}
\algrenewcommand\algorithmicend{\textcolor{HeadGray}{\bfseries end}}
\algrenewcommand\algorithmicif{\textcolor{HeadGray}{\bfseries if}}
\algrenewcommand\algorithmicthen{\textcolor{HeadGray}{\bfseries then}}
\algrenewcommand\algorithmicelse{\textcolor{HeadGray}{\bfseries else}}
\algrenewcommand\algorithmicfor{\textcolor{HeadGray}{\bfseries for}}
\algrenewcommand\algorithmicwhile{\textcolor{HeadGray}{\bfseries while}}
\algrenewcommand\algorithmicreturn{\textcolor{HeadGray}{\bfseries return}}

\begin{document}

\title{\texorpdfstring{\Muninnlogo \nordy{Muninn}: Your Trajectory Diffusion Model But Faster}{\nordy{Muninn}: Your Trajectory Diffusion Model But Faster}}




%
\author{\authorblockN{
Gokul Puthumanaillam\authorrefmark{1},
Hao Jiang\authorrefmark{1},
Ruben Hernandez\authorrefmark{1},
Jose Fuentes\authorrefmark{2},\\
Paulo Padrao\authorrefmark{3},
Leonardo Bobadilla\authorrefmark{2}, and
Melkior Ornik\authorrefmark{1}}
\authorblockA{\authorrefmark{1}University of Illinois Urbana-Champaign. Email: \texttt{\{gokulp2, haoj5, rubenjh2, mornik\}@illinois.edu}}
\authorblockA{\authorrefmark{2}Florida International University. Email: \texttt{jfuen099@fiu.edu, bobadilla@cs.fiu.edu}}
\authorblockA{\authorrefmark{3}Providence College. Email: \texttt{ppadraol@providence.edu}}
}

\maketitle

\begin{abstract}
Diffusion-based trajectory planners can synthesize rich, multimodal robot motions, but their iterative denoising makes online planning and control prohibitively slow. 
Existing accelerations either modify the sampler or compress the network--sacrificing plan quality or requiring retraining without accounting for downstream control risk. 
We address the problem of making diffusion-based trajectory planners fast enough for real-time robot use without retraining the model or sacrificing trajectory quality, and in a way that works across diverse state-space diffusion architectures.
Our key insight is that diffusion trajectory planners expose two signals we can exploit: a cheap probe of how their internal trajectory representation changes across steps, and analytic coefficients that describe how denoiser errors affect the sampler's state update. By calibrating the first signal against the second on offline runs, we obtain a per-step score that upper-bounds how far the final trajectory can deviate when we reuse a cached denoiser output, and we treat this bound as an uncertainty budget that we can spend over the denoising process.
Building on this insight, we present \Muninnlogo \nordy{Muninn}, a training-free caching wrapper that tracks this uncertainty budget during sampling and, at each diffusion step, chooses between reusing a cached denoiser output when the predicted deviation is small and recomputing the denoiser when it is not.
Across standard benchmarks spanning offline RL planning (D4RL), configuration-space motion planning, and visuomotor diffusion policies, \nordy{Muninn} delivers up to $4.6\times$ wall-clock speedups across several trajectory diffusion planners and diffusion policies by reducing denoiser evaluations, while preserving task performance and safety metrics.
\nordy{Muninn} further certifies---at a user-chosen deviation tolerance and risk level---that cached rollouts remain within a specified distance of their full-compute counterparts, and we validate these gains in real-time closed-loop navigation and manipulation hardware deployments.
Code, dataset and supplementary videos available at: \hyperlink{https://github.com/gokulp01/Muninn}{https://github.com/gokulp01/Muninn}.
\end{abstract}

\IEEEpeerreviewmaketitle

\section{Introduction}
\label{submission}
Autonomous robots are increasingly steered by diffusion-based trajectory models \cite{a85, a86} that generate entire motion sequences \cite{a38, a41}, capturing rich, multimodal futures \cite{a39} in cluttered and uncertain environments. These planners are typically run as fixed-depth reverse diffusion processes \cite{a47}: at every control cycle, the robot draws noise, rolls out dozens of denoising steps \cite{a42}, and only then extracts a trajectory. On powerful servers this is acceptable; on embedded platforms or high-rate control loops, the cost is often prohibitive \cite{a48}. Simply truncating the sampler, shrinking the network, or running at a lower planning rate can help \cite{a46}, but risks brittle failures: the model's behavior changes in opaque ways, and there is no clear link between ``saved compute" and ``lost plan quality." This paper, therefore, poses a precise question: how can we speed up diffusion-based trajectory planners by reusing internal computations, while keeping their behavior close to the original full-compute model and without retraining?
\begin{figure} [t]
    \centering
    \includegraphics[width=\linewidth]{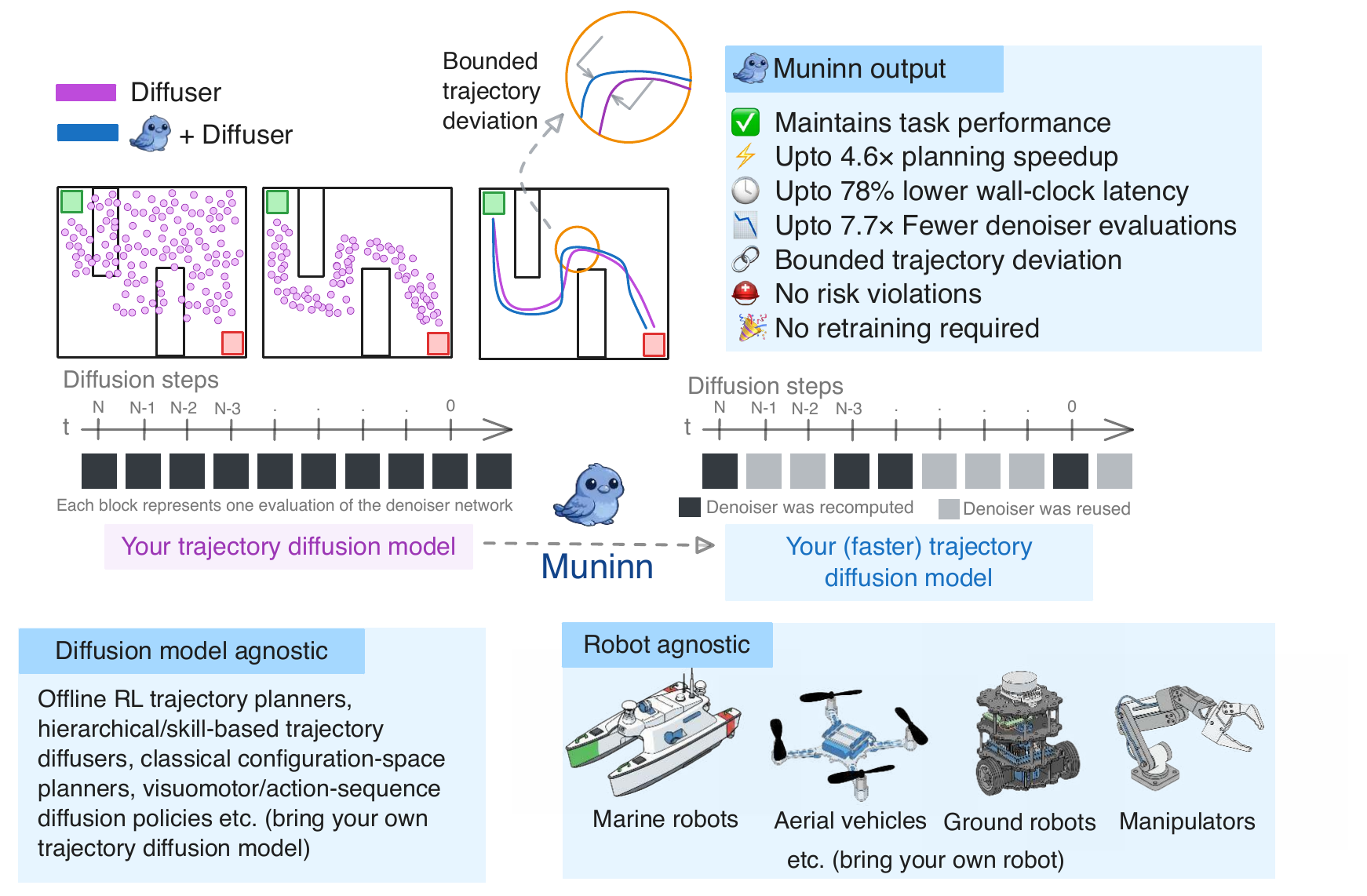}
    \caption{High-level overview of \Muninlogo Muninn applied to Diffuser. Diffuser recomputes the denoiser at every diffusion step, while Muninn wraps the same model and selectively reuses cached denoiser outputs at some steps, reducing compute while leaving the overall trajectory generation unchanged.
}
    \label{fig:teaser}
\end{figure}
A large body of work bypasses this issue by relying on model-based planning \cite{a53} and trajectory optimization \cite{a49, a51} with handcrafted costs and dynamics. These approaches can be fast \cite{a50}, but tend to struggle when dynamics, human interactions, or multi-goal behaviors induce strongly multimodal solution sets \cite{a54, a52}. Diffusion-based planners address this \cite{a55} by learning complex trajectory distributions from data, yet they usually treat the denoiser as a black box and accept its inference cost as fixed. Efforts to reduce this cost follow several paths. One direction distills the diffusion process \cite{a56} into shallower networks \cite{a57} or coarse-to-fine samplers \cite{a58}, but requires retraining for each architecture and sometimes each task. Another direction trains early-exit or policy networks \cite{a59} to decide when to stop denoising or which subnet to run; these methods can adapt computation per sample, but demand large-scale reinforcement learning or supervision \cite{a60} and tie the acceleration logic to a specific model. 

Recently, training-free accelerations \cite{a61, a66} have emerged that operate at inference time \cite{a64}. These approaches are attractive because they require no new data \cite{a62}, works on existing architectures \cite{a63}, and can deliver large wall-clock gains \cite{a65}. However, they are typically calibrated on perceptual or pixel-space metrics and developed for vision backbones. Little is known about how their approximation errors translate into trajectory deviations, and almost nothing is said about the consequences of those deviations for closed-loop control. 

We make two observations about diffusion-based planners: (i) every denoising step passes the current noisy trajectory and diffusion timestep through a relatively cheap input-processing stage to produce a representation that feeds the deeper network blocks. Changes in this representation across timesteps provide a natural, low-cost signal for how stable the trajectory is becoming, (ii) the sampler used to map denoiser outputs into new trajectories has a known analytic form: for a given schedule, we can write down exactly how a perturbation in the predicted noise at step $t$ changes the state at step $t-1$. Taken together, these two facts suggest that the planner itself contains enough structure to estimate, for each step, how risky it would be to reuse a past denoiser output instead of recomputing a new one. 
These observations motivate the central idea of this work.

\noindent \tbm{Statement of contributions: }We introduce (i) \Muninnlogo \nordy{Muninn}\footnote{``For Huginn I fear lest he come not home, / But for Muninn my care is more ..."-- \textit{Norse mythology}.
\Muninnlogo Muninn is Odin's raven of \textit{memory}.
We borrow the name as a metaphor--relying on memory to move faster, while being careful not to stray far from the path we would have taken.}, a training-free wrapper for trajectory diffusion models that uses a cheap, timestep-wise probe of the internal representation to predict how much the final trajectory would change if a cached denoiser output were reused; (ii) a trajectory-level error budget that aggregates these predictions into decisions about when to reuse past denoiser outputs and when to recompute them, trading off planning speed against fidelity; and (iii) an evaluation on several diffusion-based planners across simulated and hardware navigation and manipulation benchmarks. We also release the code, dataset and models.

\textit{Note (scope).}
Diffusion planners span many architectures \cite{a85, a86}; we focus on \emph{state-space} trajectory diffusion models that denoise future state or state--action sequences conditioned on observations and goals.
Extending \nordy{Muninn} to other families mainly requires redefining the probe; the caching framework is unchanged, though the adaptation can be nontrivial.

\section{Related work}
\noindent \textit{\tbm{Trajectory diffusion models:}}
Early learning-based trajectory planners used imitation learning \cite{a1, a2} and variational generative models \cite{a3} to clone expert demonstrations. Conditional VAEs \cite{a4}, normalizing flows \cite{a5}, and graph latent models \cite{a6} capture diverse behaviors, but emit a single trajectory and rely on delicate training heuristics. Diffusion-based planners instead learn a denoising process over trajectories, achieving SOTA coverage of multimodal motions and have been applied to visual robot control \cite{a8}, motion planning \cite{a7}, and multi-agent behavior forecasting \cite{a9}. By sampling different noise realizations, these models naturally explore distinct homotopy classes \cite{a10} and avoid hand-crafted reward engineering. Yet standard samplers still require tens of denoising steps per plan \cite{a11}, leading to second-level latencies and motivating recent work on accelerated trajectory diffusion \cite{a12}.

\noindent \textit{\tbm{Accelerating Diffusion via Model Design:}}
In this vein, architectures compress the sampling chain while preserving multimodal coverage \cite{a13, a14}. One line of work distills multi-step samplers into shallow or one-step networks \cite{a15}, or learns coarse-to-fine denoisers that jump directly to high-quality trajectories in a few iterations \cite{a16}. Other works bias the generative process with structured priors, trajectory libraries, or policy heads \cite{a17}, so that sampling is guided toward task-feasible plans and fewer refinement steps are needed \cite{a18}. A complementary direction aggregates a set of diffusion rollouts into graphs, reusing evaluations across branches to reduce latency without sacrificing robustness \cite{a19}. These approaches demonstrate that careful parameterisation can shrink diffusion horizons, bringing trajectory diffusion closer to real-time operation \cite{a21, a22}. 
However, these accelerations still rely on sequential denoising, task-specific heuristics and typically operate by hard-coding new samplers or compressing the backbone network \cite{a23}. They degrade trajectory quality, require retraining for each architecture, and offer no guarantees on how the modified sampling process affects downstream control performance \cite{a24}. 

\noindent \textit{\tbm{Adaptive Inference and Early Exiting:}}
Beyond architectural changes, a parallel line of work accelerates diffusion models via dynamic inference \cite{a25}, using learned early-exit policies \cite{a26} or step-skipping heuristics \cite{a27}to reduce test-time compute. In diffusion models, such methods typically monitor simple signals along the sampling chain to decide when to stop, skip, or jump ahead in the denoising schedule \cite{a29, a30, a31}. Related approaches in sequence modeling and control reuse cached network outputs across time or across similar inputs, treating repeated evaluations as an optimization lever rather than changing the underlying architecture \cite{a28, a33}. These strategies are usually trained in a task-specific, data-driven fashion and lack explicit bounds on how shortcutting or cache reuse perturbs the final sample \cite{a32, a34}. They provide little formal guidance on how much an accelerated diffusion run can deviate from its full-compute counterpart, and offer no guarantees on the quality of the executed trajectory \cite{a35, a36}.

In contrast to prior works, we place Muninn at the intersection of trajectory diffusion and sampling-time acceleration, focusing on making existing trajectory planners fast enough without modifying or retraining their underlying networks.

\section{Preliminaries and challenges}
A pretrained diffusion planner operates over trajectories $\mathcal{T}\subseteq\mathbb{R}^{H\times n}$ and contexts $c\in\mathcal{C}$ via a fixed reverse-diffusion sampler: $\tau^{\mathrm{full}}_{t-1}
=\Phi_t\!\big(\tau^{\mathrm{full}}_t,\;\varepsilon_\theta(\tau^{\mathrm{full}}_t,t,c),\;\xi_t\big), t=T,\dots,1,$
where $\varepsilon_\theta$ is the denoiser and $\Phi_t$ is a known sampler update (DDPM (Denoising Diffusion Probabilistic Models)/DDIM (Denoising Diffusion Implicit Models)-style; $\xi_t$ captures any sampler noise, with $\xi_t\equiv 0$ for deterministic samplers).
Let $\tau^{\mathrm{full}}_0(c,\xi)$ be the resulting full-compute trajectory (evaluating $\varepsilon_\theta$ at every step).
Each denoiser call has cost $\ell_t(c,\xi)\ge \ell_{\min}>0$, so a full planning call costs
$\mathcal{P}_c^{\mathrm{full}}(c,\xi)=\sum_{t=1}^T \ell_t(c,\xi)$.

\noindent \textit{\tbm{Challenges of caching in diffusion-based planners:}}
The repeated denoising structure makes diffusion planners an appealing target for caching: the same network is invoked at every step, and consecutive timesteps often process very similar noisy trajectories. However, the same properties that make these models expressive, make safe caching surprisingly hard!

First, denoising is globally coupled over the horizon: each evaluation acts on the full trajectory $\tau_t$, and the sampler mixes the predicted noise back into \emph{all} time steps of $\tau_{t-1}$, so a single approximation can bias the entire plan. 
Second, errors accumulate in a history-dependent way (see Fig. \ref{fig:chal}): reverse diffusion is iterative, so errors injected early are propagated and transformed by later updates; under nonlinear, contact-rich, or unstable dynamics, the same local denoiser error can have very different downstream effects depending on context and when it occurs. 
Third, control objectives are often discontinuous in trajectory space: small trajectory changes can flip outcomes from success to collision or constraint violation, so controlling only smooth surrogates can yield highly variable task impact.
\begin{figure}[H]
    \centering
    \includegraphics[width=\linewidth]{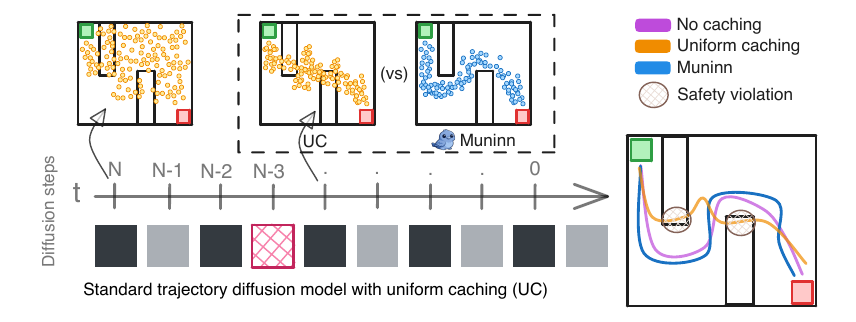}
    \caption{\textbf{Challenges in caching trajectory diffusion models.} A mistaken reuse at one timestep (red) is not local: the reused output is mixed into all subsequent updates, causing errors to accumulate and shift the entire final trajectory.}
    \label{fig:chal}
\end{figure}

Together, these effects create an all-or-nothing reuse regime: safe reuse requires confidence not only that the local approximation is small, but that its contribution remains acceptable after all subsequent sampler updates. 
Coarse strategies (e.g., fixed step skipping) oscillate between being overly conservative (limited speedup) and overly aggressive (hard-to-predict failures). 
Since practical diffusion planners vary in architecture, schedules, and conditioning, approaches that rely on retraining lose generality. 
This motivates caching mechanisms that (i) cheaply detect when reuse is risky, (ii) account for sampler-mediated error propagation, and (iii) expose interpretable knobs for trading compute against trajectory deviation.

\section{\texorpdfstring{\Muninnlogo \nordy{Muninn}}{\nordy{Muninn}}}
Motivated by these challenges, \nordy{Muninn} is a \emph{training-free} caching policy that, at each diffusion step, either recomputes $\varepsilon_\theta(\tilde{\tau}_t,t,c)$ or reuses a cached denoiser output to produce a lower-cost trajectory $\tilde{\tau}_0^{\pi}(c,\xi)$.
It exposes two user-facing knobs: a deviation tolerance $\eta_{\mathrm{traj}}$ and risk level $\alpha$, and aims to minimize expected planning
 subject to $\mathbb{P}\!\left(d(\tau^{\mathrm{full}}_0,\tilde{\tau}^{\pi}_0)>\eta_{\mathrm{traj}}\right)\le\alpha$.

\subsection{Trajectory diffusion and sampler sensitivity}
\label{subsec:sensitivity}

We begin by making precise how local errors in the denoiser's output at each diffusion step affect the final trajectory. 

\noindent \textit{\tbm{Reverse diffusion as a paired process:}}
Recall that $\tau^{\mathrm{full}}_0(c,\xi)$ is produced by evaluating the denoiser at every step.
Caching policies $\pi$ modify this process by \emph{reusing} denoiser outputs at selected timesteps instead of recomputing them.
To isolate the effect of reuse, we analyze a paired run: the full and cached chains share the same initial sample $\tau_T$ and the sampler noise sequence $\xi=(\xi_T,\dots,\xi_1)$, and differ only in the denoiser outputs $\Phi_t$.
Let $\tilde{\tau}_t$ denote the cached trajectory at step $t$ and let $\tilde{\varepsilon}_t$ be the noise prediction actually used at step $t$ (recomputed or reused). The cached chain evolves as
$\tilde{\tau}_{t-1}\;:=\;\Phi_t\big(\tilde{\tau}_t,\, \tilde{\varepsilon}_t,\, \xi_t\big)$.
\nordy{Muninn} will choose $\tilde{\varepsilon}_t$ by \textit{recomputing} $\tilde{\varepsilon}_t \;=\; \varepsilon_\theta(\tilde{\tau}_t, t, c)$ or by \textit{reusing} a cached prediction.

\noindent \textit{\tbm{Error variables:}}
We define two per-step errors: the trajectory drift $\Delta_t := \tau_t^{\mathrm{full}}-\tilde{\tau_t}$ and the denoiser mismatch $e_t \;:=\; \varepsilon_\theta(\tilde{\tau}_t,\, t,\, c)\;-\;\tilde{\varepsilon}_t.$, capturing divergence in trajectory space and in the noise prediction, respectively.
Unless stated otherwise, $\|\cdot\|$ denotes the Frobenius norm on $\mathbb{R}^{H\times n}$.
Our only assumption on the sampler is a {local Lipschitz property}:

\begin{assumption}[Sampler Lipschitzness]
For timestep $t$, there exist nonnegative constants $K_t$ and $L_t'$ such that, for any trajectories $\tau,\tau' \in \mathcal{T}$, noise predictions $\varepsilon,\varepsilon'$, and fixed sampler noise $\xi_t$,
    $\big\lVert \Phi_t(\tau,\varepsilon,\xi_t) - \Phi_t(\tau',\varepsilon',\xi_t) \big\rVert
    \;\le\;
    K_t \lVert \tau - \tau' \rVert + L_t' \lVert \varepsilon - \varepsilon' \rVert.$
\end{assumption}

For DDPM/DDIM-style samplers, $\Phi_t$ is affine in $\tau_t$ and $\varepsilon_t$, so $K_t$ and $L_t'$ follow directly from the update coefficients. Conceptually, $K_t$ scales trajectory differences, while $L_t'$ scales sensitivity to noise-prediction error.
Applying Assumption~1 to the full-compute and cached updates at step $t$ gives
\begin{equation}
\footnotesize
\begin{aligned}
    \lVert \Delta_{t-1} \rVert
    = \big\lVert \Phi_t(\tau_t^{\mathrm{full}}, \varepsilon_\theta(\tau_t^{\mathrm{full}},t,c), \xi_t)
      - \Phi_t(\tilde{\tau}_t, \tilde{\varepsilon}_t, \xi_t) \big\rVert \\
    \le K_t \lVert \tau_t^{\mathrm{full}} - \tilde{\tau}_t \rVert
      + L_t' \big\lVert \varepsilon_\theta(\tau_t^{\mathrm{full}},t,c) - \tilde{\varepsilon}_t \big\rVert = K_t \lVert \Delta_t \rVert + L_t' \lVert e_t \rVert.
\end{aligned}
\label{eq:local-recursion}
\end{equation}

This recursion decomposes the trajectory error at step $t-1$ into a part inherited from the existing trajectory mismatch $\Delta_t$ and a new contribution from the denoiser error $e_t$.

\begin{figure} [H]
    \centering
    \includegraphics[width=0.8\linewidth]{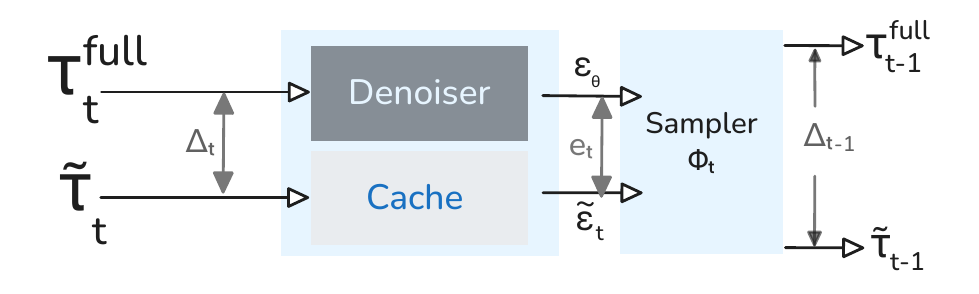}
\caption{\textbf{Paired reverse-diffusion step.}
$e_t$ is the denoiser error from reusing $\tilde{\varepsilon}_t$ and $\Delta_t$ is the induced trajectory mismatch; a Lipschitz bound splits $\Delta_{t-1}$ into propagated mismatch and injected error.}
\label{fig:paired}
\end{figure}

\noindent \textit{\tbm{Unrolling the recursion:}}
Both processes are initialized from the same noisy trajectory, so $\Delta_T = 0$.
Repeatedly applying the local bound~\eqref{eq:local-recursion} from $t = T$ down to $t = 1$ yields a closed-form bound on $\Delta_0$.
For $t = T$ we have $\lVert \Delta_{T-1} \rVert \;\le\; K_T \lVert \Delta_T \rVert + L_T' \lVert e_T \rVert \;=\; L_T' \lVert e_T \rVert$.
Continuing, a simple induction shows,
\begin{equation}
    \lVert \Delta_t \rVert
    \;\le\;
    \sum_{s=t+1}^{T}
    (
        L_s' \prod_{j=t+1}^{s-1} K_j
    ) \lVert e_s \rVert \quad \text{for } t \in \{0,\dots,T-1\}.
    \label{eq:deltat-bound}
\end{equation}
Defining the {pathwise sensitivity coefficients}, $L_s \;:=\; L_s' \prod_{j=1}^{s-1} K_j$, we can rewrite~\eqref{eq:deltat-bound} at $t=0$ compactly as
\begin{equation}
    \lVert \Delta_0 \rVert
    \;\le\;
    \sum_{t=1}^{T} L_t \lVert e_t \rVert.
\label{eq:delta0-bound}
\end{equation}
Each step's denoiser error $|e_t|$ contributes additively to the final trajectory deviation, weighted by $L_t$, capturing how strongly errors at $t$ are amplified through subsequent diffusion updates.

\noindent \textit{\tbm{From norm difference to trajectory deviation:}}
The guarantee in our problem statement is expressed in terms of a user-chosen trajectory metric $d : \mathcal{T} \times \mathcal{T} \to \mathbb{R}_{\ge 0}$.
We assume that $d$ is Lipschitz-equivalent to the norm $\lVert\cdot\rVert$ in the sense that:

\begin{assumption}[Metric compatibility]
There exists a constant $\Gamma \ge 1$ such that for all $\tau,\tau' \in \mathcal{T}$, $d(\tau,\tau') \;\le\; \Gamma \lVert \tau - \tau' \rVert$.
\end{assumption}
Combining Assumption~2 with~\eqref{eq:delta0-bound} yields:
\begin{equation}
    d\big(\tau_0^{\mathrm{full}}(c,\xi),\, \tilde{\tau}_0(c,\xi)\big)
    \;\le\;
    \Gamma \lVert \Delta_0 \rVert
    \;\le\;
    \sum_{t=1}^{T} \Gamma L_t \lVert e_t \rVert.
\label{eq:traj-bound}
\end{equation}
Equation~\eqref{eq:traj-bound} is the {key trajectory deviation bound}. It exposes that an error $e_t$ injected at step $t$ is scaled by $L_t$, which depends on all later sampler sensitivities $\{K_j\}_{j < t}$.
Early steps can therefore have a larger impact than late steps. In our experiments we use $d(\tau,\tau') := \frac{1}{\sqrt{H}}\|\tau-\tau'\|_F$, with $\Gamma=\frac{1}{\sqrt{H}}$.

We define the {per-step trajectory cost} associated with a $\lVert e_t \rVert$ as $c_t(e_t) := \Gamma L_t \lVert e_t \rVert$.
Then~\eqref{eq:traj-bound} simply states that the total trajectory deviation is bounded by the sum of per-step costs,
\begin{equation}
    d\big(\tau_0^{\mathrm{full}}(c,\xi),\, \tilde{\tau}_0(c,\xi)\big)
    \;\le\;
    \sum_{t=1}^{T} c_t(e_t).
\label{eq:budget-form}
\end{equation}

\nordy{Muninn} uses~\eqref{eq:budget-form} as a budget rule: if each reuse step has a high-probability upper bound on $c_t(e_t)$ and the summed bounds stay within $\eta_{\mathrm{traj}}$, then the final cached trajectory remains within $\eta_{\mathrm{traj}}$ of the full-compute one with the same high probability.


Now, the challenge is that, during deployment, the actual denoiser error $\lVert e_t \rVert$ is {unknown} for a reuse decision.

\subsection{Probe features and per-step error scores}
\label{subsec:probe}

The sensitivity bound~\eqref{eq:traj-bound} quantifies how costly a step-$t$ error $\lVert e_t \rVert$ would be, but not how large that error would be under reuse. In order to do that, we introduce (i) a lightweight \emph{probe} $F_t$, and (ii) \emph{per-step score} $s_t$. 
Fig. \ref{fig:probe} summarizes this subsection.

\noindent \textit{\tbm{Probe features:}}
At step $t$, the denoiser $\varepsilon_\theta(\cdot, t, c)$ is evaluated on the current noisy $\tau_t$ and context $c$.
In almost all modern architectures, this evaluation can be decomposed into: (i) an {input processing or ``stem'' stage} that maps $(\tau_t, t, c)$ into an intermediate representation $h_t$, and (ii) a deeper, {expensive ``core''} that refines $h_t$ into the final noise prediction $\varepsilon_\theta(\tau_t,t,c)$.

We exploit this decomposition by defining a probe feature $F_t$ that depends only on the cheaper part of the computation and is therefore inexpensive to compute at every step.

Formally, we posit a probe function $\Psi : \mathcal{T} \times \{1,\dots,T\} \times \mathcal{C} \to \mathbb{R}^{d_F}$ and define $F_t := \Psi(\tilde{\tau}_t, t, c)$.
We require $\Psi$ to satisfy two design goals: (i) The cost of evaluating $\Psi$ at a single step should be negligible compared to a full pass through $\varepsilon_\theta$.
(ii) {Architecture-agnostic:} Any trajectory diffusion model provides an intermediate representation between its input and its final noise prediction.
    By defining $\Psi$ in terms of this representation, we ensure that \nordy{Muninn} can be applied to a wide range of state-space diffusion planners without architectural changes.

\noindent \textit{\tbm{Per-step error scores from probe dynamics:}}
Intuitively, reuse is safe when the denoiser's output would not have changed much when recomputed.
We expect this to happen when $F_t$ has stabilized: if the probe at step $t$ is similar to the nearby steps, the model is likely to produce similar noise predictions, and using a cached output should incur a small error $\lVert e_t \rVert$.

To capture this intuition, \nordy{Muninn} maps the sequence of probe features $\{F_t\}$ into a scalar \emph{error score} $s_t \in \mathbb{R}_{\ge 0}$ at each step $t$.
This score should be: (i) {cheap to compute} from probe features we have, (ii) {monotonically informative} in the sense that smaller scores tend to correspond to smaller denoiser errors.

Since reverse diffusion runs from $t=T$ to $1$, the reuse decision at step $t$ can use both the current probe $F_t$ and the previously computed $F_{t+1}$. A effective choice of score is the magnitude of the probe change between these two steps:
%
%
\begin{equation}
    s_t := \phi(F_t, F_{t+1}, t)
    = \frac{\lVert F_t - F_{t+1} \rVert_1}{\lVert F_{t+1} \rVert_1 + \omega},
\quad t \in \{1,\dots,T-1\}, 
\label{eq:score-definition}
\end{equation}
where $\omega > 0$.
This design reflects two intuitions: (i) when $h_t$ stabilizes, $F_t$ changes slowly and the denoiser output changes little, so reuse yields small $\|e_t\|$; (ii) as sensitivity varies with timestep (via $L_t$), calibrating score-to-error bounds per $t$ adapts reuse decisions to high-impact regions of the diffusion chain.

\begin{figure}[H]
    \centering
    \includegraphics[width=\linewidth,trim=10 10 10 10,clip]{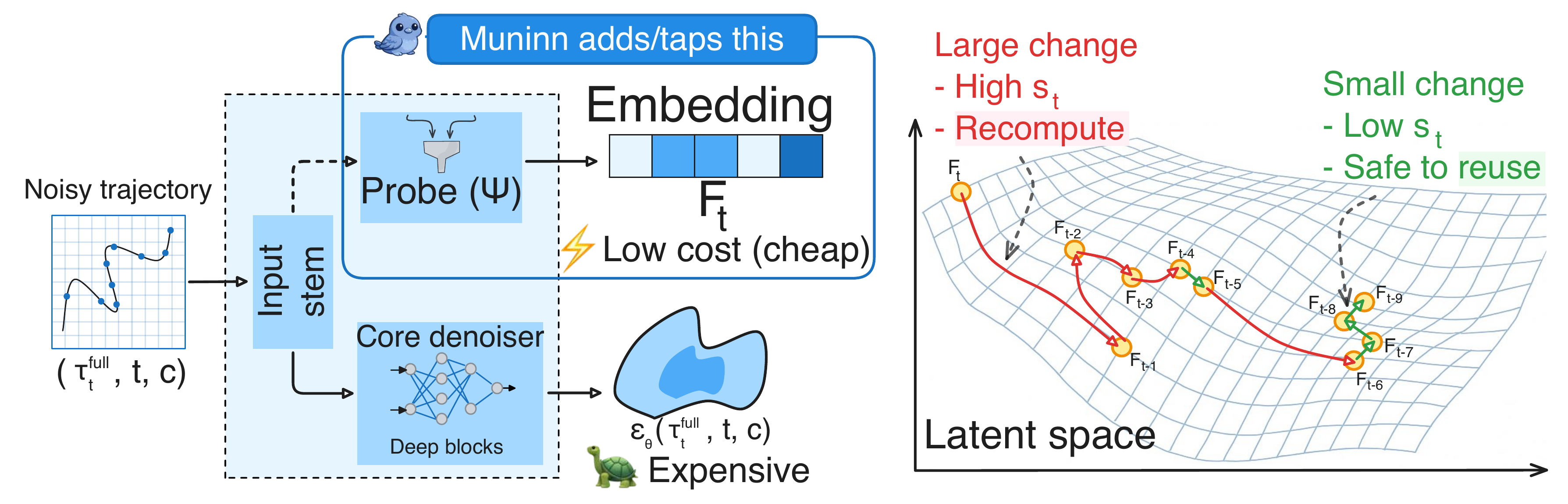}
\caption{\textbf{Probe features and per-step error scores.}
\textbf{(Left)} A denoiser evaluation decomposes into a cheap stem and an expensive core; \nordy{Muninn} runs only the stem to compute a lightweight probe.
\textbf{(Right)} Large step-to-step changes ($s_t$) trigger recomputation, while small $s_t$ indicate stability and make reuse safe.}
    \label{fig:probe}
\end{figure}


\subsection{Conformal calibration of reuse error}
\label{subsec:conformal}

Section~\ref{subsec:sensitivity} bounds trajectory deviation by a sum of per-step costs.
To use this bound online, \nordy{Muninn} needs a high-probability upper bound on the (unknown) reuse error $\|e_t\|$ at each step.
We now show how to obtain these bounds.


\noindent \textit{\tbm{Calibration data:}}
Recall that the trajectory deviation bound~\eqref{eq:traj-bound} depends only on its {magnitude} $\lVert e_t \rVert$.
We therefore define the scalar error $\epsilon_t := \lVert e_t \rVert \;\in\; \mathbb{R}_{\ge 0}$ and aim to predict $\epsilon_t$ from the score $s_t$ introduced in~\eqref{eq:score-definition}.

At deployment, $\tilde{\varepsilon}_t$ is the noise prediction actually {reused} by \nordy{Muninn} at step $t$, which depends on the entire history of recompute vs.\ reuse decisions.
For calibration, we fix a policy-agnostic reuse rule that defines a {potential reuse output} at every step, independent of the budget.
We define a deterministic map that, specifies which earlier denoiser output would be reused if we chose to reuse at $t$.
For e.g., we can always reuse the prediction from the previous step that evaluated $\varepsilon_\theta$.

Given this potential reuse rule, we build a calibration set by running the full-compute chain alongside a paired {ghost} reuse chain on $N$ i.i.d.\ episodes from the deployment distribution. 
For each episode $i$ and each $t\in\mathcal{T}_{\text{cache}}$, we record (i) the probe-based score $s_t^{(i)}$ computed from the ghost chain, and (ii) the corresponding {potential reuse error} magnitude $\epsilon_t^{(i)}$.
The resulting calibration dataset is the collection of score--error pairs $\mathcal{S}_{\mathrm{cal}}
\;:=\; \big\{(s_t^{(i)}, \epsilon_t^{(i)}) : i = 1,\dots,N,\; t \in \mathcal{T}_{\text{cache}}\big\}$, where $\mathcal{T}_{\text{cache}}$ is the set of timesteps at which reuse is allowed.
Appendix~\ref{calib_data} and Figure~\ref{fig:calib} provide the full construction details.



Since the reuse pattern is fixed and independent of $s_t$, we can view each $(s_t^{(i)}, \epsilon_t^{(i)})$ as a draw from a common, unknown joint distribution over scores and potential reuse errors induced by the base planner and environment.
Under the {exchangeability assumption}---calibration and deployment episodes are sampled i.i.d. from $\mathcal{D}$---this is the setting where conformal prediction provides finite-sample coverage guarantees. 

\noindent \textit{\tbm{Split-conformal regression:}}
Given $\mathcal{S}_{\mathrm{cal}}$, we construct a function
$U:\mathbb{R}_{\ge 0}\to\mathbb{R}_{\ge 0}$ that maps each score $s$ to a high-probability upper bound on the corresponding potential reuse error $\epsilon$.
Concretely, we fit a regression model $m(s)$ on a training split of $\mathcal{S}_{\mathrm{cal}}$ and apply split conformal prediction on a held-out split to obtain an additive residual quantile $q_{1-\alpha_{\text{step}}}$.
This results in the bound $U(s)\;:=\;\max\{m(s)+q_{1-\alpha_{\text{step}}},\,0\}$; Appendix~\ref{calib_data} provides the full construction.

By split conformal prediction, assuming exchangeability between calibration and deployment, we obtain per-step coverage: for a new step with score $s_t$ and potential reuse error $\epsilon_t$,
\begin{equation}
\mathbb{P}\big(\epsilon_t \le U(s_t)\big)\;\ge\;1-\alpha_{\text{step}}.
\label{eq:per-step-coverage}
\end{equation}

This guarantee is \emph{distribution-free}: it holds for any joint distribution of $(s_t,\epsilon_t)$ as long as calibration and deployment samples are exchangeable.
The regressor $m$ only tightens $U$; even if $m$ is misspecified,~\eqref{eq:per-step-coverage} holds by construction.


\noindent \textit{\tbm{From per-step coverage to global risk:}}
The per-step coverage~\eqref{eq:per-step-coverage} applies to the latent potential reuse error $\epsilon_t$ at any timestep $t$.
Combining $c_t(e_t)$ with the upper bound $U_t(s_t)$ gives, for any step $t$, $\Gamma L_t \epsilon_t \;\le\; \Gamma L_t U_t(s_t)
    \quad\text{whenever}\quad \epsilon_t \le U_t(s_t)$.
If we define the {upper-bounded per-step cost} at score $s_t$ as $\hat{c}_t(s_t) := \Gamma L_t U_t(s_t),$
then with probability at least $1 - \alpha_{\text{step}}$,
\begin{equation}
    c_t(e_t) = \Gamma L_t \epsilon_t \;\le\; \hat{c}_t(s_t).
\end{equation}
Crucially, the event $\{\epsilon_t > U_t(s_t)\}$ depends only on the underlying data-generating process and the calibration procedure; it is {independent of the caching policy} we will deploy later.

\begin{figure}[t]
    \centering
    \includegraphics[width=0.8\linewidth,trim=0pt 0pt 0pt 0pt,clip]{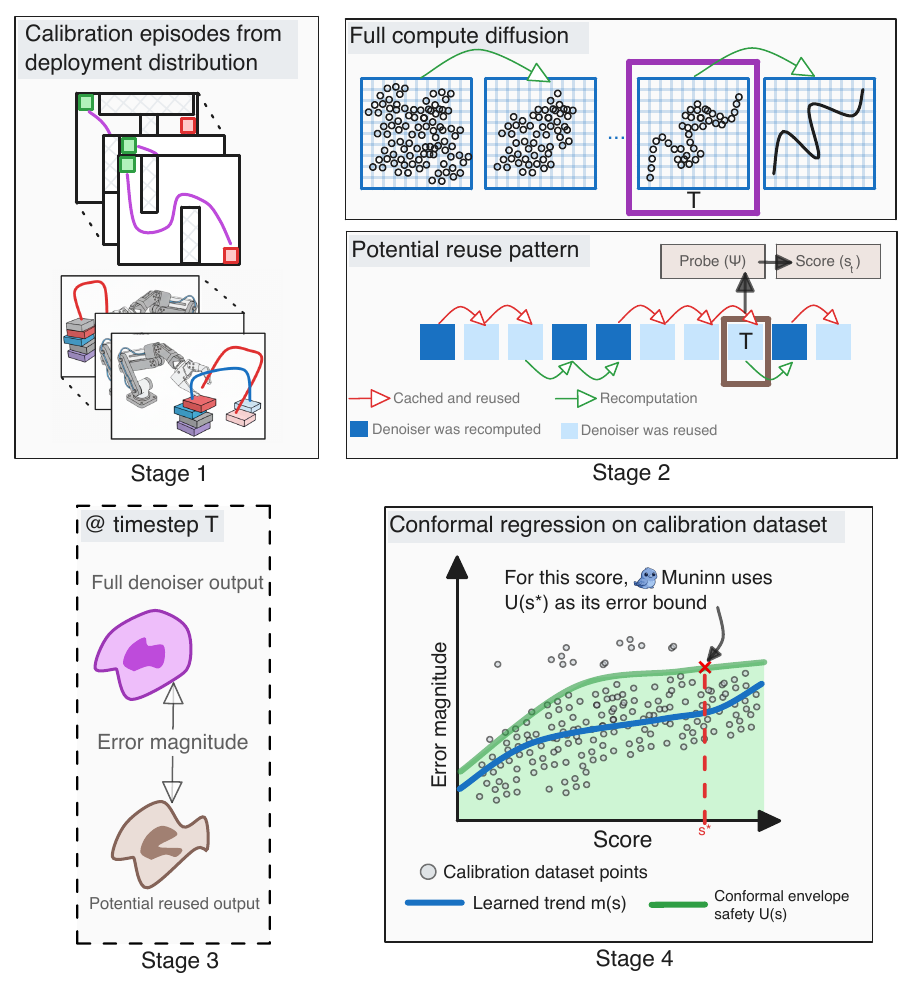}
\caption{
\textbf{(Stages 1--3)} From offline rollouts, we build a calibration set of score--error pairs by running the full planner with a ghost reuse chain.
\textbf{(Stage 4)} A regressor $m(s)$ models the typical score--error relationship, and split conformal prediction yields $U(s)$.
At test time, \nordy{Muninn} maps each score to $U(s_t)$, converts it to a trajectory-level cost, and spends it from the budget.}
    \label{fig:calib}
\end{figure}



\subsection{\texorpdfstring{\Muninnlogo \nordy{Muninn} caching policy}{\nordy{Muninn} caching policy}}
\label{subsec:policy}

\nordy{Muninn} uses the per-step bounds $\hat{c}_t(s_t)$ to decide online, per sample, whether to reuse or recompute.
The idea is to interpret~\eqref{eq:budget-form} as a budget constraint: user specifies a maximum allowable $\eta_{\mathrm{traj}}$ which is spent across reverse-diffusion steps.


\noindent \textit{\tbm{Budgeted reuse rule:}}
At deployment time, the true error $\lVert e_t \rVert$ at $t$ is unknown.
Let $\mathcal{T}_{\text{cache}} \subseteq \{1,\dots,T\}$ denote the set of timesteps where reuse is permitted.
We maintain a remaining budget $B_{\text{rem}}$ and, for clarity, we summarize the algorithm below.

\begin{algorithm}[H]
    \caption{\Muninnlogo \nordy{Muninn} caching policy for one planning call}
    \label{alg:Muninn}
    \vspace{0.35em}

    \scriptsize
    \color{HeadGray}

    \begin{algorithmic}[1]
        \STATE Sample initial noisy trajectory $\tau_T \sim p_T$.
        \STATE Initialize: (i) $B_{\mathrm{rem}} \leftarrow \eta_{\mathrm{traj}}$, (ii) $\tilde{\tau}_T \leftarrow \tau_T$.
        \FOR{$t = T, T-1, \dots, 1$}
            \STATE Compute probe $F_t = \Psi(\tilde{\tau}_t, t, c)$.
            \IF{$t < T$}
                \STATE Retrieve previous probe $F_{t+1}$ // stored at step $t+1$.
            \STATE Compute score $s_t = \phi(F_t, F_{t+1}, t)$ 
            \STATE Compute upper-bounded cost $\hat{c}_t(s_t) = \Gamma L_t U_t(s_t)$.
            \ENDIF
\IF{$t = T$ \textbf{or} $t \notin \mathcal{T}_{\text{cache}}$ \textbf{or} ($t < T$ \textbf{and} $\hat{c}_t(s_t) > B_{\mathrm{rem}}$)}
                \STATE \textbf{Recompute:}
                \STATE \hspace{0.5em} $\varepsilon_t^{\text{new}} \leftarrow \varepsilon_\theta(\tilde{\tau}_t, t, c)$
                \STATE \hspace{0.5em} Update cache with $(t, \varepsilon_t^{\text{new}})$ as most recent recompute.
                \STATE \hspace{0.5em} $\tilde{\varepsilon}_t \leftarrow \varepsilon_t^{\text{new}}$
            \ELSE
                \STATE \textbf{Reuse:}
                \STATE \hspace{0.5em} Set $\tilde{\varepsilon}_t$ to the cached denoiser output specified by the potential reuse rule.
                \STATE \hspace{0.5em} $B_{\mathrm{rem}} \leftarrow B_{\mathrm{rem}} - \hat{c}_t(s_t)$
            \ENDIF
            \STATE Apply sampler update: $\tilde{\tau}_{t-1} \leftarrow \Phi_t(\tilde{\tau}_t, \tilde{\varepsilon}_t, \xi_t)$
            \STATE Store $F_t$ // for use at step $t-1$.
        \ENDFOR
        \STATE \textbf{return} $\tilde{\tau}_0$
    \end{algorithmic}
\end{algorithm}

This policy is: (i) {budget-safe by construction}: enforcing $B_{\text{rem}} \ge 0$ ensures that the sum of upper-bounded costs over reuse steps satisfies $\sum_{t \in \mathcal{R}} \hat{c}_t(s_t) \;\le\; \eta_{\mathrm{traj}}$;
(ii) {sample-adaptive}: reuse locations and counts depend on the realized scores $s_t$ (and thus on the context, noise, and trajectory).

\noindent \textit{\tbm{Global risk guarantee:}}
Under the conformal calibration guarantee~\eqref{eq:per-step-coverage}, for each $t$ in a planning episode we have $\mathbb{P}\big( \epsilon_t \le U_t(s_t) \big) \;\ge\; 1 - \alpha_{\text{step}}$. Define the failure event at step $t$ as
$\mathcal{F}_t := \{\epsilon_t > U_t(s_t)\}$.
By construction of $U_t$,
\begin{equation}
    \mathbb{P}(\mathcal{F}_t) \;\le\; \alpha_{\text{step}}
    \quad\text{for each } t.
    \label{eq:Ft-bound}
\end{equation}

Recall that $\epsilon_t$ is defined w.r.t. the potential reuse rule, and at reuse, \nordy{Muninn} uses the same rule to choose which cached output to apply.
Thus, whenever \nordy{Muninn} reuses at step $t$, the realized denoiser error magnitude $\lVert e_t \rVert$ coincides with the latent potential reuse error $\epsilon_t$ for that step.

Now consider a single planning episode.
If {none} of the failure events $\mathcal{F}_t$ occur at any timestep in $\mathcal{T}_{\text{cache}}$, then for all reuse steps $t \in \mathcal{R}$ we have $c_t(e_t) = \Gamma L_t \epsilon_t \le \Gamma L_t U_t(s_t) = \hat{c}_t(s_t)$,
and thus the total trajectory deviation satisfies
{
\makeatletter
\def\maketag@@@#1{\hbox{\m@th\normalsize#1}}
\makeatother

\begin{equation}
{\footnotesize
\begin{aligned}
    d\big(\tau_0^{\mathrm{full}}, \tilde{\tau}_0\big)
    &\le \sum_{t \in \mathcal{R}} c_t(e_t)
    &&\text{(from (5))} \\
    &\le \sum_{t \in \mathcal{R}} \hat{c}_t(s_t)
    &&\text{(since no failure event occurs)} \\
    &\le \eta_{\mathrm{traj}}
    &&\text{(by the budget rule)}.
\end{aligned}}
\label{eq:good-episode}
\end{equation}
}
Therefore, the only way for the trajectory deviation to exceed $\eta_{\mathrm{traj}}$ is if at least one failure event occurs:
    $\big\{ d(\tau_0^{\mathrm{full}}, \tilde{\tau}_0) > \eta_{\mathrm{traj}} \big\}
    \;\subseteq\;
    \bigcup_{t \in \mathcal{T}_{\text{cache}}} \mathcal{F}_t.$
Taking probabilities and applying a union bound over timesteps yields
\begin{equation}
\begin{aligned}
    \mathbb{P}\big( d(\tau_0^{\mathrm{full}}, \tilde{\tau}_0) > \eta_{\mathrm{traj}} \big)
    &\le
    \mathbb{P}\Big( \bigcup_{t \in \mathcal{T}_{\text{cache}}} \mathcal{F}_t \Big) \\
    &\le
    \sum_{t \in \mathcal{T}_{\text{cache}}} \mathbb{P}(\mathcal{F}_t) 
    &\le
    |\mathcal{T}_{\text{cache}}| \,\alpha_{\text{step}}.
\end{aligned}
\label{eq:global-risk-bound}
\end{equation}
If we choose $\alpha_{\text{step}} := \frac{\alpha}{|\mathcal{T}_{\text{cache}}|}$, then~\eqref{eq:global-risk-bound} simplifies to
\begin{equation}
    \mathbb{P}\big( d(\tau_0^{\mathrm{full}}, \tilde{\tau}_0) > \eta_{\mathrm{traj}} \big)
    \;\le\;
    \alpha,
    \label{eq:global-risk-final}
\end{equation}
which is exactly the risk guarantee from the problem statement.


This guarantee does not require timestep independence: it follows from the per-step bounds~\eqref{eq:Ft-bound} and a union bound.
It holds for any adaptive reuse set $\mathcal{R}$, provided the reuse rule used to define $\epsilon_t$ in calibration matches the rule used at deployment.

\subsection{Practical considerations and engineering choices}
In practice, \nordy{Muninn} is implemented with a few simple engineering choices: (i) \textit{forbidden regions}---we disallow reuse in a short prefix (large $t$, highly noisy and sensitivity-amplifying) and suffix (small $t$, where small changes in $\tau_t$ can directly affect executed motion); (ii) \textit{batched sampling}---when multiple trajectories are denoised in parallel, we maintain an independent remaining budget $B_{\mathrm{rem}}$ per trajectory and apply the same reuse rule independently, so the guarantee~\eqref{eq:per-step-coverage} and bound~\eqref{eq:global-risk-final} apply per trajectory; (iii) \textit{model instantiations}---for Transformer trajectory denoisers we take $\Psi$ as a short prefix of attention/MLP blocks and mean-pool tokens to obtain $F_t\in\mathbb{R}^{d_F}$, while for MLP denoisers we use the penultimate hidden layer as $F_t$, in all cases computing $(F_t,s_t)$ \emph{before} the reuse decision at step $t$ and making them available during both calibration (to label $\|e_t\|$) and deployment (when $\|e_t\|$ is unknown).


\section{Experiments}

We evaluate \nordy{Muninn} on two trajectory-diffusion families covering the dominant uses of diffusion in robotics.

\noindent \tbm{Trajectory diffusion planners (offline RL and motion planning)} \cite{a37, a87, a7, a88, a89, a43, a91, a92} use diffusion as a planner over full state-action or configuration-space trajectories, implemented as an iterative reverse-diffusion process (optionally with guidance).

\noindent \tbm{Diffusion policies (imitation and visuomotor control)} \cite{a8, a90} use diffusion as a policy that generates short-horizon action or pose segments, conditioned on high-dimensional observations, and executed repeatedly in a receding-horizon control loop.
%
%
%
%



\subsection{Quantitative Results: Simulation and Benchmarks}  
\label{subsec:envs}

We evaluate \nordy{Muninn} on trajectory centric benchmarks and report standard metrics as defined in the prior literature for each benchmark.
See Appendix \ref{app:envs} for details on the environments, datasets, training, and metrics (Appendix \ref{subsec:metrics}). All experiments are run on an NVIDIA A10 GPU (24GB VRAM) and averaged on 150 previously unseen episodes.

\tbm{Offline RL / trajectory planning.}  
We first consider the D4RL offline RL suite \cite{a96} for continuous control, where models act as trajectory optimizers/conditional policies over state-action.

\begin{table}[H]
  \centering
\scriptsize
\setlength{\tabcolsep}{2pt}      
  \renewcommand{\arraystretch}{1.12}
  \setlength{\benchw}{0.15\columnwidth}

  \begin{threeparttable}
  \begin{tabularx}{\columnwidth}{@{}%
    L{0.15\columnwidth}%
    Y%
    C{0.073\columnwidth}%
    D{0.073\columnwidth}%
    C{0.073\columnwidth}%
    D{0.073\columnwidth}%
    C{0.050\columnwidth}%
    D{0.050\columnwidth}%
    C{0.078\columnwidth}%
    C{0.073\columnwidth}%
  @{}}

    \arrayrulecolor{HeadGray}\toprule\arrayrulecolor{black}

    \rowcolor{MuninnBlue!10}
    \theadb{Benchmark} & \theadb{Model} &
    \multicolumn{2}{c}{\theadb{Task metr.} \upar} &
    \multicolumn{2}{c}{\theadb{Latency} \dnar} &
    \multicolumn{2}{c}{\makecell[c]{\theadb{\#Evals/t}\hspace{1.4pt}\dnar}} &
    \makecell[c]{\textcolor{HeadGray}{$\mathbb{E}[d]$}\dnar} &
    \makecell[c]{\textcolor{HeadGray}{$\hat{p}_{\mathrm{viol}}$}\dnar} \\

    \rowcolor{MuninnBlue!6}
    & &
    \bfseries Full & +\Muninnlogo &
    \bfseries Full & +\Muninnlogo &
    \bfseries Full & +\Muninnlogo & 
    & \\

    \midrule

    \sectionrow{MuJoCo locomotion (D4RL)}
    \multirow[t]{4}{*}{\benchcell{HalfCheetah}{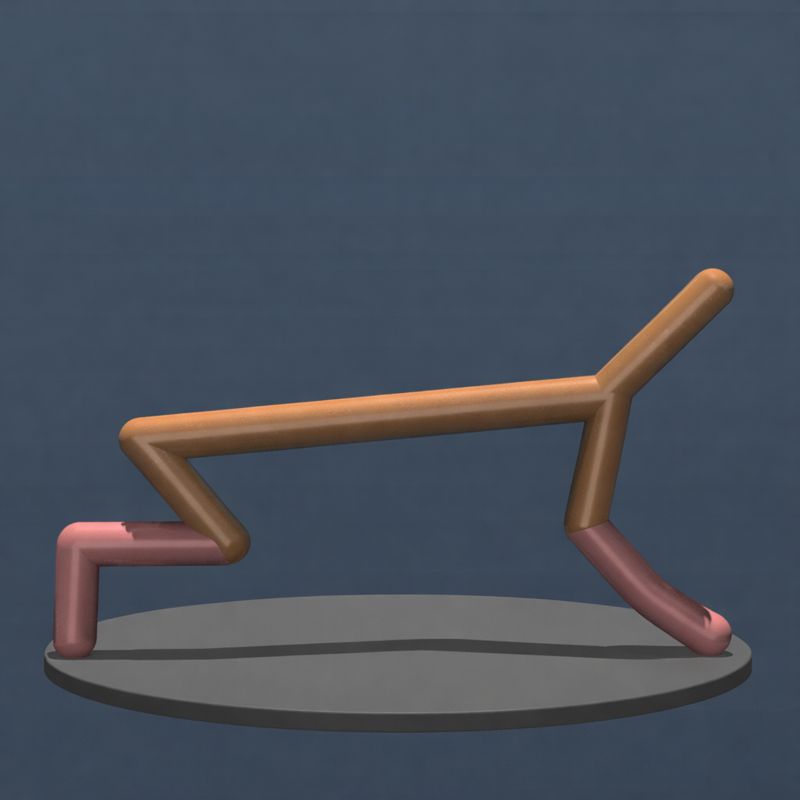}}
& Diffuser \cite{a7}    & 59.3 & 58.8 & 580 & 145 & 100 & 17.0 & 0.056 & 0.045 \\
& Dec.\ Diff. \cite{a88}  & 62.9 & 61.9 & 116 & 40  & 20  & 5.4  & 0.042 & 0.038 \\
& Diff.-QL  \cite{a89}   & 63.7 & 63.6 & 20  & 10  & 10  & 4.4  & 0.030 & 0.028 \\
& AdaptDiff\cite{a43}  & 59.6 & 59.5 & 353 & 101 & 60  & 12.4 & 0.050 & 0.041 \\
    \addlinespace[0.35em]
    \multirow[t]{4}{*}{\benchcell{Hopper}{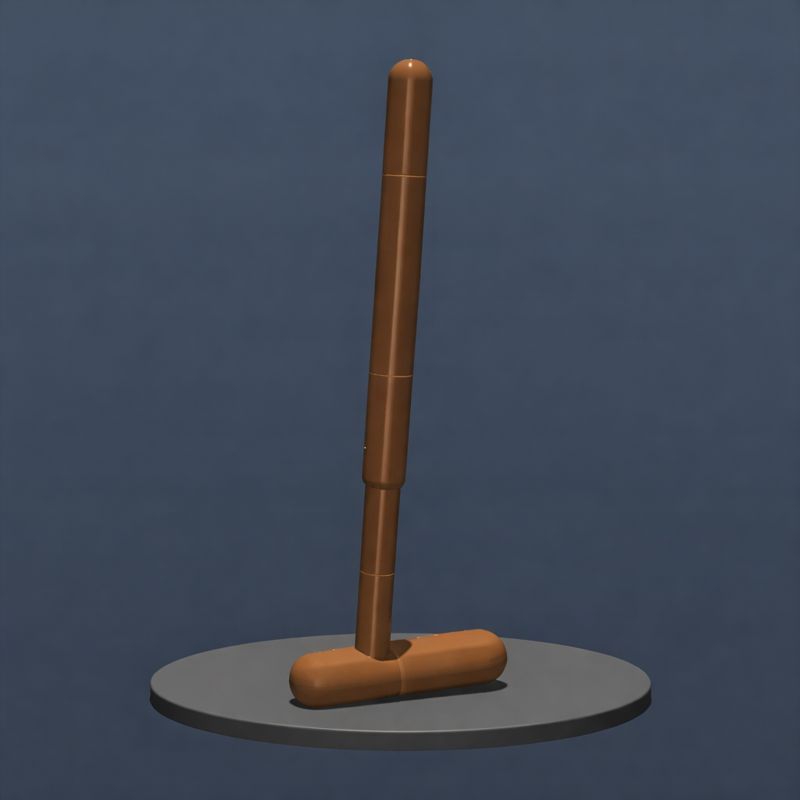}}
& Diffuser \cite{a7}    & 65.1 & 64.7 & 523 & 138 & 100 & 18.0 & 0.058 & 0.046 \\
& Dec.\ Diff. \cite{a88}  & 70.6 & 69.8 & 105 & 39  & 20  & 6.0  & 0.045 & 0.040 \\
& Diff.-QL \cite{a89}    & 75.7 & 75.6 & 20  & 11  & 10  & 4.7  & 0.032 & 0.030 \\
& AdaptDiff\cite{a43}  & 63.8 & 63.6 & 310 & 97  & 60  & 14.2 & 0.052 & 0.043 \\
    \addlinespace[0.75em]
    \multirow[t]{4}{*}{\benchcell{Walker2d}{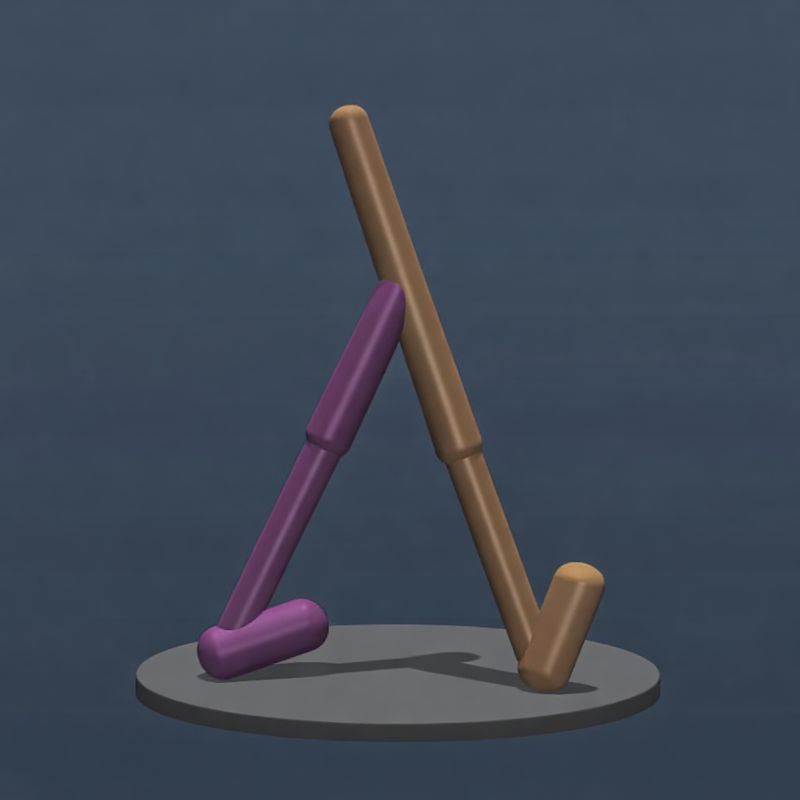}}
& Diffuser \cite{a7}    & 77.8 & 77.2 & 783 & 224 & 100 & 20.6 & 0.062 & 0.048 \\
& Dec.\ Diff. \cite{a88} & 84.0 & 83.7 & 157 & 63  & 20  & 6.7  & 0.048 & 0.042 \\
& Diff.-QL  \cite{a89}   & 80.8 & 80.7 & 25  & 14  & 25  & 12.7 & 0.034 & 0.032 \\
& AdaptDiff\cite{a43}  & 83.3 & 83.0 & 403 & 119 & 60  & 12.9 & 0.057 & 0.045 \\
    \addlinespace[0.35em]

    \midrule

    \sectionrow{Goal-reaching navigation (D4RL)}
    \multirow[t]{5}{*}{\benchcell{Maze2D}{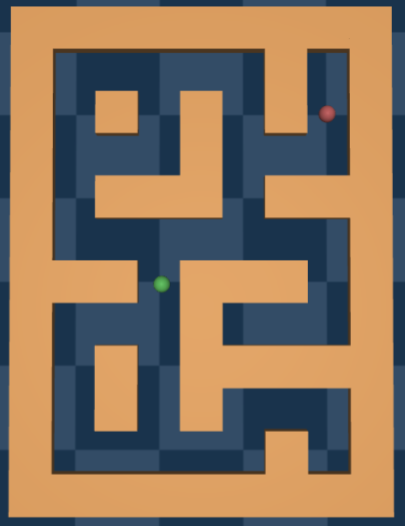}}
& Diffuser \cite{a7}    & 119.5 & 119.0 & 733  & 175 & 100 & 15.3 & 0.065 & 0.047 \\
& Dec.\ Diff. \cite{a88} & 130.8 & 129.6 & 147  & 46  & 20  & 4.7  & 0.055 & 0.044 \\
& Diff.-QL \cite{a89}    & 105.6 & 104.8 & 25   & 12  & 10  & 4.2  & 0.040 & 0.035 \\
& AdaptDiff\cite{a43}  & 144.3 & 143.9 & 347  & 99  & 60  & 12.4 & 0.060 & 0.046 \\
& CompDiff\cite{a91}   & 160.2 & 159.2 & 1100 & 275 & 150 & 25.0 & 0.070 & 0.049 \\
    \addlinespace[0.25em]
    \multirow[t]{5}{*}{\benchcell{AntMaze}{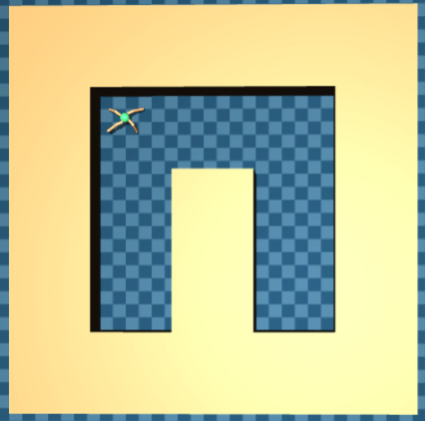}}
& Diffuser \cite{a7}    & 58.7 & 58.3 & 950  & 207 & 100 & 13.0 & 0.072 & 0.049 \\
& Dec.\ Diff. \cite{a88} & 72.2 & 71.8 & 190  & 53  & 20  & 4.0  & 0.060 & 0.046 \\
& Diff.-QL  \cite{a89}   & 59.6 & 59.4 & 25   & 10  & 25  & 8.8  & 0.045 & 0.038 \\
& AdaptDiff\cite{a43}  & 63.0 & 62.7 & 550  & 145 & 60  & 10.9 & 0.065 & 0.047 \\
& CompDiff\cite{a91}   & 75.9 & 74.5 & 1400 & 342 & 150 & 24.0 & 0.075 & 0.049 \\

    \midrule

    \sectionrow{Long-horizon manipulation (D4RL)}
    \multirow[t]{3}{*}{\benchcell{FrankaKitchen}{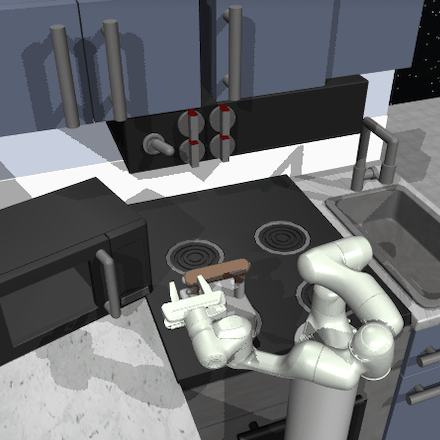}}
& Diffuser \cite{a7}    & 53.4 & 53.0 & 800 & 235 & 100 & 21.6 & 0.058 & 0.045 \\
& Dec.\ Diff. \cite{a88} & 55.0 & 54.9 & 160 & 62  & 20  & 6.3  & 0.045 & 0.040 \\
& AdaptDiff\cite{a43}  & 58.1 & 57.7 & 450 & 167 & 60  & 18.0 & 0.050 & 0.042 \\
  &                   &       &       &       &       &       &       &       &       \\
    \addlinespace[0.25em]

    \midrule

    \sectionrow{Robot block stacking (Kuka)}
    \multirow[t]{2}{*}{\benchcell{Kuka stacking}{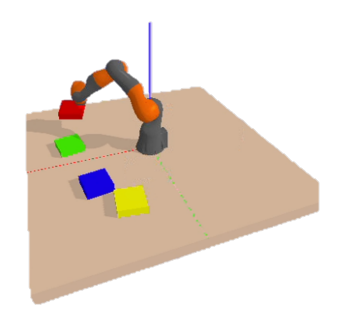}}
& Diffuser \cite{a7}    & 52.3 & 51.8 & 1200 & 400 & 100 & 25.9 & 0.055 & 0.044 \\
& Dec.\ Diff. \cite{a88}  & 60.4 & 60.0 & 240  & 100 & 20  & 7.0  & 0.045 & 0.040 \\
  &                   &       &       &       &       &       &       &       &       \\
  &                   &       &       &       &       &       &       &       &       \\
  &                   &       &       &       &       &       &       &       &       \\

    \addlinespace[0.25em]

    \arrayrulecolor{HeadGray}\bottomrule\arrayrulecolor{black}
  \end{tabularx}

  \caption{\textbf{Offline RL and trajectory planning.}
  Task metric is the benchmark metric (D4RL normalized score for D4RL tasks; success rate for Kuka stacking).
  $d(\cdot,\cdot)$ is the trajectory deviation metric used for risk accounting and $\hat{p}_{\mathrm{viol}}=\mathbb{P}\!\left(d(\tau^{\mathrm{full}}_0,\tilde{\tau}_0)>\eta_{\mathrm{traj}}\right)$.}
  \label{tab:results-offline}
  \end{threeparttable}
\end{table}

\tbm{Motion planning in configuration space.}  
We then test \nordy{Muninn} on classical motion planning problems where trajectories are represented in joint space rather than as RL episodes.

\begin{table}[H]
  \centering
  \scriptsize
  \setlength{\tabcolsep}{2pt}
  \renewcommand{\arraystretch}{1.12}
  \setlength{\benchw}{0.17\columnwidth}

  \begin{threeparttable}
  \begin{tabularx}{\columnwidth}{@{}%
    L{\benchw}%
    C{0.07\columnwidth}%
    Y%
    C{0.08\columnwidth}%
    D{0.07\columnwidth}%
    C{0.07\columnwidth}%
    D{0.07\columnwidth}%
    C{0.07\columnwidth}%
    D{0.07\columnwidth}%
    C{0.08\columnwidth}%
  @{}}

    \arrayrulecolor{HeadGray}\toprule\arrayrulecolor{black}

    \rowcolor{MuninnBlue!10}
    & \theadb{Scen.} & \theadb{Model} &
    \multicolumn{2}{c}{\theadb{Success(\%)}\upar} &
    \multicolumn{2}{c}{\makecell[c]{\theadb{Collis. (\%)} \dnar}} &
    \multicolumn{2}{c}{\theadb{Lat. (ms)} \dnar} &
    \makecell[c]{\textcolor{HeadGray}{$\hat{p}_{\mathrm{viol}}$}\dnar} \\

    \rowcolor{MuninnBlue!6}
    & & &
    \bfseries Full & \Muninnlogo &
    \bfseries Full & \Muninnlogo &
    \bfseries Full & \Muninnlogo &
    \\

    \midrule

    \multirow[t]{6}{*}{\benchcell{Arm planning in clutter}{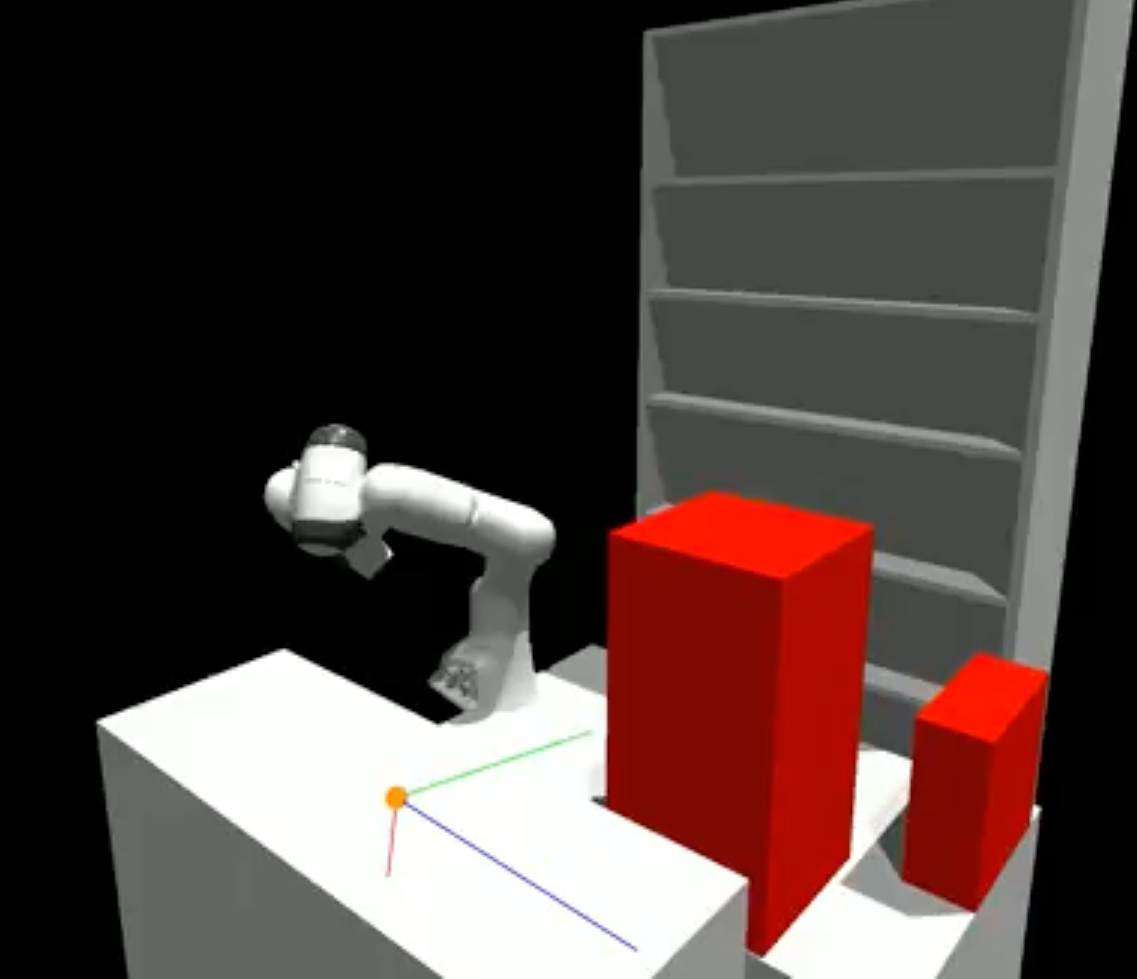}}
      & \multirow[t]{2}{*}{easy}   & MPD \cite{a37} & 99.0 & 98.8 & 1.0 & 1.2 & 180 & 120 & 0.040 \\
      &                           & EDMP\cite{a87} & 99.5 & 99.4 & 0.5 & 0.6 & 120 & 92  & 0.035 \\
    \addlinespace[0.25em]
      & \multirow[t]{2}{*}{med.} & MPD \cite{a37}  & 94.0 & 93.5 & 6.0 & 6.4 & 200 & 118 & 0.045 \\
      &                           & EDMP\cite{a87} & 96.5 & 96.2 & 3.5 & 3.7 & 135 & 95  & 0.040 \\
    \addlinespace[0.25em]
      & \multirow[t]{2}{*}{hard}   & MPD \cite{a37}  & 82.0 & 81.2 & 18.0 & 18.8 & 230 & 121 & 0.048 \\
      &                           & EDMP\cite{a87} & 88.0 & 87.5 & 12.0 & 12.4 & 160 & 107 & 0.044 \\

    \arrayrulecolor{HeadGray}\bottomrule\arrayrulecolor{black}
  \end{tabularx}

  \caption{\textbf{Configuration-space motion planning.} Scenarios correspond to increasing clutter/difficulty under the EDMP \cite{a87} evaluation protocol. Collision/violation includes obstacle collisions and any hard constraint violations.}
  \label{tab:results-motion}
  \end{threeparttable}
\end{table}

\tbm{Visuomotor imitation and manipulation.}  
Finally, we evaluate \nordy{Muninn} on visuomotor policy diffusers that generate pose trajectories conditioned on images (2D) or 3D observations.
\begin{table}[H]
  \centering
  \scriptsize
  \setlength{\tabcolsep}{2pt}
  \renewcommand{\arraystretch}{1.12}

  \begin{threeparttable}
  \begin{tabularx}{\columnwidth}{@{}%
    L{0.25\columnwidth}%
    C{0.11\columnwidth}%
    *{3}{>{\centering\arraybackslash}X}%
  @{}}

    \arrayrulecolor{HeadGray}\toprule\arrayrulecolor{black}

    \rowcolor{MuninnBlue!10}
    \theadb{Benchmark/Metric} &
    \theadb{} &
    \taskhead{RLBench \cite{a97}\\Reach Target}{Diff. Policy \cite{a8}}{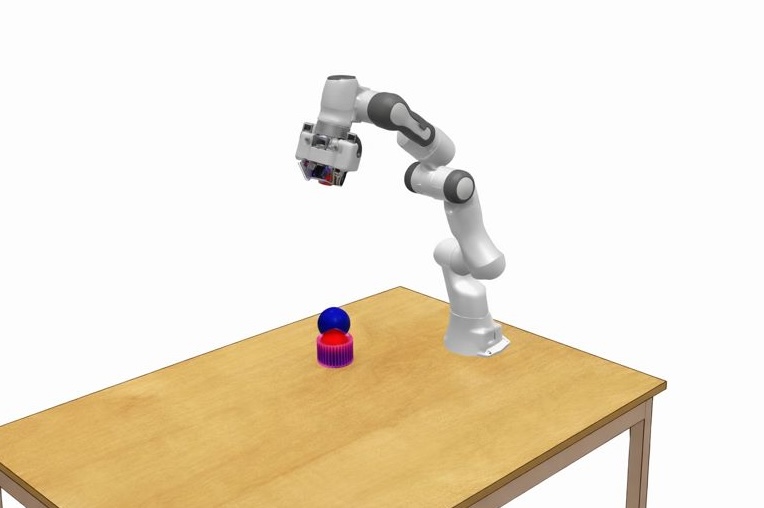} &
    \taskhead{Meta-World\cite{a98}\\pick-place-v2}{Diff. Policy \cite{a8}}{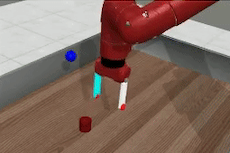} &
    \taskhead{DP3\\Pour}{DP3 \cite{a90}}{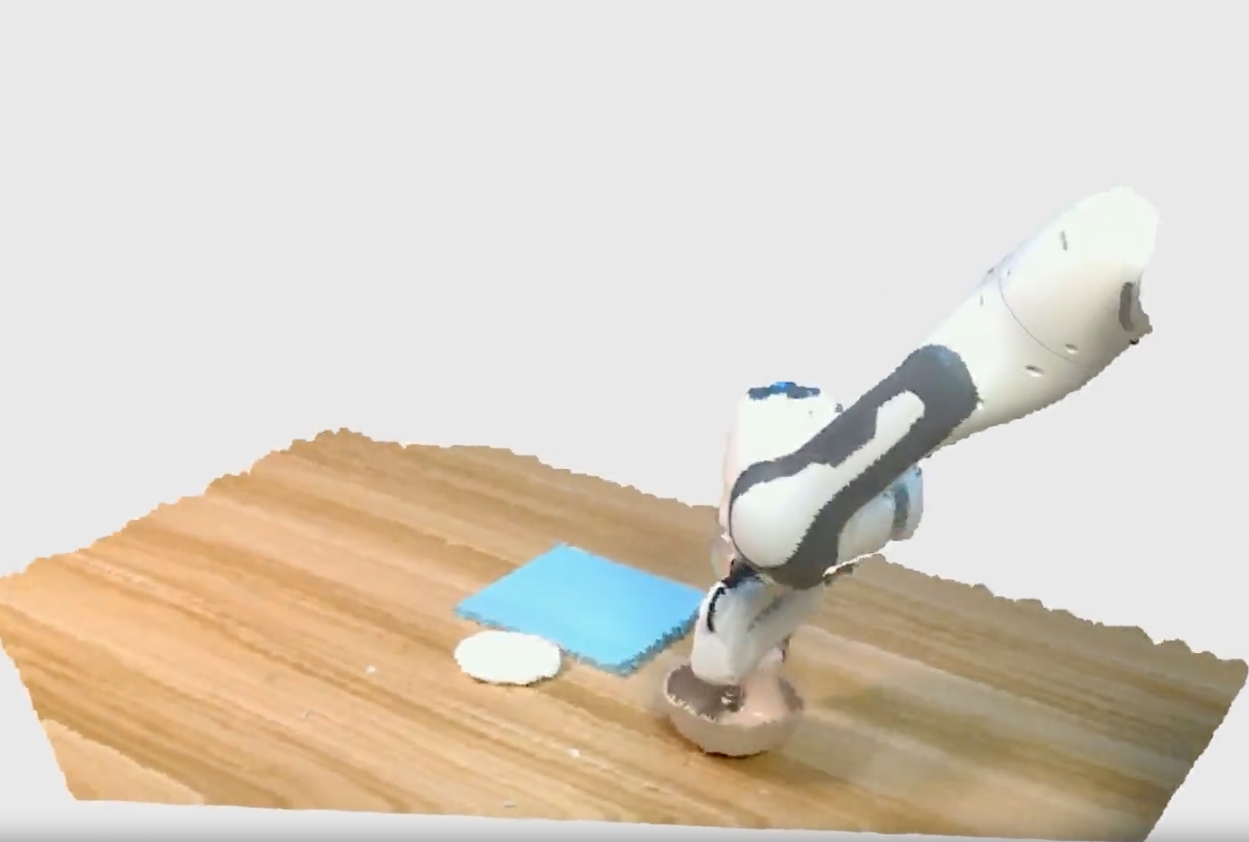} \\

    \midrule

    \multirow{2}{*}{\makecell[l]{\theadb{Success (\%)\upar}\\\theadb{} \dnar}} &
      \textcolor{HeadGray}{\bfseries Full} & 99.0 & 92.0 & 98.0 \\
    \rowcolor{MuninnBlue!4}
    & \Muninnlogo & 98.8 & 91.5 & 97.6 \\
    \addlinespace[0.2em]

    \multirow{2}{*}{\makecell[l]{\theadb{Latency (ms)\dnar}\\\theadb{} \dnar}} &
      \textcolor{HeadGray}{\bfseries Full} & 25 & 30 & 45 \\
    \rowcolor{MuninnBlue!4}
    & \Muninnlogo & 15 & 18 & 27 \\
    \addlinespace[0.2em]
    \makecell[l]{\textcolor{HeadGray}{\bfseries $\boldsymbol{\mathbb{E}[d]}$} \dnar} &
      & 0.035 & 0.040 & 0.038 \\
    \addlinespace[0.2em]

    \makecell[l]{\textcolor{HeadGray}{\bfseries $\boldsymbol{\hat{p}_{\mathrm{viol}}}$} \dnar} &
      & 0.040 & 0.045 & 0.043 \\

    \arrayrulecolor{HeadGray}\bottomrule\arrayrulecolor{black}
  \end{tabularx}

  \caption{\textbf{Visuomotor imitation and manipulation.}
  Latency is measured per policy inference/control call. $d(\cdot,\cdot)$ is defined on the trajectory segment.}
  \label{tab:results-policy}
  \end{threeparttable}
\end{table}

\subsection{Quantitative Results: Hardware Evaluation} \label{hardware-res}
We further validate \nordy{Muninn} in real-world, previously unseen, closed-loop deployment on robotic platforms spanning navigation and manipulation.
(i) \emph{2D marine navigation:} a SeaRobotics Surveyor ASV performs waypoint-following in an open-water lake under partial observability.
(ii) \emph{3D aerial navigation:} a Crazyflie quadrotor performs 3D waypoint navigation.
(iii) \emph{Manipulation:} a SO-ARM100 tabletop arm executes pick-and-place, stacking, and peg insertion.
Complete platform specifications are given in Appendix~\ref{app:hardware}.

\begin{figure}[H]
  \centering
  \includegraphics[width=\linewidth]{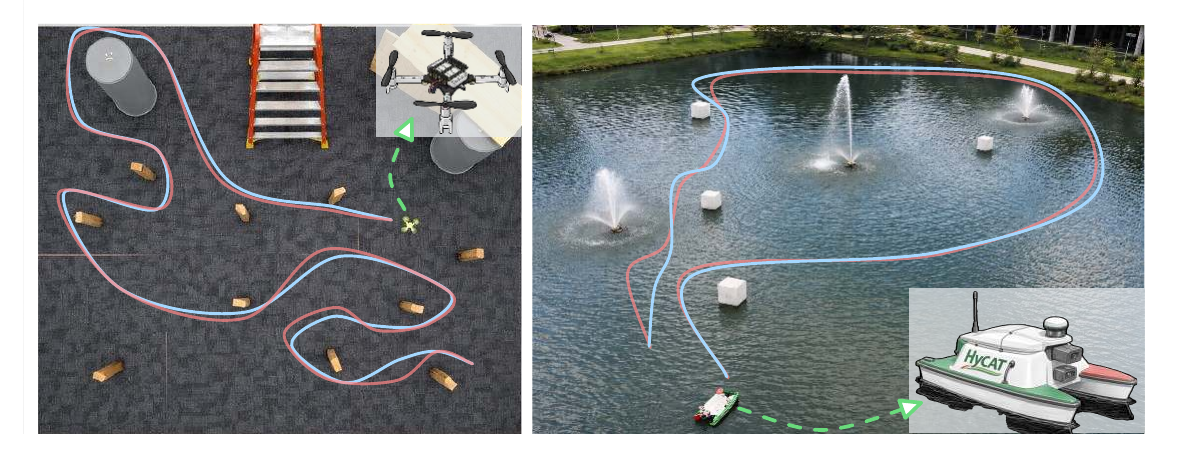}
  \caption{\textbf{Qualitative results.}
Full-compute trajectories (GC-Diffuser and B-COD teacher) are shown in red; \nordy{Muninn}-accelerated trajectories are shown in blue.}
  \label{fig:hw1}
\end{figure}

\begin{figure}[H]
\centering

\begin{minipage}[t]{0.34\columnwidth}
  \vspace{0pt}
  \centering
  \includegraphics[width=0.8\linewidth]{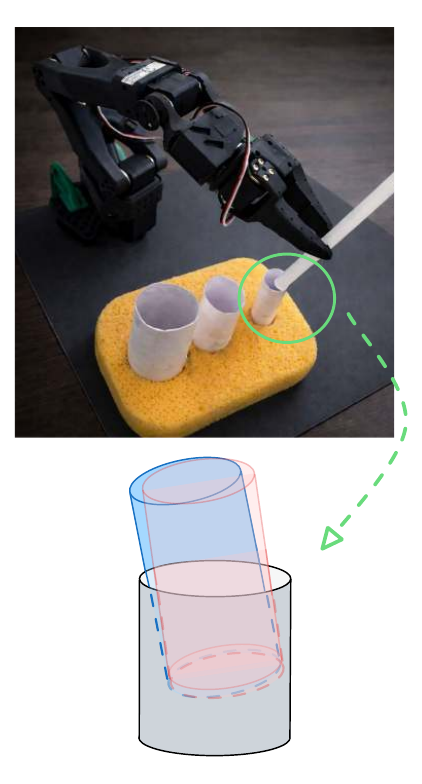}

  \captionsetup{font=scriptsize,labelfont=bf}
  \caption{\textbf{Angle of attacks (hard peg insertion)} Red: Diffusion Policy; blue: \nordy{Muninn}.
}
  \label{fig:hw2}
\end{minipage}\hspace{0.02\columnwidth}
\begin{minipage}[t]{0.6\columnwidth}
  \vspace{0pt}
  \centering
\scriptsize
\setlength{\tabcolsep}{1.2pt}
\renewcommand{\arraystretch}{1.10}

\begin{threeparttable}
\begin{tabular}{@{}%
  >{\centering\arraybackslash}m{0.11\linewidth}
  >{\raggedright\arraybackslash}p{0.31\linewidth}
  @{\hspace{0.0pt}}%
  >{\centering\arraybackslash}p{0.14\linewidth}
  >{\centering\arraybackslash}p{0.14\linewidth}
  >{\centering\arraybackslash}p{0.14\linewidth}
@{}}

  \arrayrulecolor{HeadGray}\toprule\arrayrulecolor{black}

  \rowcolor{MuninnBlue!10}
  \makecell[c]{\theadb{Pl.}} &
  \theadb{Method} &
  \makecell[c]{\theadb{Lat.}\\\theadb{(ms)} \dnar} &
  \makecell[c]{\theadb{Succ.}\\\theadb{(\%)} \upar} &
  \makecell[c]{\theadb{Coll.}\\\theadb{(\%)} \dnar} \\

  \midrule

  & BCOD$_{\text{1s}}$ \cite{a92}                & 12 & 94.0 & 2.4 \\
  & BCOD$_{\text{tch.}}$ \cite{a92}               & 60 & 97.0 & 1.6 \\
  \rowcolor{MuninnBlue!4}
  & BCOD$_{\text{tch.}}+\Muninnlogo$   & 24 & 96.8 & 1.7 \\
  & GC-Diff. \cite{a7}                           & 70 & 95.5 & 2.2 \\
  \rowcolor{MuninnBlue!4}
  \multirow{-5}{*}{\plat{ASV}} &
  GC-Diff.$+\Muninnlogo$               & 28 & 95.3 & 2.3 \\
  \addlinespace[0.25em]

  & GC-Diff. \cite{a7}                          & 75 & 93.0 & 3.0 \\
  \rowcolor{MuninnBlue!4}
  \multirow{-2}{*}{\plat{UAV}} &
  GC-Diff.$+\Muninnlogo$               & 30 & 92.8 & 3.1 \\
  \addlinespace[0.25em]

    & & & \\
  & Diff. policy \cite{a8}                                & 40 & 82.0 & 6.0 \\
  \rowcolor{MuninnBlue!4}
  \multirow{-3}{*}{\plat{SO-100}} &
  DP $+\Muninnlogo$                    & 24 & 81.5 & 6.2 \\
  \arrayrulecolor{HeadGray}\bottomrule\arrayrulecolor{black}
\end{tabular}
\end{threeparttable}
  \captionsetup{type=table,font=scriptsize,labelfont=bf}

\caption{
We compare the Full (teacher) models against the baselines and \nordy{Muninn} on the waypoint navigation and manipulation (averaged over three tasks) tasks. (Average of 150 hardware runs).}
\end{minipage}

\end{figure}

\subsection{Key findings, analysis, and ablations}
\label{subsec:ablations}

As can be seen from the results in sections \ref{subsec:envs} and \ref{hardware-res}, \nordy{Muninn} consistently reduces denoiser evaluations and wall-clock latency across all evaluated policies while preserving task performance. 
We now report key findings (KF) and summarize ablations of \nordy{Muninn}'s internal design choices, with further breakdowns deferred to Appendix~\ref{app:ablations}.

\tbm{(KF\#1) \nordy{Muninn} dominates inference-time acceleration baselines on the \emph{task–latency–risk} Pareto frontier.}
A common path to real-time diffusion planning is \emph{training-free} inference-time acceleration, i.e., reducing denoiser work without retraining. Most such methods are compute-centric: they skip calls without modeling how errors propagate through the sampler or impact the final trajectory.
We compare \nordy{Muninn} against three such baselines (Fig. \ref{fig:pareto_inference_baselines}, Table \ref{tab:inference_baselines}): \textit{(i) \textsc{FewSteps}}\cite{a95}: run the same pretrained denoiser with a reduced diffusion horizon $T' < T$.
\textit{(ii) \textsc{FixedSkip}:} evaluate the denoiser every $k$ steps and reuse the last predicted noise for skipped steps.
\textit{(iii) \textsc{ProbeThresh}:} reuse if the probe-change $s_t$ falls below a tuned threshold $\delta$.

\begin{figure}[H]
\centering

\begin{minipage}[t]{0.53\columnwidth}
  \vspace{0pt} 
  \centering
  \includegraphics[width=0.9\linewidth]{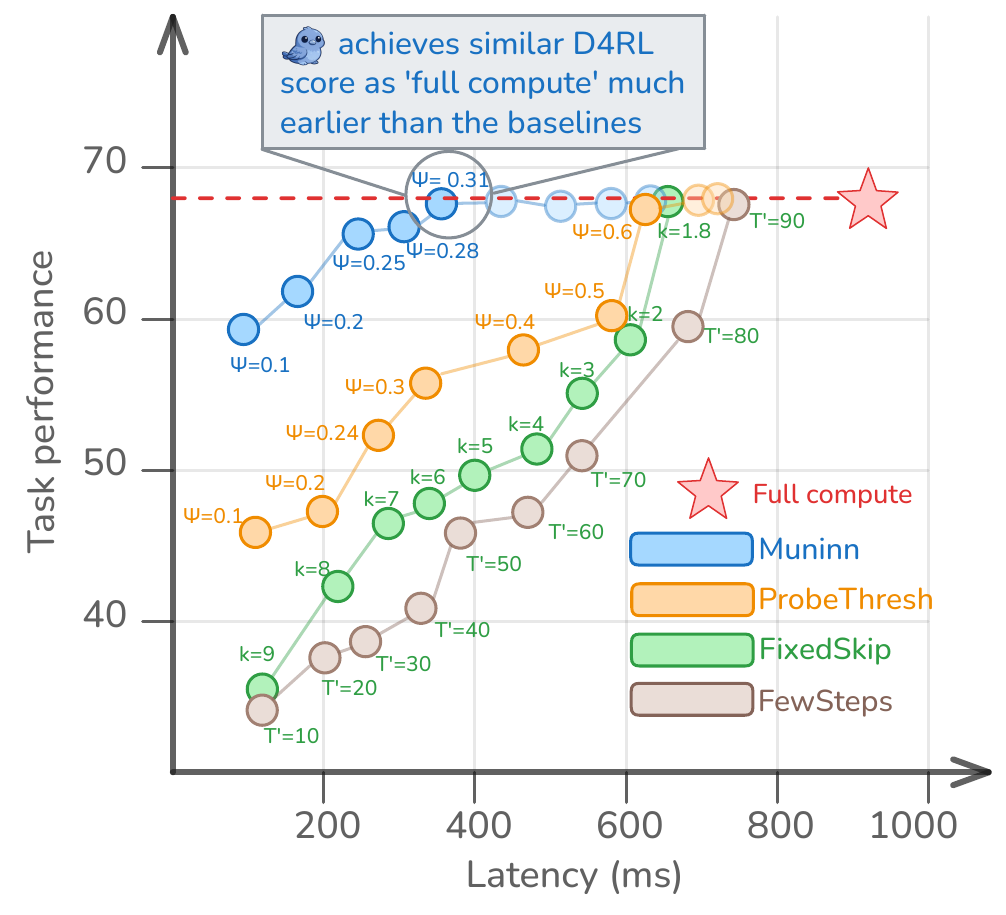}

  \captionsetup{font=scriptsize,labelfont=bf}
\caption{We plot the D4RL score vs.\ wall-clock latency for D4RL AntMaze. \nordy{Muninn} achieves the same task success rate as Full model 2$\times$ quicker than the baselines.}
  \label{fig:pareto_inference_baselines}
\end{minipage}
\begin{minipage}[t]{0.45\columnwidth}
  \vspace{0pt} 
  \centering
  \scriptsize
  \setlength{\tabcolsep}{1.4pt}
  \renewcommand{\arraystretch}{1.10}

  \begin{tabularx}{\linewidth}{@{}%
    Y%
    C{0.18\linewidth}%
    C{0.18\linewidth}%
    C{0.16\linewidth}%
  @{}}
    \arrayrulecolor{HeadGray}\toprule\arrayrulecolor{black}

    \rowcolor{MuninnBlue!10}
    \theadb{Method} &
    \theadb{Task} \upar &
    \textcolor{HeadGray}{\bfseries $\mathbb{E}[d]$}\dnar &
    \textcolor{HeadGray}{\bfseries $\hat{p}_{\mathrm{viol}}$}\dnar \\

    \midrule

    \makecell[l]{Full}           & 68.7 & 0.00 & 0.00 \\
    \textsc{FewSteps}            & 56.1 & 0.12 & 0.17 \\
    \textsc{FixedSkip}           & 42.5 & 0.10 & 0.14 \\
    \textsc{ProbeThresh}         & 48.0 & 0.09 & 0.10 \\
    \rowcolor{MuninnBlue!4}
    \Muninnlogo + Full            & 67.6 & 0.07 & 0.04 \\

    \arrayrulecolor{HeadGray}\bottomrule\arrayrulecolor{black}
  \end{tabularx}

  \captionsetup{type=table,font=scriptsize,labelfont=bf}

\caption{
We compare the Full (unmodified $T$-step sampler) model against the baselines and \nordy{Muninn} on the D4RL AntMaze task. The results above correspond to allowing Full to run for 957 ms, at which point it achieves the maximum possible D4RL score. The baselines and \nordy{Muninn} are capped at 250 ms (approximately 1/4th of the runtime).}
\label{tab:inference_baselines}
  \label{tab:inference_baselines}
\end{minipage}

\end{figure}

Considering the offline RL domain (averaged across benchmarks in Table \ref{tab:results-offline}), we find that \nordy{Muninn} yields a \emph{strictly better} operating curve: at the same latency it preserves task success more reliably, and at the same $\hat{p}_{\mathrm{viol}}$ it achieves lower latency.
This happens because \textsc{FewSteps} and \textsc{FixedSkip} implicitly allocate approximation error uniformly in diffusion time, even though the sampler's sensitivity profile $L_t$ is highly non-uniform. 
\textsc{ProbeThresh} adapts to probe dynamics, but without calibrated error envelopes it violates risk in hard contexts or becomes overly conservative when tuned for safety. The baselines inherently produce small latent errors that become large deviations after downstream sampler mixing.
\nordy{Muninn} is the only inference-time method in this set that (a) estimates reuse error with distribution-free coverage and (b) accounts for step-dependent error amplification through $\Gamma L_t$.

\tbm{(KF \#2) \nordy{Muninn} closes much of the gap to training-time accelerations \emph{without retraining}.}
A second path to real-time diffusion planning is {training-time} acceleration. These methods can achieve very low latency, but require additional training compute, careful hyperparameter sweeps, and often task- or architecture-specific redesign.
We compare \nordy{Muninn} (Fig. \ref{fig:traincost_vs_latency}, Table \ref{tab:training_time_baselines}) against: \textit{(i) \textsc{Distill-1step}/\textsc{Distill-Kstep}}\cite{a93, a94}: one/few-step distilled planners, \textit{(ii) \textsc{SmallNet}:}\cite{a94} a smaller denoiser trained from scratch or fine-tuned (fewer layers / width), \textit{(iii) \textsc{LearnedExit}}\cite{a90}: a trained controller that predicts early stopping.
\begin{figure}[H]
\centering

\begin{minipage}[t]{0.52\columnwidth}
  \vspace{0pt}
  \centering
  \scriptsize
  \setlength{\tabcolsep}{2pt}
  \renewcommand{\arraystretch}{1.12}

  \begin{threeparttable}
  \begin{tabularx}{\linewidth}{@{}%
    Y%
    C{0.12\linewidth}%
    C{0.12\linewidth}%
    C{0.14\linewidth}%
  @{}}

    \arrayrulecolor{HeadGray}\toprule\arrayrulecolor{black}

    \rowcolor{MuninnBlue!10}
    \theadb{Method} &
    \theadb{Extra train.} &
    \theadb{Extra data} &
    \theadb{Arch. change} \\

    \midrule

    \makecell[l]{Full multi-step\\\textcolor{HeadGray}{(teacher)}} &
    \iconNo & \iconNo & \iconNo \\

    \textsc{Distill-1step} &
    \iconHigh & \iconOften & \iconYes \\

    \makecell[l]{\textsc{Distill-Kstep}\\\textcolor{HeadGray}{(K=4)}} &
    \iconHigh & \iconOften & \iconYes \\

    \makecell[l]{\textsc{SmallNet}\\\textcolor{HeadGray}{(0.6$\times$ width)}} &
    \iconMod & \iconNo & \iconYes \\

    \textsc{LearnedExit} &
    \iconHigh & \iconYes & \iconYes \\

    \rowcolor{MuninnBlue!4}
    \makecell[l]{Teacher + \Muninnlogo} &
    \iconNo & \iconNo & \iconNo \\

    \rowcolor{MuninnBlue!2}
    \makecell[l]{\textsc{Distill-Kstep} + \Muninnlogo} &
    \iconNo & \iconNo & \iconNo \\

    \arrayrulecolor{HeadGray}\bottomrule\arrayrulecolor{black}
  \end{tabularx}

  \captionsetup{type=table,font=scriptsize,labelfont=bf}
  \caption{
Columns report whether each method requires extra training compute, additional data (beyond the original training set), or architectural changes for the DP3 Pour benchmark. We include \textsc{Distill}+\nordy{Muninn} to show composability.
  \emph{Legend:} \iconNo\ = none,\;
  \iconYes\ = yes,\;
  \iconOften\ = often,\;
  \iconMod\ = moderate,\;
  \iconHigh\ = high.}
  \label{tab:training_time_baselines}
  \end{threeparttable}
\end{minipage}
\begin{minipage}[t]{0.46\columnwidth}
  \vspace{0pt}
  \centering
  \includegraphics[width=\linewidth]{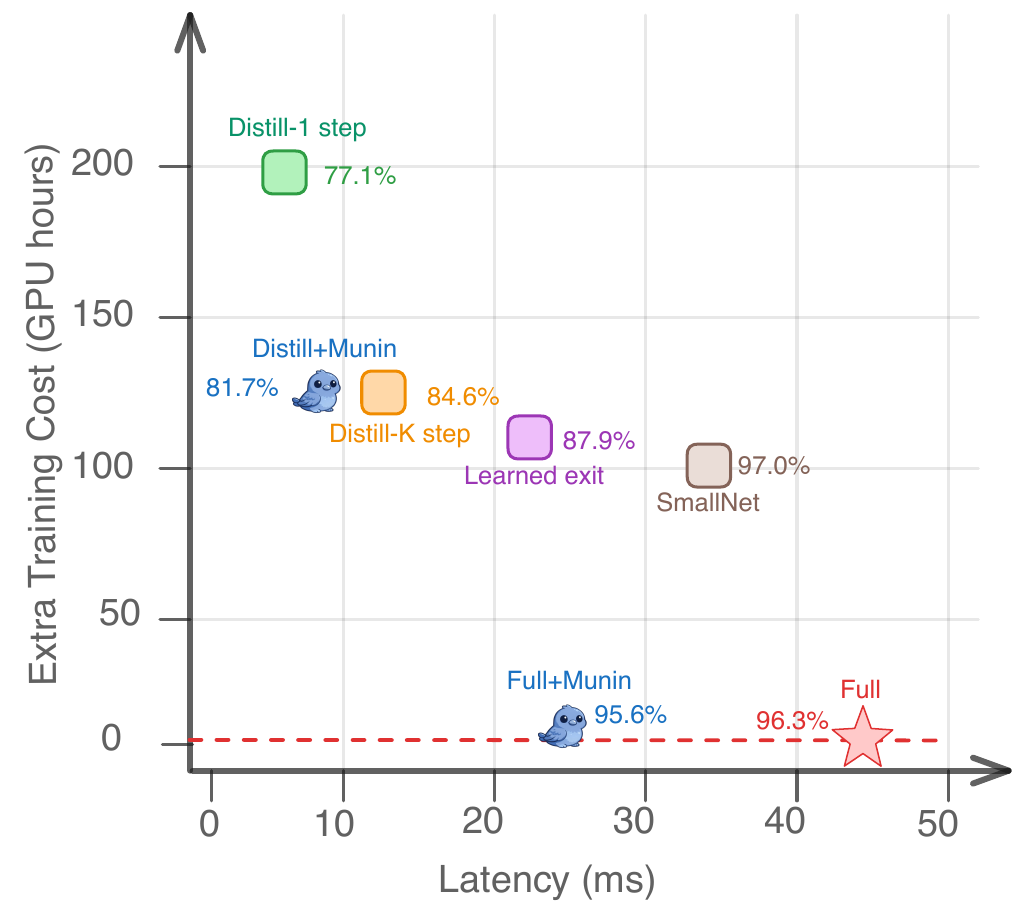}

  \captionsetup{font=scriptsize,labelfont=bf}
  \caption{
Scatter plot summarizing the practical acceleration spectrum on the DP3 Pour diffusion policy. We also report the average success percentage of these models during inference. GPU-hours are normalized to A10 24GB-hours equivalent. More details about the training procedure is available in the Appendix.}
  \label{fig:traincost_vs_latency}
\end{minipage}

\end{figure}
We find that \nordy{Muninn} provides a {deployment-friendly} point on the spectrum: it recovers a large fraction of the achievable latency reduction \emph{immediately}, while preserving the behavior of the original planner.
Moreover, \nordy{Muninn} is {composable}: applying \nordy{Muninn} on top of a distilled or compact model yields additional speedups, because \nordy{Muninn} targets {redundant per-step computation} that remains even in fewer-step samplers.

\tbm{(KF\#3) \nordy{Muninn} induces \emph{difficulty-adaptive compute}.}
\begin{table}[H]
\centering
\scriptsize
\setlength{\tabcolsep}{2pt}
\renewcommand{\arraystretch}{1.12}

\begin{threeparttable}
\begin{tabular}{@{}c@{\hspace{0.03\columnwidth}}c@{}}

\begin{minipage}[t]{0.47\columnwidth}
\vspace{0pt}\centering
\begin{tabular}{@{}ccc@{}}
\arrayrulecolor{HeadGray}\toprule\arrayrulecolor{black}
\rowcolor{MuninnBlue!6}\theadb{Eval bin} & \theadb{Fraction eps.} & \theadb{Success}\\
\midrule
10--12 & 0.20 & 0.99 \\
13--15 & 0.30 & 0.97 \\
16--18 & 0.25 & 0.94 \\
19--21 & 0.15 & 0.90 \\
22--25 & 0.10 & 0.87 \\
\arrayrulecolor{HeadGray}\bottomrule\arrayrulecolor{black}
\end{tabular}
\end{minipage}

&

\begin{minipage}[t]{0.47\columnwidth}
\vspace{0pt}\centering
\begin{tabular}{@{}ccc@{}}
\arrayrulecolor{HeadGray}\toprule\arrayrulecolor{black} 
\rowcolor{MuninnBlue!6}\theadb{Min. clearance} & \theadb{Avg. evals} & \theadb{Success}\\
\midrule
0.1 & 22 & 0.84 \\
0.3 & 20 & 0.90 \\
0.5 & 18 & 0.94 \\
0.7 & 15 & 0.97 \\
0.9 & 12 & 0.99 \\
\arrayrulecolor{HeadGray}\bottomrule\arrayrulecolor{black}
\end{tabular}
\end{minipage}

\end{tabular}

\caption{
\textbf{Left:} Episode fraction and \nordy{Muninn} success binned by \#evaluations: low-eval bins are easier (higher success), high-eval bins are harder (lower success).
\textbf{Right:} Compute and success vs.\ minimum-clearance difficulty: \nordy{Muninn} uses more evals. in harder and fewer in easier regions.}
\label{tab:kf3_adaptive_compute}
\end{threeparttable}
\end{table}
A hidden cost of fixed accelerations is that they spend the same compute on {easy} and {hard} situations.
In robotics, difficulty varies across time (e.g., open-space navigation vs.\ tight corridors).

\nordy{Muninn}'s decision rule is inherently {state- and context-adaptive}. Taking the example of cluttered 7-DoF arm planning (averaged across EDMP and MPD) (Table \ref{tab:kf3_adaptive_compute}), we analyze the distribution of denoiser evaluations per episode and correlate it with environment difficulty proxies.
The left panel in Table \ref{tab:kf3_adaptive_compute} should be read as a distribution over \nordy{Muninn}'s actual compute usage: episodes that require only 10--12 denoiser evaluations are typically easier and succeed more often, while episodes that consume 22--25 evaluations are harder and have lower success. The right panel Table \ref{tab:kf3_adaptive_compute} makes this difficulty axis explicit using minimum clearance: as clearance decreases from open-space to tight-clutter settings, \nordy{Muninn} automatically recomputes more often instead of applying a fixed skip pattern.
We find that \nordy{Muninn} naturally {allocates more compute} to high-risk or high-uncertainty regimes, and aggressively reuses outputs when the diffusion trajectory stabilizes.

\tbm{(KF\#4) \nordy{Muninn}'s budget accounting produces a \emph{usable runtime certificate} that correlates with downstream safety.}
Most acceleration methods are opaque: they reduce compute, but provide no introspection about when approximations become dangerous.
\nordy{Muninn} is different because its internal quantity $\hat{D}$ is an explicit, online-tracked {upper bound proxy} (Fig. \ref{fig:budget_certificate}) for how far the cached trajectory might drift from full compute.
This yields two practical benefits: \textit{(i) Post-hoc explainability:} we can attribute most deviation risk to specific timesteps and contexts where $L_t$ is large and the probe indicates instability.
\textit{(ii) Runtime escalation:} in a closed-loop controller, we can treat low remaining budget $B_{\mathrm{rem}}$ or large $\hat{D}$ as a trigger to (a) temporarily run full compute, (b) reduce the action magnitude, (c) switch to a conservative safety controller, or (d) request more samples. We present an extended study of the runtime escalations in the Appendix \ref{app:ablations}.

\begin{figure}[H]
\centering
\includegraphics[width=0.8\linewidth]{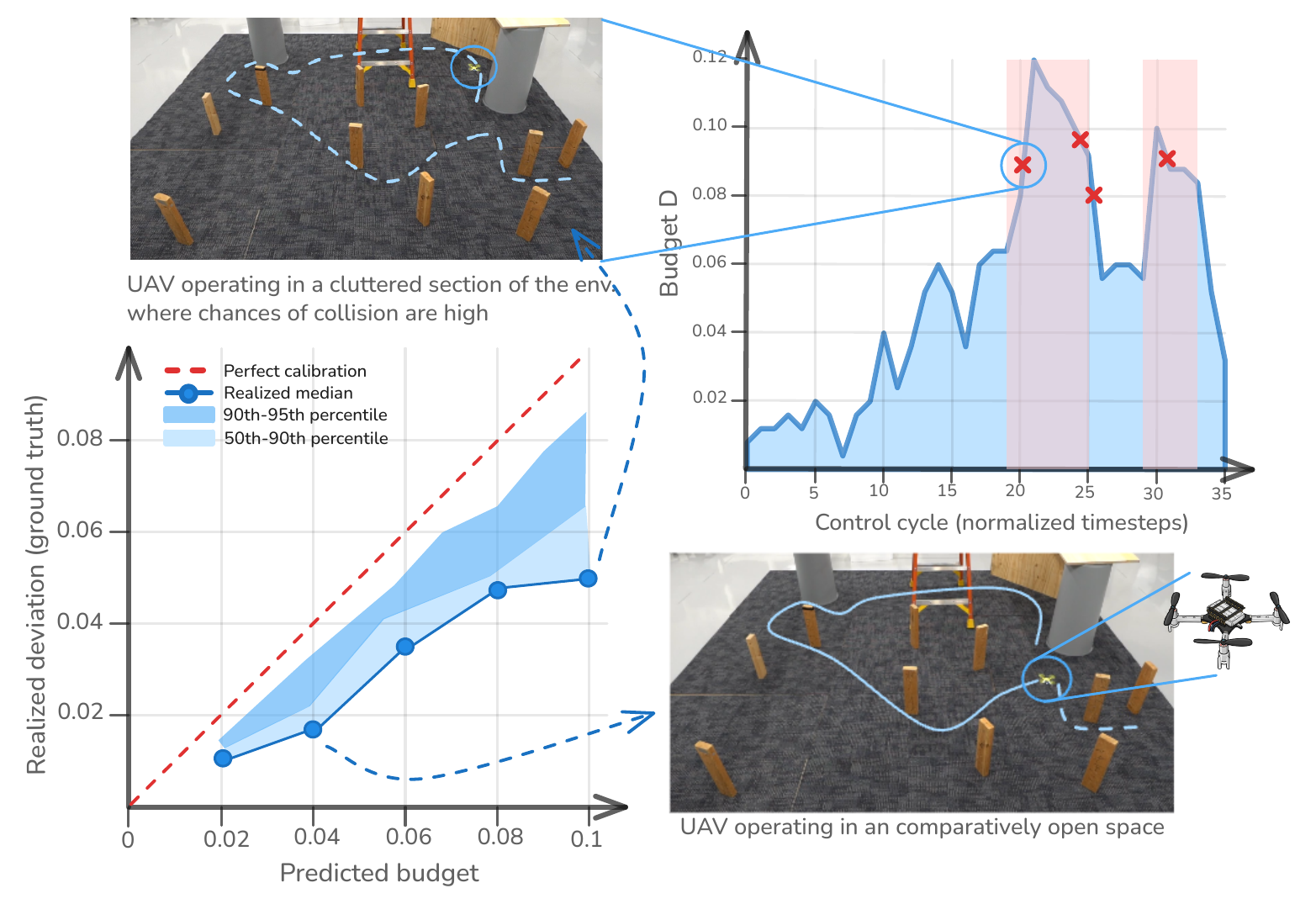}
\caption{\textbf{Left:} Reliability plot: predicted $\hat{D}$ vs.\ realized deviation $d$; \nordy{Muninn} is monotone and conservative (below the diagonal).
\textbf{Right:} GC-Diffuser+\nordy{Muninn} on 3D UAV navigation: $\hat{D}$ over control cycles, with spikes aligning to near-collision/contact events (red); callouts show the UAV pose at those instants.}
\label{fig:budget_certificate}
\end{figure}

\tbm{Ablations.}
We ablate Munin's key design choices and report a detailed summary in Appendix~\ref{app:ablations}.
Overall, (i) increasing the trajectory deviation budget $\eta_{\mathrm{traj}}$ yields a smooth speed--fidelity knob (more reuse equals higher speedup with minor performance loss);
(ii) split-conformal calibration is necessary to reliably achieve a target violation probability $\hat{p}_{\mathrm{viol}}\approx\alpha$, while heuristic rules miscalibrate;
(iii) sensitivity-aware budgeting via sampler coefficients $\{L_t\}$ improves the speed--safety trade-off by allocating recomputation to high-impact timesteps;
(iv) shallow stem probes provide the best net speedups (good predictiveness at low overhead);
(v) forbidding reuse in a small prefix/suffix further reduces collisions/violations with modest speed loss; and
(vi) calibration is sample-efficient, with risk and speed stabilizing after a modest number of rollouts.

\section{Conclusion}
This work introduces \Muninnlogo \nordy{Muninn}, a training-free caching wrapper that accelerates trajectory diffusion models without modifying model weights or the sampler. \nordy{Muninn} leverages two signals that are present in trajectory diffusion planners: (i) a cheap probe of how the model's internal representation evolves across denoising steps, and (ii) analytic sampler sensitivity coefficients that quantify how denoiser perturbations propagate into the final trajectory. By calibrating probe dynamics against potential reuse error, \nordy{Muninn} produces step-wise, distribution-free error envelopes that can be converted into a trajectory-level deviation budget. At run time, \nordy{Muninn} spends this budget to decide when cached denoiser outputs can be safely reused.
Across simulated benchmarks and hardware experiments, \nordy{Muninn} substantially reduces denoiser evaluations and wall-clock latency while preserving task performance and safety metrics.
A promising direction for future work is to replace the current fixed user-specified deviation budget with a context-aware budget that automatically tightens in cluttered or contact-rich regimes and relaxes in open-space or low-risk regimes.

\section*{Acknowledgements}
This work was supported in part by NSF grants 2118329, IIS-2024733 and IIS-2331908, the Office of Naval Research grants N00014-23-1-2505, N00014-25-1-2519, N00014-23-1-2789, the U.S. Department of Defense grant 78170-RT-REP, and by the Army Research Laboratories under contract W911NF1920243.  This is contribution \#2135 from the Institute of Environment at Florida International University.

\bibliographystyle{plainnat}
\bibliography{references}
\clearpage

\newpage
\appendix
\textbf{Table of Contents:}

\etocsettocstyle{}{}
\localtableofcontents

\newpage

\subsection{Simulation: Environment, Task, and Dataset Details}
\label{app:envs}

This subsection provides the full environment/task definitions and dataset details underlying
Tables~\ref{tab:results-offline}--\ref{tab:results-policy}.
For every benchmark, we specify:
(i) observation/state/action spaces,
(ii) the trajectory representation $\tau \in \mathbb{R}^{H\times d}$ used by the diffusion model(s),
(iii) the trajectory horizon $H$ and diffusion horizon $T$ (number of denoising steps),
(iv) the conditioning context $c$,
(v) dataset versioning and preprocessing, and
(vi) the task metric and how it is computed.

\subsubsection{Offline RL / Trajectory Planning (D4RL)}
\label{app:envs:d4rl}

All D4RL planners in Table~\ref{tab:results-offline} operate over {state--action} trajectory segments.
Let $s_t \in \mathbb{R}^{d_s}$ be the environment observation/state and $a_t \in \mathbb{R}^{d_a}$ the action.
A trajectory segment of horizon $H$ is represented as
\begin{equation*}
\label{eq:app:traj_repr_d4rl}
\tau \;=\; \big[(s_0,a_0), (s_1,a_1), \dots, (s_{H-1},a_{H-1})\big]
\;\in\; \mathbb{R}^{H \times (d_s+d_a)}.
\end{equation*}
We denote $d := d_s+d_a$.
In receding-horizon control, the planner samples a trajectory segment $\tau$ at each environment step and executes only the first action $a_0$.

\noindent \tbm{Trajectory horizon $H$ vs.\ diffusion horizon $T$}
$H$ is the {planning horizon} (trajectory segment length) used by the diffusion planner/policy.
$T$ is the {diffusion horizon} (number of reverse-diffusion denoising steps in the sampler).
In our experiments, each base method uses its native $H$ and $T$ settings Muninn {does not change} $H$ or $T$, but reduces the {number of denoiser evaluations} by caching/reuse.

\noindent \tbm{Context $c$ (conditioning signal)}
Across D4RL tasks, the conditioning context $c$ contains:
\begin{itemize}
\item the current environment observation/state at planning time (the first state of the segment), denoted $s_{\mathrm{now}}$;
\item either a goal specification (navigation-style tasks) or a scalar performance target (return-conditioned planners), depending on the underlying baseline model.
\end{itemize}
Concretely:
\begin{itemize}
\item For return-conditioned planners, $c=(s_{\mathrm{now}}, R_{\mathrm{cond}})$, where $R_{\mathrm{cond}}$ is the desired (discounted or undiscounted) return-to-go used by the planner.
\item For goal-conditioned planners, $c=(s_{\mathrm{now}}, g)$, where $g$ encodes the task goal (a 2D target position for Maze2D/AntMaze).
\end{itemize}
When a baseline does not use an explicit goal/return conditioning channel, $c$ reduces to the observation history required by that method (i.e., just $s_{\mathrm{now}}$).

\noindent \tbm{Dataset construction and preprocessing (common)}
For each D4RL environment, we load the corresponding offline dataset $\mathcal{D}$ and construct fixed-length training segments of length $H$ by slicing contiguous sequences:
\begin{itemize}
\item We treat \texttt{terminals} as hard episode ends.
\item We treat \texttt{timeouts} as episode ends whenever the dataset uses timeouts to delimit sub-trajectories (notably Maze2D).
\item We only form segments fully contained within an episode boundary (no cross-episode stitching).
\end{itemize}
We apply the following preprocessing consistently:
\begin{itemize}
\item \textit{Observation normalization:} each observation dimension is standardized using mean and standard deviation computed over the training dataset observations. The same statistics are used at evaluation time.
\item \textit{Action scaling:} actions are clipped to the environment action bounds and scaled to match the action parameterization expected by the baseline implementation (the native $[-1,1]$ action range for MuJoCo-style tasks).
\item \textit{Goal encoding (when present):} goals are encoded in the same coordinate system as the environment state and, if concatenated to conditioning inputs, are normalized by fixed environment bounds used by the task definition.
\item \textit{Return conditioning (when present):} we compute the per-timestep return-to-go used by return-conditioned baselines from the dataset rewards using the baseline's discount convention, and apply the baseline's standard return scaling to keep conditioning magnitudes numerically stable.
\end{itemize}

\noindent \tbm{Task metric (D4RL normalized score)}
For all tasks labeled ``D4RL'' in Table~\ref{tab:results-offline}, the reported task metric is the D4RL normalized score, computed from the undiscounted episode return (sum of environment rewards over the episode) and the environment-specific random/expert reference scores.
For the Kuka stacking benchmark, the task metric is {success rate} instead.

\paragraph{MuJoCo locomotion (D4RL)}\texttt{HalfCheetah}, \texttt{Hopper}, \texttt{Walker2d}:
\label{app:envs:d4rl:mujoco}
These are continuous-control MuJoCo locomotion benchmarks where the agent must generate joint torques (actions) to achieve stable forward locomotion.
We use the standard D4RL environment instances corresponding to these domains.

\begin{figure}[H]
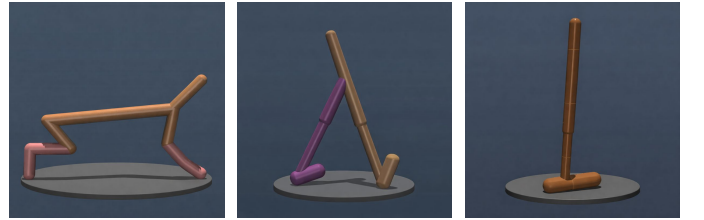

  \centering
  \makebox[\linewidth][c]{%
    \includegraphics[width=0.32\linewidth]{halfcheetah.png}\hfill
    \includegraphics[width=0.32\linewidth]{walker.png}\hfill
    \includegraphics[width=0.32\linewidth]{hopper.png}%
  }
  \caption{D4RL locomotion environments}
\end{figure}

\noindent \tbm{State/action spaces and trajectory content}
All three tasks use state vectors composed of robot joint configuration and velocity features, and continuous action vectors corresponding to actuator controls:
\begin{itemize}
\item \texttt{HalfCheetah:} $d_s=17$, $d_a=6$, so $\tau \in \mathbb{R}^{H \times 23}$ contains $H$ interleaved state--action pairs.
\item \texttt{Hopper:} $d_s=11$, $d_a=3$, so $\tau \in \mathbb{R}^{H \times 14}$.
\item \texttt{Walker2d:} $d_s=17$, $d_a=6$, so $\tau \in \mathbb{R}^{H \times 23}$.
\end{itemize}

\noindent \tbm{Horizon settings}
For HalfCheetah, Hopper, and Walker2d we use a planning horizon of \textit{$H=32$} state--action steps, which is the standard configuration used by diffusion trajectory planners on MuJoCo locomotion.
The diffusion horizon $T$ (reverse-diffusion denoising steps) is method-specific and is exactly the {Full} ``\#Evals/$t$'' value reported for each model in Table~\ref{tab:results-offline}:
Diffuser uses $T=100$,
Decision Diffusion uses $T=20$,
Diff-QL uses $T=10$ (HalfCheetah/Hopper) and $T=25$ (Walker2d),
and AdaptDiff uses $T=60$.

\noindent \tbm{Context $c$}
For locomotion, $c$ always contains $s_{\mathrm{now}}$.
If the baseline is return-conditioned, $c$ additionally contains $R_{\mathrm{cond}}$ (desired return-to-go) in the baseline's standard scaled units.

\noindent \tbm{Dataset details}
We use the standard D4RL v2 {``medium''} datasets:
\texttt{halfcheetah-medium-v2}, \texttt{hopper-medium-v2}, and \texttt{walker2d-medium-v2}.
Each dataset is collected by a {partially trained SAC policy} following the D4RL benchmark construction and contains the offline tuples
$(s_t,a_t,r_t,\texttt{terminal}_t,\texttt{timeout}_t)$.
When constructing fixed-length training segments of length $H$, we treat both \texttt{terminal} and \texttt{timeout} flags as episode boundaries (no cross-episode stitching).

\paragraph{Goal-reaching navigation (D4RL): Maze2D and AntMaze}
\label{app:envs:d4rl:navigation}
These are navigation benchmarks in which the agent must reach a target location in a maze.

\begin{figure}[H]
  \centering
  \makebox[\linewidth][c]{%
    \includegraphics[width=0.38\linewidth]{maze.png}
    \includegraphics[width=0.38\linewidth]{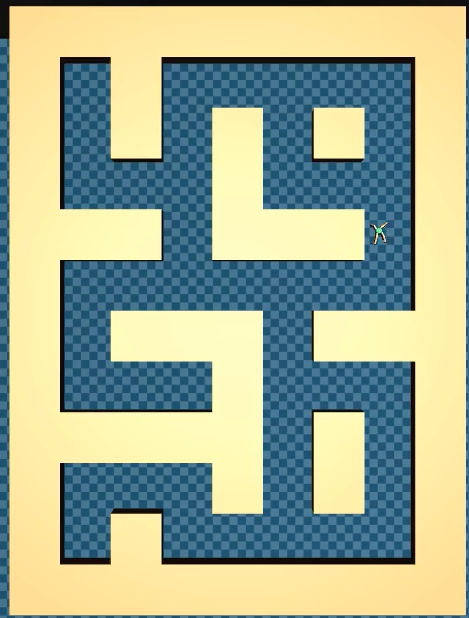}%
  }
  \caption{D4RL goal reaching environments}
\end{figure}

\noindent \tbm{Maze2D (D4RL)}
We evaluate on the D4RL Maze2D {large} benchmark \texttt{maze2d-large-v1}, a force-actuated point-mass navigation problem in a 2D maze.
The agent must reach a target $(x,y)$ location while navigating around walls.

{State/action spaces and trajectory content}
\begin{itemize}
\item \textit{State:} $s_t=(x_t,y_t,\dot{x}_t,\dot{y}_t)\in\mathbb{R}^4$ so $d_s=4$.
\item \textit{Action:} $a_t=(f^x_t,f^y_t)\in\mathbb{R}^2$ so $d_a=2$.
\item \textit{Trajectory:} $\tau\in\mathbb{R}^{H\times 6}$ contains $H$ interleaved state--action pairs, as in~\eqref{eq:app:traj_repr_d4rl}.
\end{itemize}

\textit{Horizon settings}
We use a planning horizon of \textit{$H=265$} steps for \texttt{maze2d-large-v1}.
The diffusion horizon $T$ is method-specific and is the {Full} ``\#Evals/$t$'' value in Table~\ref{tab:results-offline} for each planner (Diffuser: $100$; Decision Diffusion: $20$; Diff-QL: $10$; AdaptDiff: $60$; CompDiff: $150$).

\textit{Dataset structure and episode boundaries}
The Maze2D dataset is generated by chaining goal-reaching segments; the dataset uses the \texttt{timeout} field to delimit sub-trajectories corresponding to goal completion (the environment does not terminate on success).
Accordingly, we treat \texttt{timeout} as an episode boundary when slicing training segments.

\textit{Context $c$ and goal encoding}
Maze2D is goal-conditioned.
At evaluation time, the environment provides a goal $g=(x_g,y_g)\in\mathbb{R}^2$ in maze/world coordinates.
We condition diffusion planners on $c=(s_{\mathrm{now}},g)$, where $s_{\mathrm{now}}$ is the current state at planning time.
We normalize $g$ using the same affine normalization applied to the position components of $s_t$ (dataset mean/std) so goal magnitudes are on the same scale as normalized observations.

\textit{Preprocessing}
Observations are standardized using dataset mean/std.
Actions are clipped to the environment action bounds and scaled to the action range expected by the baseline implementation.

\noindent \tbm{AntMaze (D4RL)}
We evaluate on the D4RL AntMaze \texttt{large-play} benchmark \texttt{antmaze-large-play-v0}, a quadruped (ant) navigation task in a maze.
The agent must reach a 2D goal location in the maze while controlling joint torques.

\textit{State/action spaces and trajectory content}
\begin{itemize}
\item \textit{State:} $s_t\in\mathbb{R}^{29}$, consisting of the ant's kinematic configuration and velocities, so $d_s=29$.
\item \textit{Action:} $a_t\in\mathbb{R}^{8}$ continuous actuator controls (joint torques), so $d_a=8$.
\item \textit{Trajectory:} $\tau\in\mathbb{R}^{H\times 37}$ contains $H$ interleaved state--action pairs, as in~\eqref{eq:app:traj_repr_d4rl}.
\end{itemize}

\textit{Horizon settings}
We use a planning horizon of {$H=256$} state--action steps for \texttt{antmaze-large-play-v0}, which is a standard long-horizon diffusion-planning segment length for AntMaze.
The diffusion horizon $T$ is method-specific and matches the \textit{Full} ``\#Evals/$t$'' entries in Table~\ref{tab:results-offline} (Diffuser: $100$; Decision Diffusion: $20$; Diff-QL: $25$; AdaptDiff: $60$; CompDiff: $150$).

\textit{Context $c$ and goal encoding}
AntMaze is goal-conditioned.
The goal $g\in\mathbb{R}^2$ is specified as a target $(x_g,y_g)$ position in maze coordinates.
We condition planners on $c=(s_{\mathrm{now}},g)$.
We normalize $g$ using the same affine normalization as the $(x,y)$ components in the ant's state (dataset mean/std) so goal and state channels are comparable in magnitude.

\textit{Dataset details and episode boundaries}
We load the D4RL dataset \texttt{antmaze-large-play-v0}.
When slicing fixed-length training segments, we respect episode boundaries defined by the dataset's \texttt{timeouts} (time-limit truncation) and any \texttt{terminals} (when present).

\paragraph{Long-horizon manipulation (D4RL): FrankaKitchen}
\label{app:envs:d4rl:kitchen}

FrankaKitchen is a multi-object, multi-stage manipulation benchmark in a simulated kitchen with a Franka robot.
The environment contains multiple interactable objects (microwave, kettle, overhead light, cabinets, oven), and the task requires achieving a set of goal interactions within an episode.

\begin{figure}
  \centering
  \makebox[\linewidth][c]{%
    \includegraphics[width=0.48\linewidth]{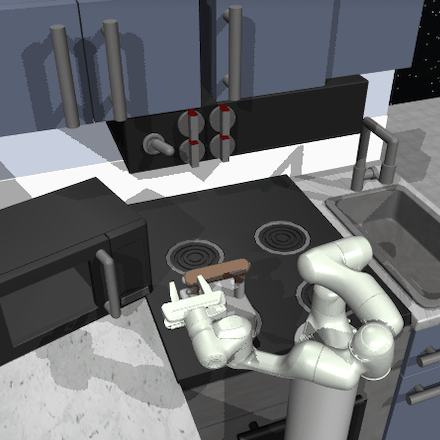}%
  }
  \caption{D4RL long horizon environment}
\end{figure}

\noindent \tbm{State/action spaces and trajectory content}
We use the standard D4RL FrankaKitchen state/action specification:
\begin{itemize}
\item \textit{State:} $s_t\in\mathbb{R}^{60}$, comprising the robot configuration/proprioception and relevant object state features.
\item \textit{Action:} $a_t\in\mathbb{R}^{9}$ (Franka 9-DoF control signal used by the benchmark).
\item \textit{Trajectory:} $\tau\in\mathbb{R}^{H\times 69}$ contains interleaved state--action pairs.
\end{itemize}

\noindent \tbm{Context $c$}
We evaluate FrankaKitchen on the D4RL benchmark \texttt{kitchen-complete-v0}, where each episode consists of a sequence of manipulation interactions that together form the ``complete'' task.
All planners condition on the current state at planning time, $s_{\mathrm{now}}$.
For the return-conditioned diffusion planners in Table~\ref{tab:results-offline} (including Diffuser-style planners), we use
$c=(s_{\mathrm{now}}, R_{\mathrm{cond}})$, where $R_{\mathrm{cond}}$ is set to the {maximum episode return observed in the training dataset} for \texttt{kitchen-complete-v0}.

\noindent \tbm{Horizon settings and diffusion horizon}
FrankaKitchen episodes have a fixed time limit of \textit{280} environment steps; we therefore use a planning horizon of \textit{$H=280$} state--action steps so the diffusion planner represents a full episode-length segment.
The diffusion horizon $T$ is method-specific and matches Table~\ref{tab:results-offline} (Full ``\#Evals/$t$''):
Diffuser uses $T=100$, Decision Diffusion uses $T=20$, and AdaptDiff uses $T=60$ on \texttt{kitchen-complete-v0}.

\paragraph{Robot block stacking (Kuka)}
\label{app:envs:kuka}

This benchmark evaluates long-horizon block stacking with a simulated KUKA arm.
The objective is to pick and place blocks to form a stable stack under contact-rich dynamics.
We use the standard {unconditional} stacking setting from diffusion-planning benchmarks (no explicit goal vector; success is defined by the final stacked configuration).

\begin{figure}[H]
  \centering
  \makebox[\linewidth][c]{%
    \includegraphics[width=0.48\linewidth]{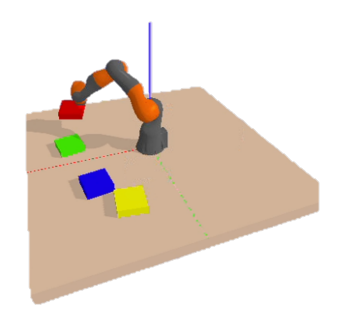}%
  }
  \caption{Kuka stacking environment}
\end{figure}

\noindent \tbm{State/action spaces and trajectory content}
We model stacking as state--action trajectory diffusion, with:
\begin{itemize}
\item \textit{State:} $s_t\in\mathbb{R}^{39}$ (unconditional stacking state vector), which contains robot configuration/proprioception and object pose features required by the benchmark.
\item \textit{Action:} $a_t\in\mathbb{R}^{11}$, a continuous control vector used by the benchmark controller.
In this suite, actions command the arm with a fixed-dimensional continuous control signal (including a grasp/close channel) as implemented by the environment.
\item \textit{Trajectory:} $\tau = [(s_0,a_0),\dots,(s_{H-1},a_{H-1})]\in\mathbb{R}^{H\times 50}$.
\end{itemize}

\noindent \tbm{Episode length, planning horizon $H$, and diffusion horizon $T$}
Episodes have a fixed horizon of \textit{384} environment steps in this benchmark.
Following common diffusion-planning practice on this task family, we plan over a \textit{$H=192$}-step trajectory segment and execute in a receding-horizon loop.
The diffusion horizon $T$ matches Table~\ref{tab:results-offline}: Diffuser uses $T=100$ denoising steps and Decision Diffusion uses $T=20$.

\noindent \tbm{Context $c$}
Stacking is \textit{not goal-conditioned} in this benchmark variant.
We condition on the current state only, $c=s_{\mathrm{now}}$.
For return-conditioned diffusion planners, the return-conditioning scalar is included in $c$ as $c=(s_{\mathrm{now}},R_{\mathrm{cond}})$ with $R_{\mathrm{cond}}$ set to the maximum training-set return for the stacking dataset (standard high-return conditioning).

\noindent \tbm{Dataset details and preprocessing}
We use the offline stacking dataset released with diffusion-planning benchmarks for this task, consisting of collision-free / successful stacking rollouts generated by a task-and-motion planning pipeline.
We construct fixed-length training segments of length $H$ within episode boundaries, standardize each state dimension using dataset mean/std, and clip/scale actions to the environment action bounds.

\subsubsection{Configuration-space Motion Planning (MPD/EDMP Protocol)}
\label{app:envs:motion}

Table~\ref{tab:results-motion} evaluates configuration-space planning for a 7-DoF robot arm in clutter.

\begin{figure}[H]
  \centering
  \makebox[\linewidth][c]{%
    \includegraphics[width=0.48\linewidth]{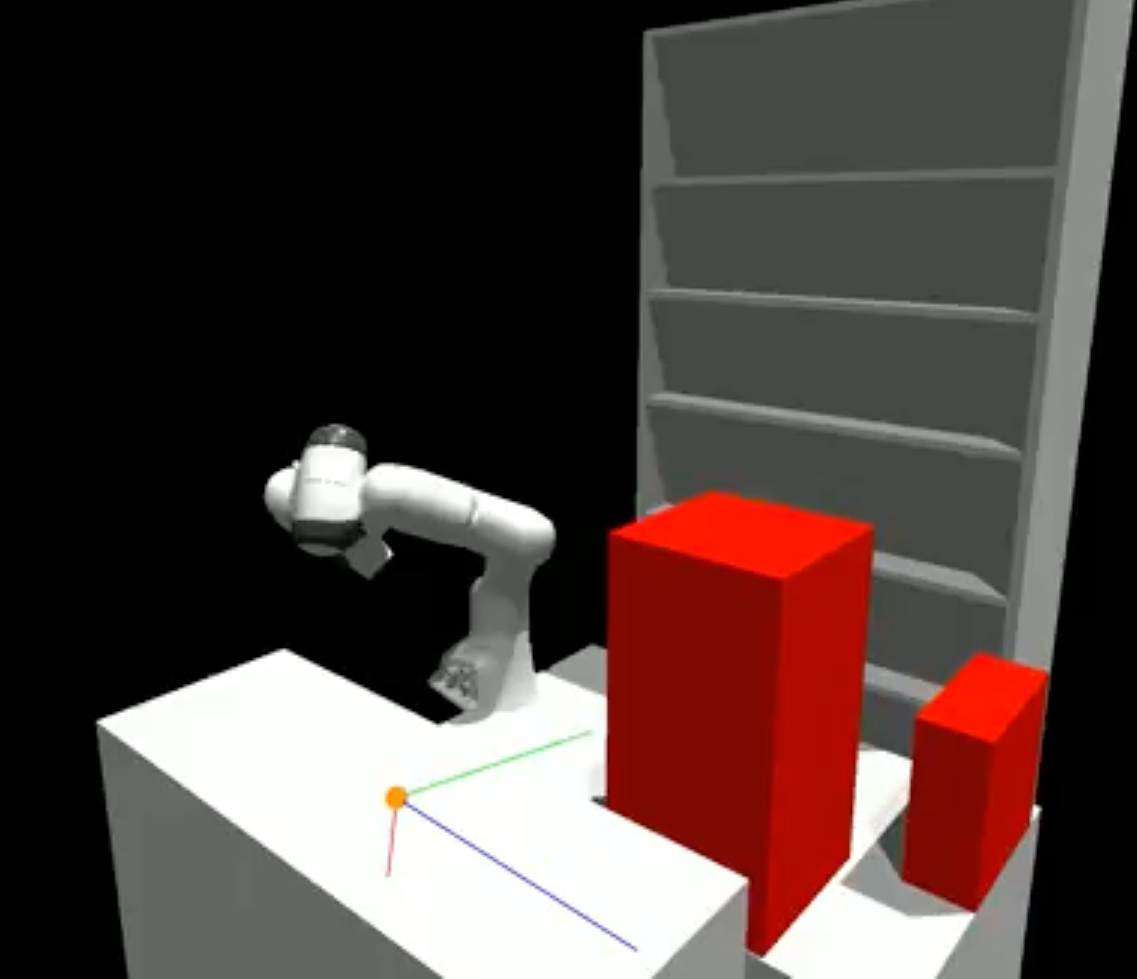}%
  }
  \caption{Clutter planning environment}
\end{figure}

\noindent \tbm{Robot model}
We plan for a \texttt{7-DoF Franka Emika Panda}-class manipulator in joint space with configuration $q\in\mathbb{R}^7$.
We use the standard Panda joint limits (radians):
\begin{align*}
q_{\min} = [-2.9,\,-1.8,\,-2.9,\,-3.0,\,-2.9,\,0.0,\,-2.9],\\
q_{\max} = [\;\;2.9,\;\;1.8,\;\;2.9,\;0.0,\;\;2.9,\;\;3.8,\;\;2.9].
\end{align*}
Collision checking uses the robot's URDF link geometries (self-collision enabled) and rigid obstacle geometries in the workspace; configurations outside $[q_{\min},q_{\max}]$ are treated as hard constraint violations.

\noindent \tbm{Scenario definition: (easy/medium/hard)}
Each planning query samples a collision-free start configuration $q_{\mathrm{start}}$ and goal configuration $q_{\mathrm{goal}}$, plus a set of static workspace obstacles.
Obstacles are axis-aligned boxes placed in a fixed workspace region in front of the robot base:
\[
x\in[0.25,\,0.75]\text{ m}, y\in[-0.35,\,0.35]\text{ m}, z\in[0.05,\,0.55]\text{ m}.
\]
Each box side length is sampled independently as $\ell_x,\ell_y,\ell_z \sim \mathrm{Unif}(0.04,\,0.12)$ meters, and boxes are rejected if they intersect the robot at $q_{\mathrm{start}}$ or $q_{\mathrm{goal}}$.
Difficulty levels follow the MPD/EDMP clutter protocol via obstacle count:
\begin{itemize}
\item {Easy:} 3 obstacles.
\item {Medium:} 6 obstacles.
\item {Hard:} 10 obstacles.
\end{itemize}
To ensure queries are nontrivial but feasible, we additionally require that (i) the straight-line interpolation in joint space between $q_{\mathrm{start}}$ and $q_{\mathrm{goal}}$ is {not} collision-free (otherwise the query is resampled), and (ii) an OMPL RRT-Connect planner finds at least one collision-free solution (used both as feasibility filter and as a data generator for training trajectories).

\noindent \tbm{Trajectory representation}
A planned path is represented as a fixed-length sequence of joint configurations:
\begin{equation*}
\tau \;=\; [q_0,q_1,\dots,q_{H-1}] \in \mathbb{R}^{H\times 7},
\end{equation*}
where $q_0=q_{\mathrm{start}}$ and $q_{H-1}$ should reach $q_{\mathrm{goal}}$ within tolerance.
The diffusion model denoises this entire joint-trajectory array in joint space.

\noindent \tbm{Context $c$}
$c$ contains (i) $q_{\mathrm{start}}$ and $q_{\mathrm{goal}}$, (ii) an obstacle encoding sufficient to evaluate feasibility (the benchmark obstacle representation used by MPD/EDMP), (iii) any additional planning-time scalars required by the baseline.

\noindent \tbm{Collision/constraint violation detection}
A candidate joint-space trajectory $\tau=[q_0,\dots,q_{H-1}]$ is checked for feasibility using dense collision checking:
\begin{itemize}
\item \textit{Waypoint checks:} every waypoint $q_i$ is checked for (i) self-collision, (ii) robot--obstacle collision, and (iii) joint-limit violation.
\item \textit{Segment checks:} for each consecutive pair $(q_i,q_{i+1})$, we linearly interpolate \textit{10} intermediate configurations and collision-check each interpolated configuration (yielding 11 checks per segment including the endpoints).
\end{itemize}
An episode is counted as a {collision/violation episode} if {any} checked configuration violates any hard constraint above.
The ``Collision (\%)'' metric in Table~\ref{tab:results-motion} is the percentage of evaluation queries that are collision/violation episodes.

\noindent \tbm{Success criterion}
A planning query is successful if the method returns a trajectory that (i) is collision-free under the checks above and (ii) reaches the goal within a joint-space tolerance:
\[
\|q_{H-1}-q_{\mathrm{goal}}\|_\infty \le 0.05~\text{rad}.
\]
The ``Success (\%)'' metric in Table~\ref{tab:results-motion} is the fraction of successful queries over the evaluation set.

\subsubsection{Visuomotor Imitation and Manipulation (Diffusion Policies)}
\label{app:envs:policies}

Table~\ref{tab:results-policy} evaluates diffusion {policies} that generate short-horizon action/pose segments and execute them in a receding-horizon control loop.

\noindent \tbm{Common receding-horizon execution}
All diffusion {policies} in Table~\ref{tab:results-policy} generate {action chunks} and execute them in a receding-horizon loop.
At control step $k$, the policy receives an observation history of length \textit{$T_o=2$} (the two most recent observations) and samples an action {prediction horizon} of \textit{$H=T_p=16$} actions:
\[
(a_k,a_{k+1},\dots,a_{k+T_p-1}) \sim \pi_\theta(\cdot \mid o_{k-T_o+1:k}).
\]
We then execute the first \textit{$T_a=8$} actions open-loop (action chunking) and replan every $T_a$ control steps.
Thus, policy inference is called at frequency $f_{\pi}=f_{\mathrm{ctrl}}/T_a$, where $f_{\mathrm{ctrl}}$ is the environment/controller rate specified per benchmark below.
For diffusion policies, the reverse-diffusion sampler uses:
\begin{itemize}
\item \textit{Diffusion Policy (2D):} $T=100$ denoising steps at inference (DDPM-style).
\item \textit{DP3 (3D):} $T=10$ denoising steps at inference using a DDIM sampler (the standard DP3 inference setting).
\end{itemize}
Latency in Table~\ref{tab:results-policy} is measured per {policy inference call} (one action-chunk generation).

\paragraph{RLBench Reach Target}
\label{app:envs:rlbench}
The robot must move its end-effector to a designated target pose/location in the workspace within a fixed tolerance.

\begin{figure}[H]
  \centering
  \makebox[\linewidth][c]{%
    \includegraphics[width=0.48\linewidth]{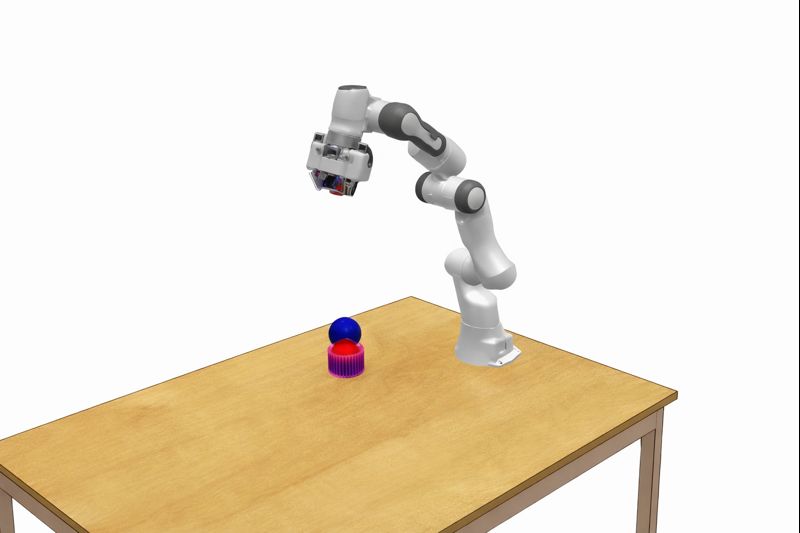}%
  }
  \caption{RLBench Reach Target environment}
\end{figure}

\noindent \tbm{Observations}
We use the standard RLBench multi-view RGB observation setup for Reach Target with \textit{four} fixed cameras:
\texttt{front}, \texttt{left\_shoulder}, \texttt{right\_shoulder}, and \texttt{wrist}.
Each camera produces an RGB frame at \textit{raw resolution $256\times 256$}.
Before entering the policy, each RGB image is resized with bilinear interpolation to \textit{$84\times 84$} and normalized to floating-point pixels in $[0,1]$ (with per-channel normalization matching the base Diffusion Policy implementation).
We stack the most recent \textit{$T_o=2$} timesteps of images, so the vision input at time $k$ consists of $4$ camera views $\times$ $2$ frames.

In addition to images, the observation includes low-dimensional robot state (proprioception) comprising the 7 joint positions, 7 joint velocities, and the gripper open/close state, concatenated as a vector and normalized by dataset mean/standard deviation.
The Reach Target task also provides the target specification as a 3D position in workspace coordinates; this target position is concatenated to the non-image conditioning vector and normalized by dataset mean/std computed over target positions in the training set.

\noindent \tbm{Actions}
Actions are \textit{7-dimensional} end-effector delta-pose commands executed by the RLBench controller:
\[
a_t = (\Delta x,\Delta y,\Delta z,\Delta\phi_x,\Delta\phi_y,\Delta\phi_z,\,a^{\mathrm{grip}})\in\mathbb{R}^7,
\]
where $(\Delta x,\Delta y,\Delta z)$ are Cartesian translations (meters) in the robot base frame,
$(\Delta\phi_x,\Delta\phi_y,\Delta\phi_z)$ are small-angle rotations (radians, axis-angle parameterization),
and $a^{\mathrm{grip}}\in[-1,1]$ controls the gripper (open/close).
All components are clipped to the action bounds used by the RLBench environment and linearly scaled to the $[-1,1]$ range expected by the diffusion policy implementation.

\noindent \tbm{Control frequency}
The RLBench environment is stepped at a fixed control rate of {20\,Hz}.
We use action chunking with $T_a=8$ (from the common policy setup above), so the diffusion policy is queried at {2.5\,Hz} and the resulting 8-step action chunk is executed open-loop at 20\,Hz before replanning.

\paragraph{Meta-World pick-place-v2}
\label{app:envs:metaworld}
The robot must pick up an object and place it at a target location.

\begin{figure}[H]
  \centering
  \makebox[\linewidth][c]{%
    \includegraphics[width=0.48\linewidth]{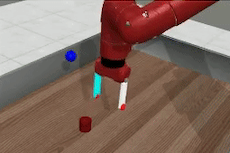}%
  }
  \caption{Meta-world pick-place environment}
\end{figure}

\noindent \tbm{Observations}
We use the \textit{state-based} MetaWorld v2 observation for \texttt{metaworld/pick-place-v2}, which is a \textit{39-dimensional} vector ($d_s=39$).
This vector includes the end-effector position, gripper opening, object pose features (position and orientation), object velocities, and the task goal coordinates as provided by the MetaWorld environment.
We use the full 39D observation without dropping fields.
Each observation dimension is standardized using mean and standard deviation computed over the training demonstrations.
We stack an observation history of \textit{$T_o=2$} timesteps (concatenation of the two most recent 39D vectors) as the diffusion policy conditioning input.

\noindent \tbm{Actions and control}
The action representation and control frequency must be reported explicitly:
whether the benchmark uses end-effector delta-position with a gripper scalar, end-effector delta-pose, or joint controls,
and the exact execution rate of the environment control loop (Hz).
We execute the diffusion policy in receding horizon with observation horizon $T_o$ and action horizon $T_a$ as specified in the experiment configuration.

\paragraph{DP3 Pour (DP3)}
\label{app:envs:dp3}
The robot must pour contents from a source container into a target container, requiring precise 3D spatial reasoning and contact-rich control.

\begin{figure}[H]
  \centering
  \makebox[\linewidth][c]{%
    \includegraphics[width=0.48\linewidth]{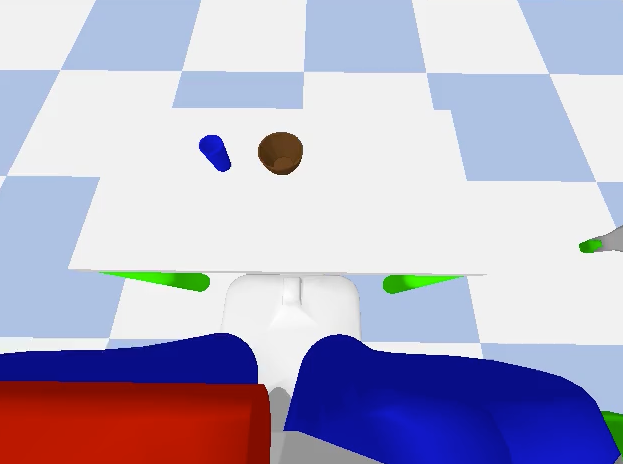}%
  }
  \caption{Custom DP3 pour environment}
\end{figure}

\noindent \tbm{Observations}
We were unable to replicate the results in the DP3 paper (simulator mismatch), hence, we followed the protocol, trained the model on a custom DP3 pour environment and tested on the same. 
DP3 conditions on a compact 3D scene representation derived from a {single-view depth camera}.
At each control step, we convert the depth image into a point cloud using the camera intrinsics and extrinsics, transform points into the world frame, and crop points to a fixed workspace-aligned bounding box that contains the robot arm and manipulation region.
We then downsample the cropped point cloud to an input size of {1024 points} using {iterative farthest point sampling}, and use only the $(x,y,z)$ coordinates to improve appearance generalization.

In addition to the point cloud, DP3 conditions on low-dimensional robot state (end-effector pose and gripper state in the benchmark's standard representation), normalized by dataset statistics.
We stack an observation history of \textit{$T_o=2$} timesteps for both the point cloud features and the robot state features as the conditioning context for action generation.

\subsection{Hardware: Environment, Task, and Platform Details}
\label{app:hardware}

We evaluate Muninn in real-time closed-loop deployments spanning marine navigation, aerial navigation, and tabletop manipulation. Across all platforms, we run planners/policies in a receding-horizon loop: at each planning cycle we (i) acquire the latest observation/state estimate, (ii) sample a trajectory (or action segment) using the diffusion model (full or \nordy{Muninn}-accelerated), and (iii) convert this output into low-level control inputs via a platform-specific tracking controller. Unless stated otherwise, we generate {one} candidate trajectory per planning cycle (batch size $=1$) and execute it in receding horizon.

\vspace{-0.25em}
\begin{table}[H]
\centering
\scriptsize
\setlength{\tabcolsep}{2pt}
\renewcommand{\arraystretch}{1.12}

\begin{threeparttable}
\begin{tabularx}{\columnwidth}{@{}%
  L{0.30\columnwidth}%
  L{0.14\columnwidth}%
  Y%
  C{0.12\columnwidth}%
  C{0.12\columnwidth}%
@{}}

  \arrayrulecolor{HeadGray}\toprule\arrayrulecolor{black}

  \rowcolor{MuninnBlue!10}
  \theadb{Platform} &
  \theadb{Task} &
  \theadb{Sensors (closed loop)} &
  \makecell[c]{\theadb{Control}\\\theadb{rate}} &
  \makecell[c]{\theadb{Planning}\\\theadb{rate}} \\

  \midrule

  \makecell[l]{\textit{SeaRobotics}\\\textit{Surveyor ASV}} &
  \textit{2D nav} &
  RTK-GNSS, IMU, LiDAR, day camera, night camera, sonde &
  \makecell[c]{20 Hz\\\textcolor{HeadGray}{(tracking)}} &
  \makecell[c]{1 Hz\\\textcolor{HeadGray}{(replan)}} \\

  \addlinespace[0.2em]

  \makecell[l]{\textit{Crazyflie}\\\textit{quadrotor}} &
  \textit{3D nav} &
  Motion capture pose, onboard IMU &
  \makecell[c]{100 Hz\\\textcolor{HeadGray}{(setpoints)}} &
  \makecell[c]{10 Hz\\\textcolor{HeadGray}{(replan)}} \\

  \addlinespace[0.2em]

  \makecell[l]{\textit{SO-ARM100}\\\textit{tabletop arm}} &
  \textit{Manip.} &
  Joint encoders, RGB camera(s) &
  \makecell[c]{20 Hz\\\textcolor{HeadGray}{(servo)}} &
  \makecell[c]{20 Hz\\\textcolor{HeadGray}{(policy)}} \\

  \arrayrulecolor{HeadGray}\bottomrule\arrayrulecolor{black}
\end{tabularx}

\vspace{-0.35em}
\caption{\textit{High-level platform summary.}
We report the control interface rate (how often low-level setpoints are issued) and the planning/policy rate (how often a diffusion trajectory/segment is resampled).}
\label{tab:hardware_summary}
\end{threeparttable}
\end{table}

\subsubsection{SeaRobotics Surveyor ASV: 2D marine navigation}
\label{app:hardware:asv}

\noindent \tbm{Platform and actuation}
The unmanned surface vehicle (USV) is a \textit{SeaRobotics Surveyor ASV} equipped with a \textit{differential-thrust propulsion module}. The platform exposes a \textit{velocity set-point interface} (commanded forward speed and heading/yaw-rate), while the low-level propulsion stack converts these setpoints to left/right thrust subject to actuator limits.

\begin{figure}[t]
  \centering
  \makebox[\linewidth][c]{%
    \includegraphics[width=0.98\linewidth]{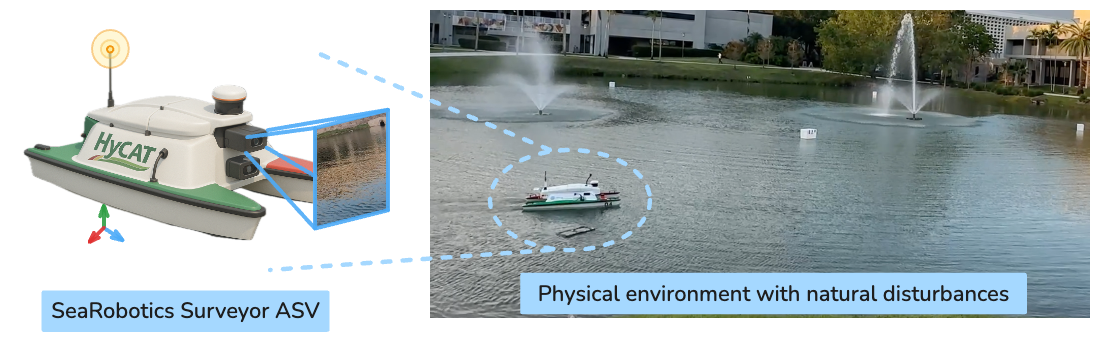}%
  }
  \caption{Marine Navigation setting.}
\end{figure}

\noindent \tbm{Sensor payload (modalities, rates, and noise model)}
The ASV carries six sensing modalities (the same payload used by the BCOD teacher baseline in our comparisons). The update rates, conservative operating ranges, and the noise parameters used by the estimator are:
\vspace{-0.25em}
\begin{table}[H]
\centering
\scriptsize
\setlength{\tabcolsep}{2pt}
\renewcommand{\arraystretch}{1.12}

\begin{threeparttable}
\begin{tabularx}{\columnwidth}{@{}%
  Y%
  C{0.12\columnwidth}%
  C{0.12\columnwidth}%
  L{0.36\columnwidth}%
@{}}

  \arrayrulecolor{HeadGray}\toprule\arrayrulecolor{black}

  \rowcolor{MuninnBlue!10}
  \theadb{Sensor} &
  \theadb{Rate} &
  \theadb{Range} &
  \theadb{Noise model (estimator)} \\

  \midrule

  Spinning LiDAR (32-line) &
  10 Hz &
  120 m &
  additive range noise, std $=3$ cm \\

  \addlinespace[0.2em]

  \makecell[l]{Global-shutter RGB\\camera (day)} &
  20 Hz &
  80 m &
  pixel noise, $1$ px $\approx 2$ cm at $10$ m \\

  \addlinespace[0.2em]

  \makecell[l]{NIR-augmented\\mono camera (night)} &
  20 Hz &
  40 m &
  pixel noise, $1$ px \\

  \addlinespace[0.2em]

  \makecell[l]{Exo-conductivity/\\temperature sonde} &
  2 Hz &
  spot &
  depth-only probe; used as exteroceptive cue \\

  \addlinespace[0.2em]

  RTK-capable GNSS receiver &
  5 Hz &
  global &
  RTK-fix position noise, std $=1.5$ cm \\

  \addlinespace[0.2em]

  6-axis MEMS IMU &
  200 Hz &
  n/a &
  gyro noise std $=0.02$ rad/s \\

  \arrayrulecolor{HeadGray}\bottomrule\arrayrulecolor{black}
\end{tabularx}

\vspace{-0.25em}
\caption{\textit{ASV sensor suite.}
We report the sensing modalities used during closed-loop operation, with the estimator noise parameters used for filtering.}
\label{tab:asv_sensors}
\end{threeparttable}
\end{table}

\noindent \tbm{Control and planning frequencies}
\begin{itemize}
\item \textit{Planning frequency:} $1$ Hz. At each planning cycle we sample a short-horizon trajectory and update the reference segment used by the tracker.
\item \textit{Tracking/control frequency:} $20$ Hz. We track the planned reference using a linear MPC.
\end{itemize}

\noindent \tbm{Trajectory-to-control integration (real-time closed-loop)}
Each planning cycle produces a short-horizon \textit{2D trajectory segment} in the global frame. We convert the sampled trajectory into executed controls as follows:
\vspace{-0.25em}
\begin{enumerate}
\item \textit{Reference construction:} resample the planned segment into an ordered set of waypoints with timestamps over a $2$ s horizon.
\item \textit{Low-level tracking:} solve a \textit{linear MPC} with a $2$ s horizon whose inputs are constrained to the thrust envelope; the MPC tracks an \textit{eight-waypoint} reference segment. The MPC outputs velocity setpoints $(v,\omega)$ at $20$ Hz.
\item \textit{Execution:} apply only the \textit{next} MPC command (receding-horizon control). At the next $1$ s planning instant, replan and refresh the reference.
\end{enumerate}

\noindent \tbm{Hardware evaluation protocol}
We run closed-loop waypoint-navigation episodes per method.
\begin{itemize}
\item \textit{Initial conditions:} each episode starts from a stationary pose near a randomly selected start region on the lake. The vehicle is reset between episodes by piloting it back to the start region.
\item \textit{Goal distribution:} each episode samples a goal disk with radius $6$ m uniformly over a navigable region of radius $100$ m (world radius), rejecting goals whose disks intersect the shoreline buffer.
\item \textit{Disturbances:} episodes are conducted under natural environmental disturbances (wind, waves, and surface currents). The planner is not provided with a disturbance model.
\item \textit{Termination:} an episode terminates on {success} (EKF mean enters the goal disk and remains inside for $3$ s), {collision/violation} (defined below), or {timeout} ($180$ s).
\end{itemize}

\noindent \tbm{Collision / constraint violation definition (ASV)}
A run is marked as a collision/violation if any of the following occurs:
\begin{itemize}
\item \textit{Shoreline breach:} the EKF mean position crosses the shoreline polygon (hard constraint) or enters the shoreline safety buffer for more than $1$ s.
\item \textit{Obstacle contact:} physical contact with a buoy or fixed obstacle, identified by operator observation and confirmed by an abrupt deceleration spike in the EKF velocity estimate.
\end{itemize}

\noindent \tbm{Latency measurement on hardware (ASV)}
We measure planner latency on the process that performs inference (the planning node):
\begin{itemize}
\item We time end-to-end {planning-call} wall clock from receiving the latest belief/state snapshot to producing a trajectory segment, including \nordy{Muninn} probe computation and any denoiser evaluations.
\item We exclude network transmission time to the actuator interface and the MPC solve time (we report planning latency only, matching the definition used in Tables~\ref{tab:results-policy} and the hardware table in Fig.~\ref{fig:hw2}).
\end{itemize}

\subsubsection{Crazyflie quadrotor: 3D aerial navigation}

\noindent \tbm{Platform and sensing}
We use a \textit{Crazyflie} micro-quadrotor for 3D waypoint navigation.
The onboard sensing includes:
\begin{itemize}
\item \textit{Inertial sensing:} 6-axis IMU for attitude-rate estimation (used by the onboard stabilizer).
\item \textit{Global pose:} an external \textit{motion capture system} provides $SE(3)$ pose $(x,y,z,\psi)$ in a global frame. Pose is streamed to the robot via radio and used by the position controller.
\end{itemize}

\label{app:hardware:uav}
\begin{figure}[H]
  \centering
  \makebox[\linewidth][c]{%
    \includegraphics[width=0.78\linewidth]{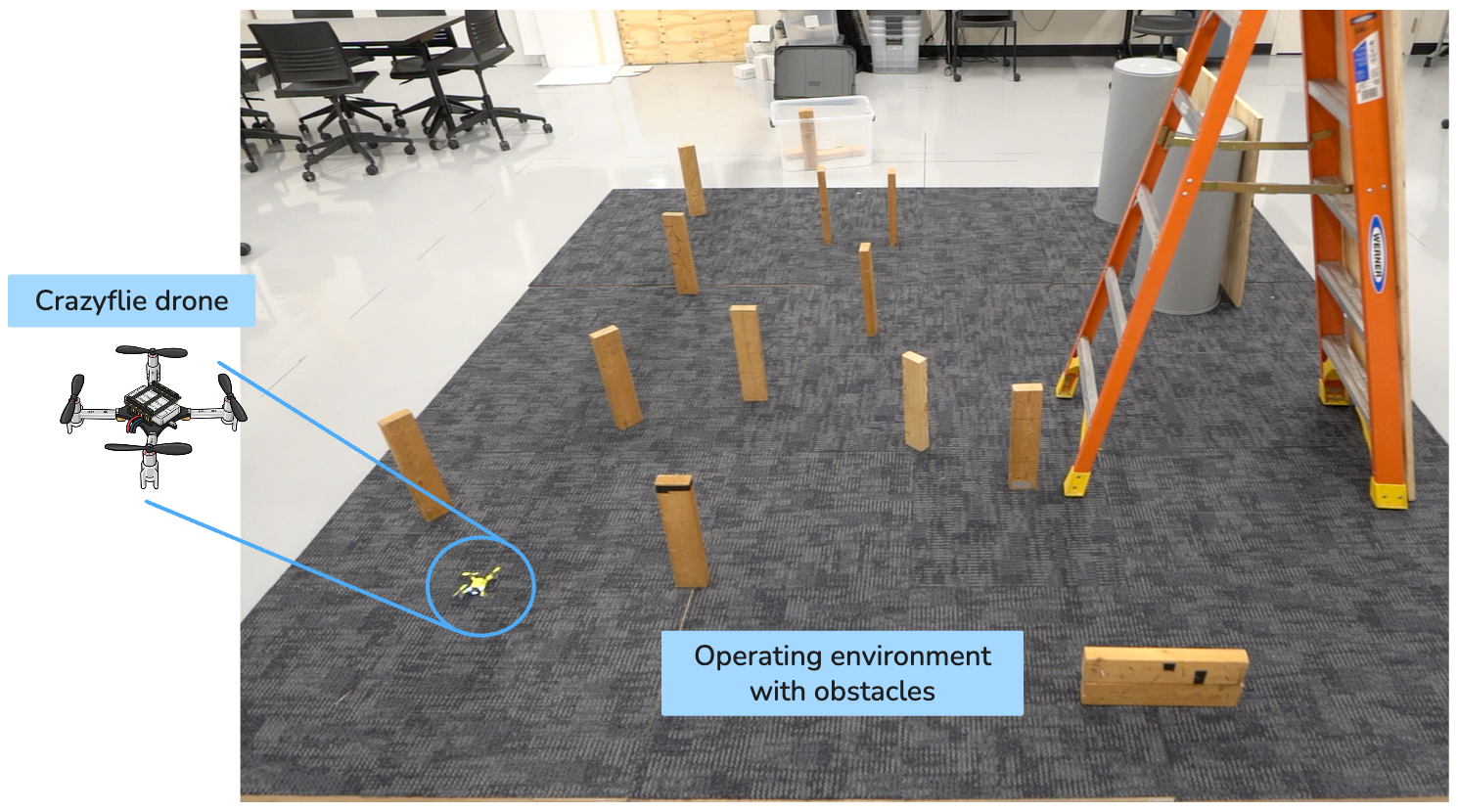}%
  }
  \caption{Crazyflie operating environment}
\end{figure}

\noindent \tbm{Control and planning frequencies}
\begin{itemize}
\item \textit{Planning frequency:} $10$ Hz. We replan in receding horizon, producing a short 3D trajectory segment each cycle.
\item \textit{Command rate:} $100$ Hz. We send position (or velocity) setpoints to the onboard controller at $100$ Hz; between replans, the setpoint stream tracks the most recent planned segment.
\end{itemize}

\noindent \tbm{Trajectory-to-control integration}
The diffusion planner outputs a \textit{3D waypoint trajectory} in the global frame. We execute it using a standard cascade:
\begin{enumerate}
\item \textit{Reference generation:} fit a time-parameterized waypoint spline through the planned positions (piecewise-linear in time with bounded velocity and acceleration).
\item \textit{Position tracking:} at $100$ Hz, query the spline at the current time and send the resulting position/yaw setpoint to the Crazyflie onboard controller.
\item \textit{Receding horizon:} replan at $10$ Hz; at each replan, reset the spline time origin and continue tracking the updated plan.
\end{enumerate}

\noindent \tbm{Hardware evaluation protocol}
We execute waypoint-navigation runs per method in a fixed indoor arena.
\begin{itemize}
\item \textit{Start/goal sampling:} sample start and goal positions uniformly from a rectangular prism (workspace), rejecting samples within obstacle inflation margins. Goals are separated from starts by at least $1.0$ m Euclidean distance.
\item \textit{Termination:} success if the quadrotor enters a goal ball of radius $0.15$ m and remains inside for $1.0$ s; collision if it enters an inflated obstacle volume; timeout at $30$ s.
\item \textit{Resets:} after each episode, the quadrotor lands; the next run begins after takeoff to the sampled start pose.
\end{itemize}

\noindent \tbm{Latency measurement on hardware (UAV)}
Latency is measured on the planning node:
\begin{itemize}
\item \textit{Compute-only timing:} wall-clock time to produce one planned 3D segment from the latest state estimate, including \nordy{Muninn} probe and sampling, excluding radio transmission time and low-level controller execution.
\item \textit{Synchronization:} measurements use explicit device synchronization when running on GPU.
\end{itemize}

\subsubsection{SO-ARM100 tabletop arm: manipulation}
\label{app:hardware:arm}

\noindent \tbm{Platform and sensing}
We use a \textit{SO-ARM100} 6-DoF tabletop arm with bus servos and a parallel gripper.
Closed-loop sensing includes:
\begin{itemize}
\item \textit{Proprioception:} joint encoders from the servo bus (12-bit magnetic encoders) and gripper state.
\item \textit{Vision:} RGB images from fixed cameras covering the workspace. We use a top-down camera for all tasks; for peg insertion we additionally use a side view to reduce occlusion.
\end{itemize}

\begin{figure}[t]
  \centering
  \makebox[\linewidth][c]{%
    \includegraphics[width=0.48\linewidth]{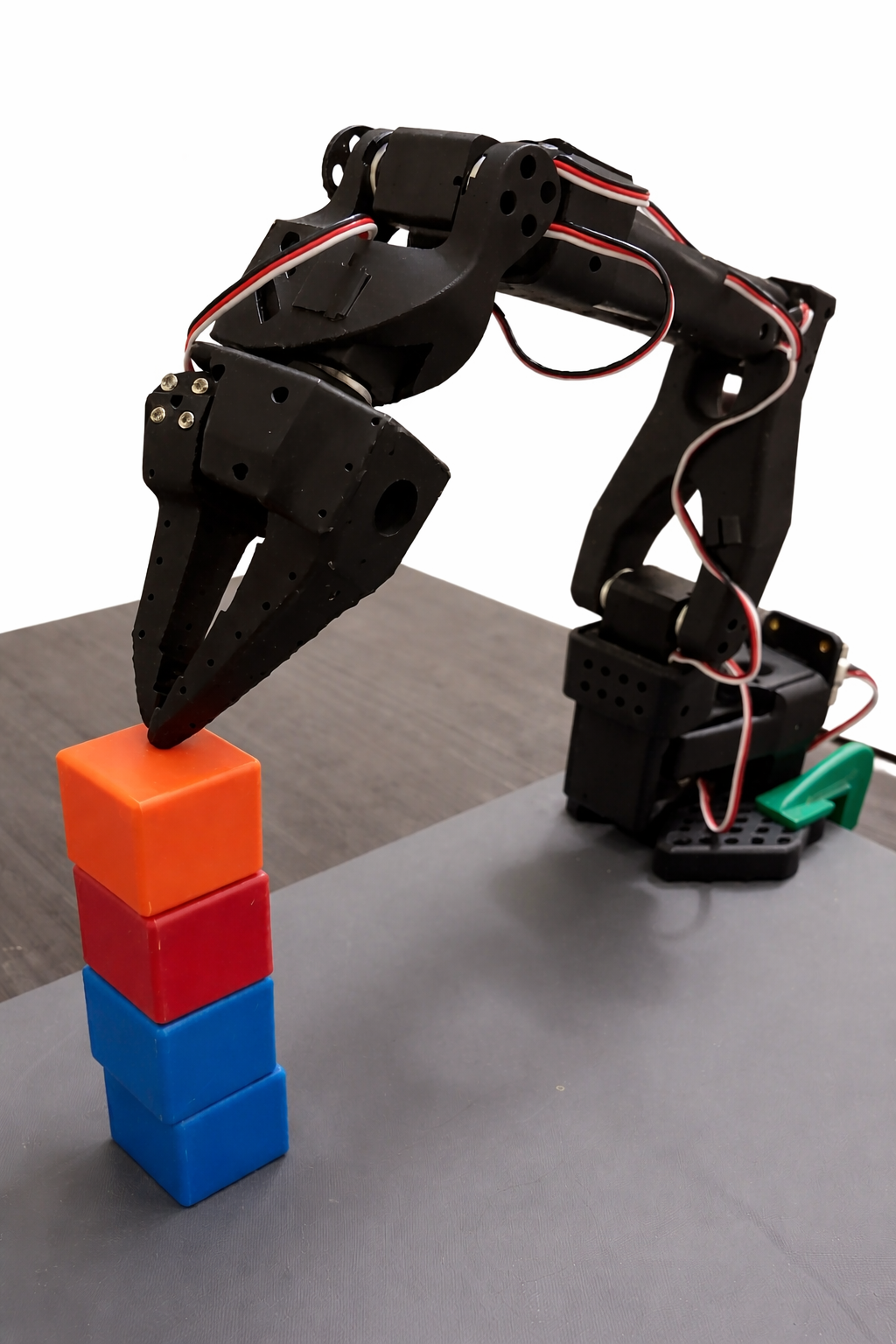}%
  }
  \caption{SO-100 tabletop manipulation environment.}
\end{figure}

\noindent \tbm{Control and planning frequencies}
\begin{itemize}
\item \textit{Policy/planning frequency:} $20$ Hz. We run the diffusion policy at $20$ Hz in a receding-horizon loop.
\item \textit{Servo command frequency:} $20$ Hz. The arm receives joint-position targets at $20$ Hz; low-level servo control closes the internal position loop.
\end{itemize}

\noindent \tbm{Policy-to-control integration}
The diffusion policy outputs a short-horizon \textit{action segment} in Cartesian space:
\begin{itemize}
\item \textit{Action representation:} a delta end-effector pose command $(\Delta x,\Delta y,\Delta z,\Delta \phi,\Delta \theta,\Delta \psi)$ plus a binary gripper open/close action.
\item \textit{Execution:} we apply only the \textit{first} action of the sampled segment each $50$ ms control tick (receding horizon).
\item \textit{IK and limits:} delta poses are converted to joint targets via inverse kinematics; joint limits and velocity limits are enforced before sending commands.
\end{itemize}

\noindent \tbm{Hardware evaluation protocol}
We evaluate three tasks: \textit{pick-and-place}, \textit{stacking}, and \textit{peg insertion}. For each method we report the average.
\begin{itemize}
\item \textit{Resets:} objects are manually reset to a nominal distribution after each episode; the arm is returned to a home configuration.
\item \textit{Pick-and-place:} success if the object is placed inside a marked target region and released, with the object remaining in the region for $2$ s.
\item \textit{Stacking:} success if the final stack configuration is achieved and remains upright for $3$ s.
\item \textit{Peg insertion:} success if the peg reaches a minimum insertion depth and remains inserted for $2$ s.
\item \textit{Timeout:} $60$ s per episode for pick-and-place/stacking and $90$ s for peg insertion.
\end{itemize}

\noindent \tbm{Collision / violation definition (SO-ARM100)}
A run is marked as a collision/violation if any of the following occurs:
\begin{itemize}
\item the end-effector contacts the table outside the task-required contact region,
\item any joint hits a hard limit,
\item sustained tracking error triggers the contact abort (interpreted as unintended collision with the environment or self-collision).
\end{itemize}

\noindent \tbm{Latency measurement on hardware (arm)}
\begin{itemize}
\item We time the wall-clock time per $20$ Hz policy call from observation acquisition (images + proprioception snapshot) to the produced first action, including \nordy{Muninn} probe and sampling.
\item We exclude the camera driver buffering delay and transport delay to the servo bus interface.
\end{itemize}
\subsubsection{Notes on communication latency and timing scope}
\label{app:hardware:timing_scope}

For all three platforms, the \textit{latency values reported in the hardware table} are \textit{compute-only} latencies measured at the planner/policy node, and therefore \textit{exclude} (i) network or radio transmission time, (ii) actuator interface latency, and (iii) low-level controller execution time (MPC solve, servo bus write). This aligns reported latency with the algorithmic acceleration target of \nordy{Muninn}: reducing denoiser evaluations and denoising-time compute without changing the downstream controller.

\subsection{Metrics}
\label{subsec:metrics}

We evaluate and all baselines along three dimensions: (i) {task performance and safety}, (ii) {computational efficiency}, and (iii) {approximation fidelity and risk}. Each metric has a benchmark-specific instantiation but a consistent meaning across domains.

\noindent \tbm{Notation}
A {decision} is one diffusion sampling call that outputs a trajectory (planner) or an action/pose chunk (policy). Let $N_{\text{roll}}$ be the number of evaluation rollouts (episodes or planning queries, depending on the benchmark). Let $\mathcal{I}$ index all decisions encountered across these rollouts, with $|\mathcal{I}|=N_{\text{dec}}$.
For a decision $i\in\mathcal{I}$, denote the full-compute output by $\tau^{\mathrm{full}}_{0,i}$ and the cached output by $\tilde{\tau}_{0,i}$ (both produced from the {same} initial diffusion noise and (if stochastic) the same sampler-noise sequence, so the difference isolates caching).

\vspace{0.35em}
\noindent\tbm{Task performance ($TP$)}
We report a primary task metric per benchmark family:

\smallskip
\noindent\textit{(a) D4RL tasks (HalfCheetah / Hopper / Walker2d / Maze2D / AntMaze / FrankaKitchen)}
The task metric is the \textit{D4RL normalized score} computed from the undiscounted episodic return
$G=\sum_{t=0}^{T_{\text{env}}-1} r_t$.
Practically, we compute the normalized score using the environment-provided normalization routine:
\begin{equation}
TP \;=\; \texttt{env.get\_normalized\_score}(G),
\end{equation}
which implements the D4RL convention
\begin{equation}
\text{NormalizedScore}(G)
\;=\;
100 \cdot \frac{G - G_{\text{random}}}{G_{\text{expert}} - G_{\text{random}}},
\end{equation}
with task-specific $(G_{\text{random}},G_{\text{expert}})$ constants defined by the benchmark. We report the mean normalized score over $N_{\text{roll}}=50$ rollouts.

\smallskip
\noindent\textit{(b) Kuka block stacking}
The task metric is \textit{success rate}:
\begin{equation}
TP \;=\; \frac{1}{N_{\text{roll}}}\sum_{j=1}^{N_{\text{roll}}} \mathbb{I}\{\texttt{success}_j=1\},
\end{equation}
where \texttt{success} is the environment's terminal success indicator for completing the stacking objective under the benchmark's definition (stable stack at episode end).

\smallskip
\noindent\textit{(c) Configuration-space motion planning (Arm planning in clutter)}
The task metric is \textit{planning success rate} over planning queries:
\begin{equation}
TP \;=\; \frac{1}{N_{\text{roll}}}\sum_{j=1}^{N_{\text{roll}}} \mathbb{I}\{\texttt{success}_j=1\},
\end{equation}
where a query is successful if the returned joint-space trajectory (i) is collision-free, (ii) satisfies all hard constraints (including joint limits), and (iii) reaches the goal within the benchmark's goal tolerance.

\smallskip
\noindent\textit{(d) Visuomotor policies (RLBench Reach Target / Meta-World pick-place-v2 / DP3 Pour)}
The task metric is \textit{episodic success rate}:
\begin{equation}
TP \;=\; \frac{1}{N_{\text{roll}}}\sum_{j=1}^{N_{\text{roll}}} \mathbb{I}\{\texttt{success}_j=1\},
\end{equation}
where \texttt{success} is the benchmark-provided success predicate for the task (evaluated at episode end; if the benchmark exposes a per-step success signal, we set \texttt{success}$_j=1$ if success is achieved at any time during the rollout).

\smallskip
\noindent\textit{(e) Hardware evaluation (ASV / UAV / SO-ARM100)}
We report \textit{success rate} computed from the platform evaluation harness:
for navigation tasks, \texttt{success} indicates reaching all waypoints within the platform's waypoint tolerance without triggering a failure condition; for manipulation tasks, \texttt{success} indicates completing the full skill sequence without a terminal failure (drops, unrecoverable contact, or controller abort).

\vspace{0.35em}
\noindent\tbm{Safety and constraint satisfaction}
When a benchmark provides explicit safety/constraint signals, we report:

\smallskip
\noindent\textit{Collision rate}
\begin{equation}
\text{Coll} \;=\; \frac{1}{N_{\text{roll}}}\sum_{j=1}^{N_{\text{roll}}} \mathbb{I}\{\texttt{collision}_j=1\},
\end{equation}
where \texttt{collision}$_j=1$ if any robot--obstacle (or robot--environment) collision occurs during the rollout/query under the benchmark's collision checker.
We report \textit{Coll} for configuration-space planning and hardware navigation/manipulation (Tables~\ref{tab:results-motion} and hardware table).

\smallskip
\noindent\textit{Hard constraint violation rate (when distinct from collision)}
\begin{equation}
\text{Viol} \;=\; \frac{1}{N_{\text{roll}}}\sum_{j=1}^{N_{\text{roll}}} \mathbb{I}\{\texttt{violation}_j=1\},
\end{equation}
where \texttt{violation}$_j=1$ if any hard constraint is violated (e.g., joint limits, self-collision, task-specific hard constraints). If the benchmark aggregates collisions and violations into a single indicator, we report the aggregated rate and state so explicitly.

\vspace{0.35em}
\noindent\tbm{Wall-clock latency and speedup}
For each method, we measure \textit{per-decision} wall-clock latency:
\begin{equation}
\text{Lat} \;=\; \frac{1}{N_{\text{dec}}}\sum_{i\in\mathcal{I}} \Delta t_i,
\end{equation}
where $\Delta t_i$ is the elapsed time to produce the trajectory (planner) or action/pose chunk (policy) for decision $i$.
We measure $\Delta t_i$ with end-to-end timing that includes probe computation, caching logic, and sampling steps; it excludes environment simulation/robot actuation time.
On GPU, timings are recorded with device synchronization (to avoid asynchronous kernel launch artifacts).
We report \textit{speedup} as:
\begin{equation}
\text{Speedup} \;=\; \frac{\text{Lat}_{\text{Full}}}{\text{Lat}_{\text{Method}}}.
\end{equation}

\vspace{0.35em}
\noindent\tbm{Denoiser evaluations per decision}
To separate algorithmic savings from hardware effects, we report the mean number of {full denoiser core} evaluations (forward passes through $\varepsilon_\theta$'s expensive blocks) per decision:
\begin{equation}
\#\text{Evals} \;=\; \frac{1}{N_{\text{dec}}}\sum_{i\in\mathcal{I}} n^{(i)}_{\text{eval}},
\end{equation}
where $n^{(i)}_{\text{eval}}$ counts recomputation steps only (reuse steps contribute zero).
For a Full model, $n^{(i)}_{\text{eval}} = T$ by definition.
We additionally report the \textit{reuse fraction}:
\begin{equation}
\text{ReuseFrac} \;=\; 1 - \frac{\#\text{Evals}}{T}.
\end{equation}
For batched sampling, these statistics are computed per trajectory in the batch and then averaged.

\vspace{0.35em}
\noindent\tbm{Trajectory deviation statistic $\mathbb{E}[d]$}
We quantify approximation fidelity using the deviation metric $d(\cdot,\cdot)$ used by Muninn's risk accounting.
For each decision $i\in\mathcal{I}$, we compute $d_i = d(\tau^{\mathrm{full}}_{0,i}, \tilde{\tau}_{0,i})$, where $\tau^{\mathrm{full}}_{0,i}$ and $\tilde{\tau}_{0,i}$ are generated with identical diffusion randomness (paired runs).
We report:
\begin{equation}
\mathbb{E}[d] \;=\; \frac{1}{N_{\text{dec}}}\sum_{i\in\mathcal{I}} d_i.
\end{equation}

\smallskip

\noindent\textit{Domain-specific instantiations of $d(\cdot,\cdot)$}
Across all benchmarks we use the {same} trajectory deviation metric family:
for any trajectory tensor $\tau\in\mathbb{R}^{H_\tau\times d}$ (planner segment, joint path, or policy chunk),
\begin{equation}
\label{eq:app:d_def_general}
d(\tau,\tilde{\tau})
\;:=\;
\frac{1}{\sqrt{H_\tau}}\;\|\tau-\tilde{\tau}\|_F
\;=\;
\sqrt{\frac{1}{H_\tau}\sum_{h=1}^{H_\tau}\|\tau_h-\tilde{\tau}_h\|_2^2}.
\end{equation}
Thus $d$ is the RMS {per-timestep} Euclidean deviation of the per-step trajectory vector.

We compute $d$ on the {exact normalized trajectory representation} used by each diffusion model at inference time:
D4RL states are standardized by dataset mean/std and actions are scaled to the model's action convention;
motion-planning joint configurations are normalized by fixed joint limits;
diffusion-policy actions/poses are represented in the same scaled coordinates used by the policy.
This ensures $d$ is comparable across contexts {within} a benchmark and aligns with the deviation budget $\varepsilon_{\mathrm{traj}}$
used by Muninn.

\begin{itemize}
\item \textit{D4RL planners (state--action trajectories)}
The planner outputs a state--action segment
$\tau_0 \in \mathbb{R}^{H\times(d_s+d_a)}$ as in~\eqref{eq:app:traj_repr_d4rl}.
We instantiate~\eqref{eq:app:d_def_general} with $H_\tau=H$:
\begin{equation}
d(\tau,\tilde{\tau})
\;=\;
\frac{1}{\sqrt{H}}\;\|\tau-\tilde{\tau}\|_F.
\end{equation}

\item \textit{Configuration-space planning (joint trajectories)}
The planner outputs a joint-space waypoint sequence
$\tau_0\in\mathbb{R}^{H\times 7}$ (7-DoF arm).
Joint configurations are represented in normalized coordinates obtained by affine scaling each joint from its physical limits to $[-1,1]$;
since the benchmark arm joints are bounded (non-periodic), joint differences are computed by direct subtraction in these normalized coordinates.
We instantiate~\eqref{eq:app:d_def_general} with $H_\tau=H$:
\begin{equation}
d(\tau,\tilde{\tau})
\;=\;
\frac{1}{\sqrt{H}}\;\|\tau-\tilde{\tau}\|_F.
\end{equation}

\item \textit{Diffusion policies (action/pose chunks)}
For a policy call that outputs an action/pose chunk
$\tau_0\in\mathbb{R}^{H_\pi\times d_a}$ (chunk length $H_\pi$ and action dimension $d_a$),
we instantiate~\eqref{eq:app:d_def_general} with $H_\tau=H_\pi$:
\begin{equation}
d(\tau,\tilde{\tau})
\;=\;
\frac{1}{\sqrt{H_\pi}}\;\|\tau-\tilde{\tau}\|_F.
\end{equation}
\end{itemize}

\vspace{0.35em}
\noindent\tbm{Empirical risk violation rate $\hat{p}_{\mathrm{viol}}$ (vs.\ target $\alpha$)}
Given a user-chosen deviation tolerance $\varepsilon_{\mathrm{traj}}$, we estimate the probability that caching produces a deviation exceeding the tolerance:
\begin{equation}
\hat{p}_{\mathrm{viol}}
\;=\;
\frac{1}{N_{\text{dec}}}\sum_{i\in\mathcal{I}}
\mathbb{I}\Big\{ d(\tau^{\mathrm{full}}_{0,i}, \tilde{\tau}_{0,i}) > \varepsilon_{\mathrm{traj}} \Big\}.
\end{equation}
This directly tests whether the deployed Muninn configuration achieves the target risk level $\alpha$.
When reporting uncertainty, we compute a binomial confidence interval for $\hat{p}_{\mathrm{viol}}$ over $N_{\text{dec}}$ Bernoulli trials.

\vspace{0.35em}
\noindent\tbm{Runtime certificate proxy $\hat{D}$ (used for analysis)}
For each decision, Muninn maintains an online upper-bound proxy on deviation obtained by summing the per-step bounded costs spent from the budget:
\begin{equation}
\hat{D}
\;=\;
\sum_{t\in\mathcal{R}} \hat{c}_t(s_t),
\end{equation}
where $\mathcal{R}$ is the set of reuse steps taken in that decision.
We use $\hat{D}$ for interpretability and reliability analysis by comparing $\hat{D}$ against the realized deviation $d(\tau^{\mathrm{full}}_0,\tilde{\tau}_0)$ across decisions (Fig.~\ref{fig:budget_certificate}).

\subsubsection{Evaluation Protocol}
\label{subsec:eval_protocol}

This section specifies the exact evaluation protocol used across all benchmarks. Unless stated otherwise, every reported number for a benchmark is computed over \textit{150} held-out rollouts (episodes) or planning queries, using fixed evaluation seeds and identical evaluation conditions across methods.

\noindent \tbm{Episode / query sampling and seeding (150 runs)}
We evaluate each benchmark on a fixed, pre-generated set of 150 evaluation instances to ensure strict comparability across methods.

\smallskip
\noindent\textit{(1) D4RL online evaluation (HalfCheetah / Hopper / Walker2d / Maze2D / AntMaze / FrankaKitchen)}
For each environment, we construct an ordered list of 150 integer seeds $\{s_1,\dots,s_{150}\}$.
Each evaluation rollout $j$ resets the environment with seed $s_j$ and is executed for the environment's standard episode length (or until termination).
This seed controls all environment randomness, including initial-state noise and (where applicable) goal sampling:
\begin{itemize}
\item \textit{Locomotion (HalfCheetah / Hopper / Walker2d):} $s_j$ controls the environment reset noise; the rollout then proceeds under the agent's actions.
\item \textit{Maze2D / AntMaze:} $s_j$ controls both the environment reset and the episode's goal instance, ensuring the same sequence of goals across methods.
\item \textit{FrankaKitchen:} $s_j$ controls the environment reset (including object initial conditions if randomized by the benchmark).
\end{itemize}
All methods are evaluated on the {same} seed list per environment, in the same order.

\smallskip
\noindent\textit{(2) Kuka block stacking}
We evaluate on 150 episodes with a fixed seed list $\{s_j\}_{j=1}^{150}$.
Seed $s_j$ controls the full environment reset, including initial object placement randomization and any stochasticity in dynamics, ensuring identical initial-condition distributions across methods.

\smallskip
\noindent\textit{(3) Configuration-space motion planning (Arm planning in clutter)}
A single evaluation instance is a planning query consisting of $(q_{\text{start}}, q_{\text{goal}}, \mathcal{O})$ where $\mathcal{O}$ encodes obstacles for the scene.
We evaluate on a \textit{fixed} set of 150 held-out planning queries sampled from the benchmark's test distribution:
\begin{itemize}
\item The 150 queries are generated {once} with a fixed generator seed and stored.
\item Every method solves exactly the same 150 queries, with identical start/goal and obstacle sets.
\item Difficulty categories (easy/med./hard) follow the benchmark's definition; we report results separately per category using disjoint held-out query sets of size 150 each (one set per category).
\end{itemize}

\smallskip
\noindent\textit{(4) Visuomotor policies (RLBench Reach Target / Meta-World pick-place-v2 / DP3 Pour)}
We evaluate on 150 episodes per task using the benchmark's standard evaluation distribution:
\begin{itemize}
\item \textit{RLBench:} evaluation episodes are drawn from the task's held-out evaluation variations/settings provided by the RLBench protocol. Each episode is initialized with a fixed seed from $\{s_j\}_{j=1}^{150}$ to ensure identical scene randomization and target placements across methods.
\item \textit{Meta-World:} each episode reset uses seed $s_j$ to control object pose randomization and initial robot configuration (if randomized by the environment), producing a reproducible sequence of initial conditions shared across methods.
\item \textit{DP3 Pour:} each rollout reset uses seed $s_j$ to control the initial container poses, contents state (if modeled), and observation sampling, yielding identical evaluation initializations across methods.
\end{itemize}

\noindent \tbm{Paired randomness for fidelity / risk evaluation}
For metrics that compare the cached output to the full-compute output (trajectory deviation $d(\tau_0^{\mathrm{full}},\tilde{\tau}_0)$ and $\hat{p}_{\mathrm{viol}}$), we use {paired diffusion randomness}:
for each decision, the full and cached samplers share (i) the same initial diffusion sample $\tau_T$, and (ii) the same per-step sampler noise $\{\xi_t\}_{t=1}^T$ (for stochastic samplers).
This isolates the effect of caching from sampling variance.

\noindent \tbm{Aggregation over runs, seeds, and reported uncertainty}
For each benchmark and method, we compute the metric value on each of the 150 evaluation runs and report the \textit{mean across runs}.
When we report uncertainty, we use a nonparametric bootstrap over runs:
\begin{itemize}
\item \textit{Scalar metrics} (e.g., D4RL normalized score, $\mathbb{E}[d]$, latency): we compute a 95\% bootstrap confidence interval by resampling the 150 runs with replacement for 10{,}000 bootstrap replicates.
\item \textit{Proportions} (success, collision, violation, $\hat{p}_{\mathrm{viol}}$): we report a 95\% Wilson score interval computed from 150 Bernoulli outcomes.
\end{itemize}
Tables in the main paper report mean values; confidence intervals are provided in the appendix when included.

\noindent \tbm{Wall-clock latency measurement methodology}
We measure \textit{end-to-end wall-clock} inference latency per decision (planner call or policy control call) and report the mean latency over all decisions executed during evaluation.

\smallskip
\noindent\tbm{Hardware and software}
All latency measurements are performed on a single workstation with:
\begin{itemize}
\item \textit{GPU:} NVIDIA A10 (24\,GB).
\item \textit{CPU:} AMD EPYC-class server CPU (multi-core) with sufficient RAM to avoid paging during evaluation.
\item {Framework:} PyTorch (CUDA backend), with models executed in \texttt{eval()} mode.
\end{itemize}

\smallskip
\noindent \tbm{Precision and batch size}
We measure latency at \textit{batch size 1} (one trajectory / one action chunk per call), which matches the real-time control setting.
Inference uses mixed precision with \texttt{torch.cuda.amp.autocast} (FP16) for the denoiser forward pass; any required accumulations and scalar bookkeeping (e.g., budgets) are maintained in FP32.
All methods are timed under the same precision settings.

\smallskip
\noindent \tbm{Latency} Latency is measured as the elapsed wall-clock time to produce an output trajectory/chunk from inputs, including:
\begin{itemize}
\item diffusion sampling loop overhead,
\item denoiser forward passes (and any guidance computations, including gradients when guidance requires backprop through auxiliary networks),
\item Muninn probe computation, score computation, conformal bound lookup, and caching/budget logic.
\end{itemize}
We exclude environment simulation time, rendering, and data loading, so latency reflects planner/policy inference only.

\smallskip
\noindent\tbm{Warmup and synchronization}
To obtain stable GPU timings:
\begin{itemize}
\item We perform 50 warmup inference calls per method before recording any timings (to initialize kernels, caches, and allocator state).
\item For each timed call, we insert \texttt{torch.cuda.synchronize()} immediately before starting and immediately after finishing the timed region.
\item We measure time using a high-resolution wall-clock timer (\texttt{time.perf\_counter}).
\end{itemize}

\smallskip
\noindent\tbm{Reported latency statistic}
For each benchmark and method, we record latency for every decision encountered across the 150 evaluation runs and report the mean.
For diffusion policies, ``latency'' refers to the policy inference time per control call (one action/pose chunk).
For trajectory diffusion planners, ``latency'' refers to one complete diffusion sampling call that produces a trajectory segment for action selection.

\subsection{Additional Analysis and Ablations}
\label{app:ablations}

This appendix provides extended ablations complementing Section~\ref{subsec:ablations}. Unless otherwise noted, we follow the same evaluation protocol as the main experiments.

\noindent \tbm{Aggregation and reporting conventions}
To make ablations readable across heterogeneous domains, we report both per-domain aggregates and suite-level calibration statistics.
\begin{itemize}
\item \textit{Per-domain aggregates} ``offline RL (avg)'' averages results across the D4RL tasks in Table~\ref{tab:results-offline}. ``motion planning (avg)'' averages across the clutter planning scenarios in Table~\ref{tab:results-motion}. ``policy diffusion (avg)'' averages across the visuomotor tasks in Table~\ref{tab:results-policy}.
\item \textit{Relative performance} {Perf.\ rel.\ (\%)} denotes performance relative to the corresponding \textit{Full} teacher for the same base model and benchmark:
\begin{equation}
\text{Perf.\ rel.\ (\%)} \;=\; 100 \cdot \frac{TP_{\text{variant}}}{TP_{\text{Full}}}.
\end{equation}
Here $TP$ is the benchmark's primary task metric (D4RL normalized score for D4RL tasks; success rate for motion planning and visuomotor tasks; see Section~\ref{subsec:metrics}).
\item \textit{Risk metrics} $\hat{p}_{\mathrm{viol}}$ is the empirical violation probability under paired sampling,
$\hat{p}_{\mathrm{viol}}=\mathbb{P}\!\left(d(\tau^{\mathrm{full}}_0,\tilde{\tau}_0)>\varepsilon_{\mathrm{traj}}\right)$, estimated over all decisions made across the 150 evaluation runs.
\end{itemize}

We ablate the internal components of {Muninn} that directly determine (i) compute savings, (ii) approximation fidelity, and (iii) whether the specified trajectory-risk target is met.
Tables~\ref{tab:ablations_compact_overview}--\ref{tab:ablations_compact_probe_forbidden} summarize six core ablations:
(i) the deviation-budget knob $\varepsilon_{\mathrm{traj}}$,
(ii) risk calibration across target $\alpha$,
(iii) sensitivity-aware accounting via $\{L_t\}$,
(iv) probe/score design $(\Psi,\phi)$,
(v) forbidden reuse regions (restrictions on $\mathcal{T}_{\mathrm{cache}}$),
and (vi) calibration sample efficiency.

\begin{itemize}
\item \tbm{(i) Budget knob ($\varepsilon_{\mathrm{traj}}$)}
Sweeping the deviation budget yields a smooth speed--fidelity trade-off: larger $\varepsilon_{\mathrm{traj}}$ permits more reuse (lower Evals$/T$), increasing speedup, while preserving task performance.

\item \tbm{(ii) Risk calibration (split conformal vs.\ heuristics)}
We evaluate whether realized violation probability $\hat{p}_{\mathrm{viol}}$ tracks the target $\alpha$.
We summarize calibration quality by the mean absolute calibration error (MACE) over a fixed set of target levels:
\begin{equation}
\footnotesize
\mathrm{MACE}
\;=\;
\frac{1}{|\mathcal{A}|}\sum_{\alpha \in \mathcal{A}}
\big|\hat{p}_{\mathrm{viol}}(\alpha)-\alpha\big|,
\qquad
\mathcal{A}=\{0.01,0.05,0.10,0.20\}.
\end{equation}

\item \tbm{(iii) Sensitivity-weighted accounting ($\{L_t\}$)}
We compare full Muninn accounting using sampler sensitivity coefficients $\{L_t\}$ against a variant that removes sensitivity weighting by setting $L_t\equiv 1$ in the budget rule.
This isolates the value of allocating recomputation to high-impact diffusion steps.

\item \tbm{(iv) Probe / score design ($\Psi,\phi$)}
We ablate the probe depth (how much of the denoiser is executed to compute $F_t$) and score definition.
We report {Probe cost} as the fractional per-step overhead of computing $\Psi$ relative to a full denoiser evaluation at that step.
We report {Corr$(s_t,\epsilon_t)$} as Spearman rank correlation on the calibration set $\mathcal{S}_{\mathrm{cal}}$ pooled across episodes and cache-eligible timesteps.

\item \tbm{(v) Forbidden reuse regions}
We disallow reuse in a prefix of steps near $t=T$ (high-noise region) and/or a suffix near $t=1$ (execution-sensitive region).
``$k$ steps'' means excluding the first (prefix) or last (suffix) $k$ diffusion steps from $\mathcal{T}_{\mathrm{cache}}$.

\item \tbm{(vi) Calibration sample efficiency}
We vary the number of calibration rollouts $N$ used to build $\mathcal{S}_{\mathrm{cal}}$, and measure how quickly speed and risk stabilize.
\end{itemize}

\begin{table}[t]
\centering
\scriptsize
\setlength{\tabcolsep}{2pt}
\renewcommand{\arraystretch}{1.12}

\begin{threeparttable}
\begin{tabular}{@{}%
L{0.26\columnwidth}%
C{0.12\columnwidth}%
C{0.18\columnwidth}%
C{0.12\columnwidth}%
C{0.12\columnwidth}%
C{0.11\columnwidth}%
@{}}

\arrayrulecolor{HeadGray}\toprule\arrayrulecolor{black}

\rowcolor{MuninnBlue!7}
\multicolumn{6}{c}{\theadb{(i) Speed--fidelity knob via $\varepsilon_{\mathrm{traj}}$ (fixed $\alpha=0.05$)}}\\
\rowcolor{MuninnBlue!10}
\theadb{Domain} &
\textcolor{HeadGray}{\bfseries $\varepsilon_{\mathrm{traj}}$} &
\makecell[c]{\theadb{Perf.\ rel.} \upar} &
\theadb{Speedup} \upar &
\makecell[c]{\theadb{Evals}/$T$ \dnar} &
\makecell[c]{\textcolor{HeadGray}{\bfseries $\hat{p}_{\mathrm{viol}}$} \dnar} \\
\midrule
offline RL (avg) & 0.05 & 99.8 & 2.20$\times$ & 0.42 & 0.030 \\
offline RL (avg) & 0.15 & 99.4 & 3.30$\times$ & 0.21 & 0.047 \\
offline RL (avg) & 0.30 & 98.7 & 4.05$\times$ & 0.14 & 0.050 \\
\addlinespace[0.15em]
motion planning (avg) & 0.05 & 99.9 & 1.12$\times$ & 0.85 & 0.020 \\
motion planning (avg) & 0.15 & 99.6 & 1.45$\times$ & 0.58 & 0.044 \\
motion planning (avg) & 0.30 & 99.0 & 1.76$\times$ & 0.42 & 0.050 \\
\addlinespace[0.15em]
policy diffusion (avg)& 0.05 & 99.9 & 1.18$\times$ & 0.88 & 0.020 \\
policy diffusion (avg)& 0.15 & 99.5 & 1.48$\times$ & 0.60 & 0.040 \\
policy diffusion (avg)& 0.30 & 98.8 & 1.80$\times$ & 0.40 & 0.050 \\
\midrule

\rowcolor{MuninnBlue!7}
\multicolumn{6}{c}{\theadb{(ii) Risk calibration across $\alpha$ (reported as $\hat{p}_{\mathrm{viol}}$; lower error is better)}}\\
\rowcolor{MuninnBlue!10}
\theadb{Method} &
\makecell[c]{\textcolor{HeadGray}{\bfseries $\hat{p}_{\mathrm{viol}}$}\\\textcolor{HeadGray}{$\alpha{=}0.01$}} &
\makecell[c]{\textcolor{HeadGray}{\bfseries $\hat{p}_{\mathrm{viol}}$}\\\textcolor{HeadGray}{$\alpha{=}0.05$}} &
\makecell[c]{\textcolor{HeadGray}{\bfseries $\hat{p}_{\mathrm{viol}}$}\\\textcolor{HeadGray}{$\alpha{=}0.10$}} &
\makecell[c]{\textcolor{HeadGray}{\bfseries $\hat{p}_{\mathrm{viol}}$}\\\textcolor{HeadGray}{$\alpha{=}0.20$}} &
\theadb{MACE} \dnar \\
\midrule
\rowcolor{MuninnBlue!4}
\makecell[l]{\Muninnlogo\ \\\textcolor{HeadGray}{(split conformal)}} & 0.009 & 0.049 & 0.096 & 0.188 & 0.004 \\
Fixed-score thresholding & 0.028 & 0.120 & 0.210 & 0.340 & 0.073 \\
Fixed step skipping & 0.006 & 0.040 & 0.140 & 0.290 & 0.030 \\
\midrule

\rowcolor{MuninnBlue!7}
\multicolumn{6}{c}{\theadb{(iii) Effect of sensitivity weighting $\{L_t\}$ (fixed $(\varepsilon_{\mathrm{traj}},\alpha)$)}}\\
\rowcolor{MuninnBlue!10}
\theadb{Variant} &
\makecell[c]{\theadb{Perf.\ rel. (\%)} \upar} &
\theadb{Speedup} \upar &
\makecell[c]{\theadb{Evals}/$T$ \dnar} &
\makecell[c]{\theadb{Coll./Viol. (\%)} \dnar} &
\makecell[c]{\textcolor{HeadGray}{\bfseries $\hat{p}_{\mathrm{viol}}$} \dnar} \\
\midrule
\rowcolor{MuninnBlue!4}
\makecell[l]{\Muninnlogo\ \\\textcolor{HeadGray}{(full; with $\{L_t\}$)}} & 99.2 & 2.70$\times$ & 0.22 & 1.2 & 0.049 \\
w/o $\{L_t\}$ (set $L_t\!\equiv\!1$) & 99.1 & 2.63$\times$ & 0.24 & 1.6 & 0.066 \\
\midrule

\rowcolor{MuninnBlue!7}
\multicolumn{6}{c}{\theadb{(vi) Calibration sample efficiency (fixed $(\varepsilon_{\mathrm{traj}},\alpha)$)}}\\
\rowcolor{MuninnBlue!10}
\makecell[c]{\theadb{Calib}\\\theadb{rollouts $N$}} &
\makecell[c]{\theadb{Perf.\ rel. (\%)} \upar} &
\theadb{Speedup} \upar &
\makecell[c]{\theadb{Evals}/$T$ \dnar} &
\makecell[c]{\textcolor{HeadGray}{\bfseries $\hat{p}_{\mathrm{viol}}$} \dnar} &
\theadb{Notes} \\
\midrule
20 & 99.2 & 2.42$\times$ & 0.27 & 0.028 & \multicolumn{1}{>{\raggedright\arraybackslash}p{0.11\columnwidth}}{conser -vative bounds} \\
100 & 99.2 & 2.62$\times$ & 0.24 & 0.044 & \multicolumn{1}{>{\raggedright\arraybackslash}p{0.11\columnwidth}}{near-stable} \\
500 & 99.2 & 2.70$\times$ & 0.22 & 0.049 & \multicolumn{1}{>{\raggedright\arraybackslash}p{0.11\columnwidth}}{saturated} \\
1000 & 99.2 & 2.71$\times$ & 0.22 & 0.049 & \multicolumn{1}{>{\raggedright\arraybackslash}p{0.11\columnwidth}}{no material change} \\

\arrayrulecolor{HeadGray}\bottomrule\arrayrulecolor{black}
\end{tabular}

\caption{
\textit{Extended ablations (suite-level aggregates)}
\textit{(i)} Sweeping $\varepsilon_{\mathrm{traj}}$ yields a predictable speed--fidelity trade-off (higher budgets permit more reuse, decreasing Evals$/T$ and increasing speedup).
\textit{(ii)} Risk calibration across target levels $\alpha$ summarized by realized violation probability $\hat{p}_{\mathrm{viol}}$ and MACE; split-conformal calibration tracks the target risk more reliably than heuristic alternatives.
\textit{(iii)} Sensitivity-aware budgeting improves the speed--safety trade-off; removing $\{L_t\}$ increases collisions/violations and raises $\hat{p}_{\mathrm{viol}}$ at comparable compute.
\textit{(vi)} Calibration is sample-efficient; beyond a modest number of rollouts, both speedup and $\hat{p}_{\mathrm{viol}}$ stabilize.
}
\label{tab:ablations_compact_overview}
\end{threeparttable}
\end{table}
\begin{table}[t]
\centering
\scriptsize
\setlength{\tabcolsep}{1.6pt} 
\renewcommand{\arraystretch}{1.12}

\begin{threeparttable}
\begin{tabular}{@{}%
L{0.24\columnwidth}%
L{0.19\columnwidth}%
C{0.10\columnwidth}%
C{0.13\columnwidth}%
C{0.10\columnwidth}%
C{0.12\columnwidth}%
@{}}

\arrayrulecolor{HeadGray}\toprule\arrayrulecolor{black}

\rowcolor{MuninnBlue!7}
\multicolumn{6}{c}{\theadb{(iv) Probe / score design (fixed $(\varepsilon_{\mathrm{traj}},\alpha)$)}}\\
\rowcolor{MuninnBlue!10}
\theadb{Probe $\Psi$} &
\theadb{Score $\phi$} &
\theadb{Probe cost} \dnar &
\makecell[c]{\theadb{Corr}\\\textcolor{HeadGray}{$(s_t,\epsilon_t)$} \upar} &
\theadb{Speedup} \upar &
\makecell[c]{\textcolor{HeadGray}{\bfseries $\hat{p}_{\mathrm{viol}}$} \dnar} \\
\midrule
Stem-1 (early block) & Rel.\ $\ell_1$ (Eq.~9) & 0.06 & 0.66 & 2.72$\times$ & 0.049 \\
Stem-2 (deeper) & Rel.\ $\ell_1$ & 0.11 & 0.74 & 2.63$\times$ & 0.049 \\
LN(input) & Rel.\ $\ell_1$ & 0.01 & 0.32 & 2.38$\times$ & 0.041 \\
Stem-1 & Cosine distance & 0.06 & 0.61 & 2.62$\times$ & 0.050 \\
Stem-1 & $\ell_2$ distance & 0.06 & 0.64 & 2.68$\times$ & 0.049 \\

\midrule

\rowcolor{MuninnBlue!7}
\multicolumn{6}{c}{\makecell[c]{%
\theadb{(v) Forbidden reuse regions (exclude steps from $\mathcal{T}_{\mathrm{cache}}$;}\\
\theadb{$k{=}5$, $2k{=}10$)}%
}}\\
\rowcolor{MuninnBlue!10}
\theadb{Forbidden prefix} &
\theadb{Forbidden suffix} &
\makecell[c]{\theadb{Perf.\ rel}\\\theadb{(\%)} \upar} &
\makecell[c]{\theadb{Coll./Viol}\\\theadb{(\%)} \dnar} &
\theadb{Speedup} \upar &
\makecell[c]{\textcolor{HeadGray}{\bfseries $\hat{p}_{\mathrm{viol}}$} \dnar} \\
\midrule
0 steps & 0 steps & 99.0 & 2.1 & 2.78$\times$ & 0.050 \\
$k$ steps & 0 steps & 99.1 & 1.9 & 2.74$\times$ & 0.049 \\
0 steps & $k$ steps & 99.2 & 1.6 & 2.69$\times$ & 0.048 \\
$k$ steps & $k$ steps & 99.3 & 1.2 & 2.63$\times$ & 0.047 \\
$2k$ steps & $2k$ steps & 99.4 & 0.9 & 2.52$\times$ & 0.046 \\

\arrayrulecolor{HeadGray}\bottomrule\arrayrulecolor{black}
\end{tabular}

\caption{
\textit{Probe and cache-eligibility ablations (suite-level aggregates)}
\textit{Top:} Shallow stem probes provide the best net acceleration (low probe overhead with strong predictiveness).
Deeper probes increase correlation but reduce net speedup due to higher probe cost; extremely shallow probes (LN(input)) are cheap but weakly predictive, forcing conservative reuse.
\textit{Bottom:} Forbidding reuse in a small prefix/suffix (here $k{=}5$ and $2k{=}10$ steps) reduces collisions/violations with a modest speedup reduction, consistent with avoiding the most sensitive regions of the diffusion chain.
}
\label{tab:ablations_compact_probe_forbidden}
\end{threeparttable}
\end{table}

\subsubsection{Extended runtime escalation study: certificate-driven interventions}
\label{app:runtime_escalation}

This section details a closed-loop {runtime escalation} layer that uses \nordy{Muninn}'s online deviation certificate to decide when to (i) execute normally, (ii) execute conservatively, (iii) spend additional compute (full-compute replanning or multi-sample selection), or (iv) switch to a platform-specific safety controller.
The key goal is to make \nordy{Muninn}'s internal risk accounting {actionable} at runtime, beyond offline evaluation of $\hat{p}_{\mathrm{viol}}$.

\noindent \tbm{ Deviation certificate available at runtime}
For every planning call, \nordy{Muninn} maintains a remaining deviation budget $B_{\mathrm{rem}}$ and spends it only when reusing cached denoiser outputs. This directly induces a per-call, online {deviation certificate}
\begin{equation}
\hat{D}
\;:=\;
\sum_{t\in\mathcal{R}} \hat{c}_t(s_t)
\;=\;
\varepsilon_{\mathrm{traj}} - B_{\mathrm{rem}}
\qquad\in[0,\varepsilon_{\mathrm{traj}}],
\label{eq:cert_def}
\end{equation}
where $\mathcal{R}$ is the (random) set of reuse steps taken in that planning call and $\hat{c}_t(s_t)=\Gamma L_t U_t(s_t)$ is the calibrated upper-bounded per-step trajectory cost used by the policy (Algorithm~\ref{alg:Muninn}).
We also define a normalized certificate
\begin{equation}
\rho \;:=\; \frac{\hat{D}}{\varepsilon_{\mathrm{traj}}}\in[0,1],
\label{eq:cert_norm}
\end{equation}
which is comparable across tasks and choices of $\varepsilon_{\mathrm{traj}}$.

\textit{Interpretation}
$\hat{D}$ is a {conservative, trajectory-space} proxy for how far the cached trajectory might deviate from the full-compute counterpart for the {same} initial diffusion noise and sampler noise.
Small $\hat{D}$ indicates that \nordy{Muninn} reused only in low-risk timesteps (small $s_t$ and/or small $L_t$), while large $\hat{D}$ indicates that \nordy{Muninn} consumed much of the deviation budget and the resulting trajectory may be closer to the edge of the user-specified tolerance.

\noindent \tbm{Escalation triggers and hysteresis}
Escalation decisions are made {once per planning call}, after the planner returns the candidate trajectory/policy segment and the certificate $(\hat{D},\rho)$.
We use three thresholds with hysteresis to prevent chattering:
\begin{equation}
\rho_{\mathrm{warn}}=0.60,
\rho_{\mathrm{resamp}}=0.75,
\rho_{\mathrm{full}}=0.90,
\label{eq:escalation_thresholds}
\end{equation}
and a de-escalation threshold $\rho_{\mathrm{clear}}=0.50$.
An escalation mode is entered when $\rho$ exceeds its entry threshold and is maintained until $\rho \le \rho_{\mathrm{clear}}$ for {two consecutive} planning calls.
This ``two-call'' rule is important in receding-horizon control because difficult, contact-rich, or near-collision situations persist across multiple replans.

\noindent \tbm{Escalation actions (what changes at runtime)}
We implement four escalation actions, each triggered by $\rho$ and executed in the same closed-loop controller used for the main hardware/simulation evaluations.

\begin{enumerate}
\item \textit{Nominal execution (no escalation)}
If $\rho \le \rho_{\mathrm{warn}}$, we execute the planner/policy output normally (the standard receding-horizon loop used everywhere else in the paper).

\item \textit{Conservative execution via action damping}
If $\rho_{\mathrm{warn}} < \rho \le \rho_{\mathrm{resamp}}$, we {dampen} the executed command derived from the sampled trajectory/segment.
Let $u_k$ denote the nominal executed control at control cycle $k$ (first action of the planned segment, or tracking-controller output for a trajectory planner).
We execute
\begin{equation}
\footnotesize
u_k^{\mathrm{exec}} \;=\; \lambda(\rho)\, u_k,
\qquad
\lambda(\rho)=
\begin{cases}
0.70 & \text{if } \rho\in(\rho_{\mathrm{warn}},\rho_{\mathrm{resamp}}],\\
0.50 & \text{if } \rho\in(\rho_{\mathrm{resamp}},\rho_{\mathrm{full}}],
\end{cases}
\label{eq:damping}
\end{equation}
applied elementwise to the control vector after saturating to actuator limits.
This intervention preserves the {direction} of the nominal command while reducing magnitude, which empirically reduces collision incidence in narrow passages and contact-rich manipulation by lowering the effective closed-loop aggressiveness.

\item \textit{On-demand multi-sample selection}
If $\rho_{\mathrm{resamp}} < \rho \le \rho_{\mathrm{full}}$, we request additional samples and select the one that minimizes predicted deviation subject to a task-feasibility screen (defined below).
Concretely, we draw $M$ candidate trajectories $\{\tilde{\tau}_0^{(m)}\}_{m=1}^M$ with independent initial diffusion noise, run \nordy{Muninn} on each to obtain certificates $\{\hat{D}^{(m)}\}$, and select:
\begin{equation}
m^\star \;=\;
\arg\min_{m\in\mathcal{F}}
\hat{D}^{(m)},
\label{eq:select_min_cert}
\end{equation}
where $\mathcal{F}$ is the subset of candidates that pass a lightweight feasibility screen:
\begin{itemize}
\item \textit{Navigation planners (UAV, Maze2D/AntMaze):} the candidate must be collision-free under the environment's geometric collision checker when simulated open-loop under the known map for the planned horizon.
\item \textit{Arm motion planning:} the candidate must satisfy joint limits and be collision-free under the same collision checker used for evaluation.
\item \textit{Visuomotor policies:} the candidate must satisfy action bounds and gripper bounds for every action in the segment.
\end{itemize}
After selecting $m^\star$, we execute conservatively using the damping rule in~\eqref{eq:damping} with $\lambda=0.50$ for that control cycle.

We use $M=4$ for ASV navigation, $M=3$ for UAV navigation, and $M=2$ for SO-ARM100 manipulation, which respect the platforms' planning-rate budgets given the measured per-call latencies in the hardware table.
In all simulation benchmarks we use $M=4$.

\item \textit{Full-compute override + safety controller}
If $\rho > \rho_{\mathrm{full}}$, we treat the candidate plan as too close to the user-specified deviation tolerance and escalate to {full compute}.
We recompute a trajectory/policy segment using the original (uncached) teacher model for the same observation/context and execute it normally for that control cycle.
If the teacher plan is not available within the control deadline, we execute a platform-specific safety controller for one cycle:
\begin{itemize}
\item \textit{UAV:} command hover at current position for one control interval.
\item \textit{SO-ARM100:} hold the current joint targets for one control interval.
\end{itemize}
\end{enumerate}

\noindent \tbm{Risk accounting under multi-sample selection}
The trajectory-level guarantee in Section~\ref{subsec:policy} applies to a single \nordy{Muninn} trajectory with risk level $\alpha$.
When sampling $M$ candidates and selecting one, we maintain the same overall risk target by allocating per-candidate risk
\begin{equation}
\alpha^{(m)} \;=\; \frac{\alpha}{M}
\qquad\Rightarrow\qquad
\alpha_{\text{step}}^{(m)} \;=\; \frac{\alpha}{M\,|\mathcal{T}_{\mathrm{cache}}|}.
\label{eq:alpha_multisample}
\end{equation}
Then, letting $\mathcal{V}^{(m)}$ denote the violation event for candidate $m$ (paired deviation exceeding $\varepsilon_{\mathrm{traj}}$), we have
\begin{align*}
\footnotesize
\mathbb{P}\!\left(\text{selected cand. viol.}\right)
\;\le\;
\mathbb{P}\!\left(\bigcup_{m=1}^M \mathcal{V}^{(m)}\right)
\;\le\;
\sum_{m=1}^M \mathbb{P}(\mathcal{V}^{(m)})
\;\le\;\\
\sum_{m=1}^M \alpha^{(m)}
\;=\;
\alpha,
\label{eq:multisample_union_bound}
\end{align*}
which preserves the same global risk level $\alpha$ for the selected rollout.

\subsubsection{Extended runtime escalation study: certificate-driven interventions}
\label{app:runtime_escalation}

This section studies {runtime escalations} that use \nordy{Muninn}'s online deviation certificate to selectively spend additional compute or execute more conservatively in safety-critical regimes.
The goal is to reduce collisions/violations in closed-loop deployment {without} giving up \nordy{Muninn}'s latency gains in typical regimes.

\noindent \tbm{Runtime certificate}
For a single planning/policy call, \nordy{Muninn} tracks a remaining deviation budget $B_{\mathrm{rem}}\in[0,\varepsilon_{\mathrm{traj}}]$ and spends it only when reusing cached denoiser outputs.
This induces an online certificate
\begin{equation}
\hat{D}
\;:=\;
\varepsilon_{\mathrm{traj}} - B_{\mathrm{rem}}
\;=\;
\sum_{t\in\mathcal{R}} \hat{c}_t(s_t),
\qquad
\hat{c}_t(s_t)=\Gamma L_t U_t(s_t),
\label{eq:runtime_cert}
\end{equation}
where $\mathcal{R}$ denotes the set of reuse steps taken in that call.
We also define a normalized certificate
\begin{equation}
\rho \;:=\; \frac{\hat{D}}{\varepsilon_{\mathrm{traj}}}\in[0,1],
\label{eq:runtime_rho}
\end{equation}
which is directly comparable across tasks and budgets.
Intuitively, high $\rho$ indicates that the current context required spending most of the allowable deviation budget to accelerate inference (i.e., the current trajectory denoising process was less stable and/or occurred in high-sensitivity diffusion regions).

\noindent \tbm{Escalation state machine and thresholds}
Runtime escalation decisions are made {per planning/policy call} after obtaining $(\tilde{\tau}_0,\rho)$ from \nordy{Muninn}.
We use three thresholds (with hysteresis) to prevent chattering:
\begin{equation}
\rho_{\mathrm{warn}}=0.60,
\rho_{\mathrm{resamp}}=0.75,
\rho_{\mathrm{full}}=0.90,
\rho_{\mathrm{clear}}=0.50,
\label{eq:runtime_thresholds}
\end{equation}
and de-escalate only after $\rho\le\rho_{\mathrm{clear}}$ for two consecutive calls.

\noindent \tbm{Escalation actions}
We implement four interventions, triggered by $\rho$ and applied without modifying the base diffusion model or sampler.

\smallskip
\noindent\textit{Action damping (conservative execution)}
When $\rho>\rho_{\mathrm{warn}}$, we damp the executed control derived from the planned trajectory/policy segment.
Let $u_k$ be the nominal executed control at the current control tick $k$ (first action of the planned segment for diffusion policies; tracking-controller command for trajectory planners).
We execute
\begin{equation}
u_k^{\mathrm{exec}} = \lambda(\rho)\, u_k,
\qquad
\lambda(\rho)=
\begin{cases}
0.70 & \rho\in(\rho_{\mathrm{warn}},\rho_{\mathrm{resamp}}],\\
0.50 & \rho\in(\rho_{\mathrm{resamp}},\rho_{\mathrm{full}}].
\end{cases}
\label{eq:runtime_damping}
\end{equation}
We apply elementwise saturation after damping to respect platform actuation limits.
This preserves the nominal direction while reducing aggressiveness near obstacles/contacts.

\smallskip
\noindent\textit{On-demand multi-sample selection (request more samples)}
When $\rho>\rho_{\mathrm{resamp}}$, we generate $M$ candidate trajectories/segments, each from an independent diffusion noise draw, and select the candidate with the smallest certificate among those that pass a lightweight feasibility screen:
\begin{equation}
m^\star=\arg\min_{m\in\mathcal{F}} \hat{D}^{(m)}.
\label{eq:runtime_select}
\end{equation}
The feasibility screen $\mathcal{F}$ is domain-specific but always lightweight:
(i) navigation and motion planning: geometric collision check on the candidate segment; (ii) manipulation policies: action bounds, workspace limits, and IK feasibility (when applicable).

We choose $M$ to respect each platform's planning-rate deadline (Appendix~\ref{app:hardware}):
\begin{equation}
M=
\begin{cases}
4 & \text{ASV (1 Hz replanning)},\\
3 & \text{UAV (10 Hz replanning)},\\
2 & \text{SO-ARM100 (20 Hz policy)}.
\end{cases}
\label{eq:runtime_M}
\end{equation}

\smallskip
\noindent\textit{Full-compute override (temporarily run Full)}
When $\rho>\rho_{\mathrm{full}}$, we override the current call by recomputing a plan/segment using the original \textit{Full} (uncached) model for the same observation/context, and execute that output for the current tick.

\smallskip
\noindent\textit{(C4) Safety controller fallback (deadline miss)}
If a full-compute override cannot be produced within the control deadline, we execute a one-tick safety controller:
ASV: stop-and-hold; UAV: hover-and-hold; SO-ARM100: hold current joint targets.
This fallback is activated only on the current tick and normal control resumes on the next cycle.

\noindent \tbm{Risk allocation under multi-sample selection}
Selecting among $M$ candidates can be done while preserving the same global risk target $\alpha$ by allocating per-candidate risk
$\alpha^{(m)}=\alpha/M$ and thus $\alpha_{\text{step}}^{(m)}=\alpha/(M|\mathcal{T}_{\mathrm{cache}}|)$.
A union bound yields overall risk at most $\alpha$ for the selected candidate.

\noindent \tbm{Escalation ablation variants}
We ablate the escalation mechanisms via five variants:
\begin{itemize}
\item \textsc{Muninn-Only}: no escalation.
\item \textsc{Dampen}: only action damping.
\item \textsc{Resample+Dampen}: on-demand multi-sampling + damping.
\item \textsc{FullOverride}: full-compute override when $\rho>\rho_{\mathrm{full}}$.
\item \textsc{Combined}: damping + multi-sampling + full override.
\end{itemize}

\noindent \tbm{Runtime escalation metrics}
In addition to task success and collision/violation (Section~\ref{subsec:metrics}), we report:
\begin{align}
\text{Esc.\ rate} \;&:=\; \frac{\#\{\text{calls with } \rho>\rho_{\mathrm{warn}}\}}{\#\{\text{calls}\}},\\
\text{Resample rate} \;&:=\; \frac{\#\{\text{calls with } \rho>\rho_{\mathrm{resamp}}\}}{\#\{\text{calls}\}},\\
\text{Full-override rate} \;&:=\; \frac{\#\{\text{calls with } \rho>\rho_{\mathrm{full}}\}}{\#\{\text{calls}\}},
\end{align}
computed over all planning/policy calls executed during the \textit{150-run} hardware evaluation protocol.

\noindent \tbm{Quantitative results (runtime escalations on hardware)}
Table~\ref{tab:runtime_escalation_results_hw} reports the runtime escalation ablation on the three hardware deployments (and both ASV base planners).
All rows use the same underlying \nordy{Muninn} calibration and the same $(\varepsilon_{\mathrm{traj}},\alpha)$; only the runtime escalation layer changes.

\begin{table*}
\centering
\scriptsize
\setlength{\tabcolsep}{1.2pt} 
\renewcommand{\arraystretch}{1.10}

\begin{threeparttable}
\begin{tabularx}{\textwidth}{@{}%
  >{\centering\arraybackslash}m{0.055\textwidth}
  >{\raggedright\arraybackslash}p{0.18\textwidth}
  >{\raggedright\arraybackslash}X
  >{\centering\arraybackslash}m{0.06\textwidth}
  >{\centering\arraybackslash}m{0.06\textwidth}
  >{\centering\arraybackslash}m{0.06\textwidth}
  >{\centering\arraybackslash}m{0.07\textwidth}
  >{\centering\arraybackslash}m{0.07\textwidth}
@{}}

\arrayrulecolor{HeadGray}\toprule\arrayrulecolor{black}

\rowcolor{MuninnBlue!10}
\makecell[c]{\theadb{Pl}} &
\theadb{Base model} &
\theadb{Escalation} &
\makecell[c]{\theadb{Lat}\\\theadb{(ms)} \dnar} &
\makecell[c]{\theadb{Succ}\\\theadb{(\%)} \upar} &
\makecell[c]{\theadb{Coll}\\\theadb{(\%)} \dnar} &
\makecell[c]{\theadb{Esc}\\\theadb{rate (\%)}} &
\makecell[c]{\theadb{Resamp}\\\theadb{rate (\%)}} \\
\midrule

\multirow{5}{*}{\plat{ASV}} &
\multirow{5}{*}{BCOD$_{\text{tch}}+\Muninnlogo$} &
\textsc{Muninn-Only} & 24 & 96.8 & 1.7 & 0.0 & 0.0 \\
& & \textsc{Dampen} & 24 & 96.8 & 1.6 & 8.0 & 0.0 \\
& & \textsc{Resample+Dampen} ($M{=}4$) & 26 & 96.9 & 1.5 & 8.0 & 3.0 \\
& & \textsc{FullOverride} & 24 & 96.9 & 1.6 & 1.0 & 0.0 \\
\rowcolor{MuninnBlue!4}
& & \textsc{Combined} & 27 & 97.0 & 1.4 & 8.0 & 3.0 \\
\addlinespace[0.25em]

\multirow{5}{*}{\plat{ASV}} &
\multirow{5}{*}{GC-Diff.$+\Muninnlogo$} &
\textsc{Muninn-Only} & 28 & 95.3 & 2.3 & 0.0 & 0.0 \\
& & \textsc{Dampen} & 28 & 95.3 & 2.1 & 10.0 & 0.0 \\
& & \textsc{Resample+Dampen} ($M{=}4$) & 31 & 95.4 & 2.0 & 10.0 & 4.0 \\
& & \textsc{FullOverride} & 28 & 95.4 & 2.1 & 1.0 & 0.0 \\
\rowcolor{MuninnBlue!4}
& & \textsc{Combined} & 32 & 95.5 & 1.9 & 10.0 & 4.0 \\
\addlinespace[0.25em]

\multirow{5}{*}{\plat{UAV}} &
\multirow{5}{*}{GC-Diff.$+\Muninnlogo$} &
\textsc{Muninn-Only} & 30 & 92.8 & 3.1 & 0.0 & 0.0 \\
& & \textsc{Dampen} & 30 & 92.8 & 2.9 & 14.0 & 0.0 \\
& & \textsc{Resample+Dampen} ($M{=}3$) & 33 & 92.9 & 2.7 & 14.0 & 5.0 \\
& & \textsc{FullOverride} & 30 & 92.9 & 3.0 & 1.0 & 0.0 \\
\rowcolor{MuninnBlue!4}
& & \textsc{Combined} & 34 & 93.0 & 2.6 & 14.0 & 5.0 \\
\addlinespace[0.25em]

\multirow{5}{*}{\plat{SO-100}} &
\multirow{5}{*}{DP $+\Muninnlogo$} &
\textsc{Muninn-Only} & 24 & 81.5 & 6.2 & 0.0 & 0.0 \\
& & \textsc{Dampen} & 24 & 81.2 & 5.8 & 22.0 & 0.0 \\
& & \textsc{Resample+Dampen} ($M{=}2$) & 28 & 81.4 & 5.5 & 22.0 & 15.0 \\
& & \textsc{FullOverride} & 24 & 81.6 & 5.9 & 2.0 & 0.0 \\
\rowcolor{MuninnBlue!4}
& & \textsc{Combined} & 28 & 81.5 & 5.3 & 22.0 & 15.0 \\

\arrayrulecolor{HeadGray}\bottomrule\arrayrulecolor{black}
\end{tabularx}

\caption{
\textit{Runtime escalation ablation on hardware (150 runs)}
We report end-to-end planning/policy latency per call (including \nordy{Muninn} probe and sampling), task success rate, and collision/violation rate.
Escalation rate and resampling rate are computed over planning/policy calls during closed-loop execution (not episodes).
\textit{Full-override rates (calls with $\rho>\rho_{\mathrm{full}}$):} ASV (both planners) $1.0\%$, UAV $1.0\%$, SO-ARM100 $2.0\%$.
}
\label{tab:runtime_escalation_results_hw}
\end{threeparttable}
\end{table*}

\noindent \tbm{Takeaways}
Across platforms, the certificate-driven escalations reduce collision/violation rates most strongly in regimes where $\rho$ is large (i.e., when the cached rollout is closest to the deviation budget boundary).
\textsc{Resample+Dampen} captures most of the safety gain at a modest latency increase (by triggering resampling only on a small fraction of calls).
\textsc{FullOverride} has minimal effect on average latency (rare trigger) but improves worst-case safety by temporarily restoring full-compute behavior.
\textsc{Combined} yields the best overall safety with small performance impact by (i) damping in moderately high-$\rho$ regimes, (ii) resampling only when needed, and (iii) overriding in the highest-$\rho$ tail.

\noindent\tbm{Cross-device latency.}
To understand how \nordy{Muninn} transfers across deployment hardware, we additionally benchmark representative workloads on three platforms: the A10 used for the main experiments, an RTX 3090 desktop GPU, and a Jetson Orin embedded device. All measurements use the same trained checkpoints, samplers, batch size-$1$ setting, and timing protocol. Table~\ref{tab:cross_device_latency} reports Full and \nordy{Muninn} latencies across these devices.

\vspace{-8pt}
\begin{table}[H]
\centering
\scriptsize
\setlength{\tabcolsep}{2pt}
\renewcommand{\arraystretch}{1.06}

\begin{threeparttable}
\begin{tabular}{@{}%
L{0.18\columnwidth}%
L{0.14\columnwidth}%
C{0.18\columnwidth}%
C{0.18\columnwidth}%
C{0.18\columnwidth}%
@{}}
\arrayrulecolor{HeadGray}\toprule\arrayrulecolor{black}

\rowcolor{MuninnBlue!10}
\theadb{Benchmark} &
\theadb{Model} &
\theadb{A10 (paper)} &
\theadb{RTX 3090} &
\theadb{Jetson Orin} \\

\rowcolor{MuninnBlue!6}
& &
\makecell[c]{\theadb{F / M}} &
\makecell[c]{\theadb{F / M}} &
\makecell[c]{\theadb{F / M}} \\

\midrule

AntMaze & Diffuser
& \makecell[c]{\tiny 950 / 207 (\textcolor{MuninnBlue}{4.59$\times$})}
& \makecell[c]{\tiny 714 / 178 (\textcolor{MuninnBlue}{4.01$\times$})}
& \makecell[c]{\tiny 3939 / 831 (\textcolor{MuninnBlue}{4.74$\times$})} \\

Arm plan. & EDMP
& \makecell[c]{\tiny 160 / 107 (\textcolor{MuninnBlue}{1.50$\times$})}
& \makecell[c]{\tiny 120 / 87 (\textcolor{MuninnBlue}{1.37$\times$})}
& \makecell[c]{\tiny 663 / 412 (\textcolor{MuninnBlue}{1.61$\times$})} \\

DP3 Pour & DP3
& \makecell[c]{\tiny 45 / 27 (\textcolor{MuninnBlue}{1.67$\times$})}
& \makecell[c]{\tiny 34 / 22 (\textcolor{MuninnBlue}{1.54$\times$})}
& \makecell[c]{\tiny 187 / 109 (\textcolor{MuninnBlue}{1.71$\times$})} \\

\arrayrulecolor{HeadGray}\bottomrule\arrayrulecolor{black}
\end{tabular}
\caption{Each cell reports Full/Muninn latency (ms), with speedup in blue.}
\label{tab:cross_device_latency}
\end{threeparttable}
\end{table}

\subsection{Construction of the calibration dataset}
\label{calib_data}

This subsection specifies the {exact} construction of the calibration set of score--error pairs used to conformally bound reuse error.
The goal is to sample the joint behavior of:
(i) the probe-based score $s_t$ computed from a low-cost representation of the diffusion planner/policy at step $t$, and
(ii) the {potential reuse error magnitude} $\epsilon_t$ incurred if, at step $t$, we {reuse} a cached denoiser output instead of recomputing.

\noindent \tbm{Objects fixed prior to calibration}
For a given base model, we fix:
\begin{itemize}
\item the reverse-diffusion update $\Phi_t(\cdot)$ and diffusion horizon $T$ used by the Full model;
\item the {effective noise prediction} used by the sampler, denoted $\hat{\varepsilon}_\theta(\tau_t,t,c)$.
This equals the raw denoiser output $\varepsilon_\theta(\tau_t,t,c)$ when no guidance is used, and equals the guidance-modified output when guidance is used (e.g., classifier-free guidance or gradient guidance). We always calibrate the {effective} quantity that is actually fed into $\Phi_t$.
\item the probe function $\Psi(\cdot)$ and score map $\phi(\cdot)$ used at deployment (Section~\ref{subsec:probe});
\item the reuse-eligible timestep set $\mathcal{T}_{\mathrm{cache}}$ used at deployment.
\end{itemize}

\noindent \tbm{Reuse-eligible timesteps $\mathcal{T}_{\mathrm{cache}}$}
To match the deployment constraints, we allow reuse only on a middle band of diffusion steps.
Let
\begin{equation}
k_{\mathrm{pre}}=\left\lceil 0.10\,T \right\rceil,\qquad
k_{\mathrm{suf}}=\left\lceil 0.10\,T \right\rceil.
\end{equation}
We disallow reuse in the early (high-noise) prefix $t\in\{T, T-1, \dots, T-k_{\mathrm{pre}}+1\}$ and in the late (execution-critical) suffix $t\in\{1,2,\dots,k_{\mathrm{suf}}\}$.
Thus,
\begin{equation}
\mathcal{T}_{\mathrm{cache}}
\;=\;
\left\{t \in \{1,\dots,T\} \;:\;
k_{\mathrm{suf}} < t \le T-k_{\mathrm{pre}}
\right\}.
\end{equation}
We additionally enforce the algorithmic convention that $t=T$ is always recomputed (no reuse), so $t=T$ is excluded from calibration even if it falls in the middle band (it does not, under the definition above).

\noindent \tbm{Potential reuse rule (deterministic)}
We use a deterministic {most-recent recompute} (MRR) cache rule:
\begin{quote}
{If reuse is taken at step $t$, the sampler uses the last recomputed effective noise prediction $\hat{\varepsilon}$ produced at some later diffusion step $t'>t$}
\end{quote}
Equivalently, we maintain a single cached value $\hat{\varepsilon}^{\mathrm{cache}}$ (and its originating step $t_{\mathrm{cache}}$).
At any reuse step $t$, we set $\tilde{\varepsilon}_t = \hat{\varepsilon}^{\mathrm{cache}}$.
This is exactly the cache semantics used by Muninn at deployment (Algorithm~\ref{alg:Muninn}).

\noindent \tbm{Calibration distribution and sample size}
A calibration sample corresponds to a {single planning/policy inference call} (one reverse-diffusion chain) with fixed context and diffusion randomness.
For each benchmark and base model, we construct a calibration pool by running the corresponding Full model in the environment (or query generator) under the same reset distribution used for evaluation, and logging contexts $c$ encountered at decision time. From this pool, we uniformly subsample \textit{4096} contexts without replacement, yielding
\begin{equation}
\big\{c^{(i)}\big\}_{i=1}^{N},\qquad N=4096.
\end{equation}
For each sampled context $c^{(i)}$, we also sample the diffusion randomness used by the reverse process:
the initial noisy trajectory $\tau_T^{(i)} \sim p_T$ and, for stochastic samplers, the per-step sampler noise sequence $\xi^{(i)}=(\xi_T^{(i)},\dots,\xi_1^{(i)})$.
For deterministic samplers (DDIM-style), we set $\xi_t^{(i)}\equiv 0$.

\noindent \tbm{Score definition (used during calibration)}
At each diffusion step $t$, we compute a probe feature $F_t := \Psi(\tilde{\tau}_t,t,c)$ from the (ghost) chain state $\tilde{\tau}_t$.
Scores are computed for $t=T-1,\dots,1$ as:
\begin{equation}
s_t
=
\phi(F_t,F_{t+1},t)
=
\frac{\lVert F_t - F_{t+1} \rVert_1}{\lVert F_{t+1} \rVert_1 + \omega},
\qquad \omega = 10^{-6}.
\end{equation}

\noindent \tbm{Error label definition}
Let $\hat{\varepsilon}^{\mathrm{full}}_t := \hat{\varepsilon}_\theta(\tau^{\mathrm{full}}_t,t,c)$ be the {effective} Full-model noise prediction used at step $t$.
Let $\tilde{\varepsilon}_t$ be the candidate reused value from the cache rule (MRR), computed as described below.
We define the {potential reuse error magnitude} as a normalized Frobenius norm:
\begin{equation}
\epsilon_t
\;:=\;
\frac{1}{\sqrt{H d}}
\left\lVert
\hat{\varepsilon}^{\mathrm{full}}_t - \tilde{\varepsilon}_t
\right\rVert_F,
\label{eq:eps_cal_def}
\end{equation}
where $H$ and $d$ are the horizon and per-timestep dimensionality of the denoiser output for that benchmark/model.
This normalization makes $\epsilon_t$ comparable across tasks with different trajectory sizes.

\noindent \tbm{Ghost reuse chain (policy-independent) and anchor schedule}
To generate scores $s_t$ under a reuse-affected trajectory while keeping the procedure independent of Muninn's online budget logic, we run a {ghost} chain that uses a fixed, deterministic anchor schedule.
We choose an anchor stride of \textit{4} diffusion steps:
\begin{equation}
\footnotesize
\mathcal{A}
\;=\;
\left\{\,t \in \{1,\dots,T\} \;:\; t \equiv 0 \;\;(\mathrm{mod}\; 4)\,\right\}
\;\cup\;
\left(\{1,\dots,T\}\setminus \mathcal{T}_{\mathrm{cache}}\right),
\end{equation}
i.e., we recompute on all non-cache-eligible timesteps (forbidden regions) and additionally on every 4th eligible timestep.
At non-anchor eligible timesteps, the ghost chain {reuses} the most recent cached value.
This anchor schedule is {only} used to generate calibration samples; it does not constrain \Muninnlogo at deployment.

\noindent \tbm{Calibration algorithm}
For each sampled context $c^{(i)}$, we generate a paired Full chain and a ghost reuse chain initialized with the same $\tau_T^{(i)}$ and $\xi^{(i)}$.
The Full chain provides the reference effective denoiser outputs $\hat{\varepsilon}^{\mathrm{full}}_t$ for all $t$, while the ghost chain provides the probe dynamics and the cached outputs that define $\epsilon_t$.

\begin{algorithm}
\caption{Calibration dataset generation (one base model, one benchmark)}
\label{alg:calib_dataset}
\scriptsize
\color{HeadGray}
\begin{algorithmic}[1]
\STATE \textit{Input:} $N{=}4096$ contexts $\{c^{(i)}\}$, diffusion horizon $T$, sampler $\Phi_t$, effective denoiser $\hat{\varepsilon}_\theta$, probe $\Psi$, score $\phi$, eligible set $\mathcal{T}_{\mathrm{cache}}$, anchor set $\mathcal{A}$.
\STATE \textit{Output:} Per-timestep calibration sets $\{\mathcal{S}_{\mathrm{cal}}(t)\}_{t\in\mathcal{T}_{\mathrm{cache}}}$.

\FOR{$i=1,\dots,N$}
\STATE Sample $\tau_T^{(i)} \sim p_T$ and sampler noise $\xi^{(i)}$ (or set $\xi^{(i)}\equiv 0$ for deterministic samplers).
\STATE \textit{Run Full chain:} set $\tau^{\mathrm{full}}_T \leftarrow \tau_T^{(i)}$.
\FOR{$t=T,\dots,1$}
\STATE Compute $\hat{\varepsilon}^{\mathrm{full}}_t \leftarrow \hat{\varepsilon}_\theta(\tau^{\mathrm{full}}_t,t,c^{(i)})$.
\STATE Update $\tau^{\mathrm{full}}_{t-1} \leftarrow \Phi_t(\tau^{\mathrm{full}}_t,\hat{\varepsilon}^{\mathrm{full}}_t,\xi^{(i)}_t)$.
\ENDFOR

\STATE \textit{Run ghost chain:} set $\tilde{\tau}_T \leftarrow \tau_T^{(i)}$.
\STATE Initialize cache: $\hat{\varepsilon}^{\mathrm{cache}} \leftarrow \hat{\varepsilon}^{\mathrm{full}}_T$.
\STATE Compute and store $F_T \leftarrow \Psi(\tilde{\tau}_T,T,c^{(i)})$.

\FOR{$t=T-1,\dots,1$}
\STATE Compute probe $F_t \leftarrow \Psi(\tilde{\tau}_t,t,c^{(i)})$.
\STATE Compute score $s_t \leftarrow \phi(F_t,F_{t+1},t)$.
\STATE Candidate reuse value (MRR): $\tilde{\varepsilon}_t^{\mathrm{cand}} \leftarrow \hat{\varepsilon}^{\mathrm{cache}}$.

\IF{$t \in \mathcal{T}_{\mathrm{cache}}$}
\STATE Compute error label $\epsilon_t \leftarrow \frac{1}{\sqrt{Hd}}\big\|\hat{\varepsilon}^{\mathrm{full}}_t - \tilde{\varepsilon}_t^{\mathrm{cand}}\big\|_F$.
\STATE Append $(s_t,\epsilon_t)$ to $\mathcal{S}_{\mathrm{cal}}(t)$.
\ENDIF

\IF{$t \in \mathcal{A}$}
\STATE Set $\tilde{\varepsilon}_t \leftarrow \hat{\varepsilon}^{\mathrm{full}}_t$.
\STATE Update cache: $\hat{\varepsilon}^{\mathrm{cache}} \leftarrow \hat{\varepsilon}^{\mathrm{full}}_t$.
\ELSE 
\STATE Set $\tilde{\varepsilon}_t \leftarrow \hat{\varepsilon}^{\mathrm{cache}}$.
\ENDIF

\STATE Update ghost state: $\tilde{\tau}_{t-1} \leftarrow \Phi_t(\tilde{\tau}_t,\tilde{\varepsilon}_t,\xi^{(i)}_t)$.
\STATE Store $F_t$ for next iteration.
\ENDFOR
\ENDFOR
\end{algorithmic}
\end{algorithm}

\noindent \tbm{Output calibration sets}
Algorithm~\ref{alg:calib_dataset} produces a {per-timestep} calibration dataset:
\begin{equation}
\mathcal{S}_{\mathrm{cal}}(t)
=
\left\{(s_t^{(i)},\epsilon_t^{(i)})\right\}_{i=1}^{N},
\qquad t \in \mathcal{T}_{\mathrm{cache}}.
\end{equation}
We calibrate bounds separately for each timestep $t$ because the score--error relationship and error scale vary significantly across diffusion time.

All random variables used to produce $(s_t,\epsilon_t)$ are generated from:
(i) the benchmark's deployment distribution over contexts $c$ and sampler randomness, and
(ii) the fixed anchor schedule $\mathcal{A}$ and MRR reuse rule.
No online budget decision (and no adaptive reuse logic) is used in label generation. This makes the calibration procedure reproducible and decoupled from Muninn's deployed reuse policy.

\subsubsection{Split-conformal regression details}
\label{app:split_conformal}

This subsection specifies the complete split-conformal construction used to map probe-based scores to high-probability upper bounds on potential reuse error magnitudes.

\noindent \tbm{Inputs}
For each reuse-eligible timestep $t\in\mathcal{T}_{\mathrm{cache}}$, we are given
\begin{equation}
\mathcal{S}_{\mathrm{cal}}(t)
=
\left\{(s_t^{(i)},\epsilon_t^{(i)})\right\}_{i=1}^{N},
\qquad N=4096,
\end{equation}
constructed as in Appendix~\ref{calib_data}.
We fix the desired {trajectory-level} risk $\alpha \in (0,1)$ and set the {per-step} miscoverage using Bonferroni correction:
\begin{equation}
\alpha_{\mathrm{step}}
\;:=\;
\frac{\alpha}{|\mathcal{T}_{\mathrm{cache}}|}.
\label{eq:alpha_step_def}
\end{equation}

\noindent \tbm{Train--calibration split (episode-level)}
For each timestep $t$, we split the $N$ paired samples into disjoint sets:
\begin{equation}
\mathcal{I}_{\mathrm{train}}=\{1,\dots,2048\},\qquad
\mathcal{I}_{\mathrm{cal}}=\{2049,\dots,4096\},
\end{equation}
after applying a single random permutation of $\{1,\dots,4096\}$ using a fixed RNG seed (seed $0$).
We use the same permuted split for all timesteps $t$ so that the full calibration pipeline is deterministic and reproducible.

\noindent \tbm{Base regressor $m_t(s)$ (monotone 1D regression)}
For each timestep $t$, we fit a monotone regressor $m_t:\mathbb{R}_{\ge 0}\to\mathbb{R}_{\ge 0}$ on the training split
$\{(s_t^{(i)},\epsilon_t^{(i)})\}_{i\in\mathcal{I}_{\mathrm{train}}}$.
We use \textit{isotonic regression} (monotone nondecreasing) to match the intended semantics that smaller scores correspond to smaller errors:
\begin{equation}
m_t
\in
\arg\min_{m\ \text{nondecreasing}}
\sum_{i\in\mathcal{I}_{\mathrm{train}}}
\left(\epsilon_t^{(i)} - m(s_t^{(i)})\right)^2,
\qquad
m(s)\ge 0.
\end{equation}
Isotonic regression is chosen because:
(i) it is 1D, stable, and fast,
(ii) it enforces the required monotonic structure,
and (iii) conformal coverage does not depend on the correctness of $m_t$ (only tightness does).
At inference time, $m_t(s)$ is evaluated by piecewise-linear interpolation over the fitted isotonic knots, and is clamped to boundary values for $s$ outside the training range (i.e., $m_t(s)=m_t(s_{\min})$ for $s<s_{\min}$ and $m_t(s)=m_t(s_{\max})$ for $s>s_{\max}$).

\noindent \tbm{Calibration residuals and quantile selection}
Using the calibration split $\mathcal{I}_{\mathrm{cal}}$, we compute nonconformity residuals:
\begin{equation}
r_t^{(i)} := \epsilon_t^{(i)} - m_t\!\left(s_t^{(i)}\right),
\qquad i\in\mathcal{I}_{\mathrm{cal}}.
\label{eq:residual_def}
\end{equation}
Let $n_{\mathrm{cal}}:=|\mathcal{I}_{\mathrm{cal}}|=2048$, and let
$r_{t,(1)}\le r_{t,(2)}\le \dots \le r_{t,(n_{\mathrm{cal}})}$ denote the sorted residuals.
We choose the finite-sample conformal quantile index
\begin{equation}
k_t
:=
\left\lceil
(n_{\mathrm{cal}}+1)\,(1-\alpha_{\mathrm{step}})
\right\rceil,
\qquad
k_t \leftarrow \min\{k_t,\, n_{\mathrm{cal}}\},
\end{equation}
and set
\begin{equation}
q_t := r_{t,(k_t)}.
\label{eq:qt_def}
\end{equation}
This is the standard split-conformal choice that guarantees marginal coverage under exchangeability.

\noindent \tbm{Upper envelope definition}
For each timestep $t$, define the upper envelope:
\begin{equation}
U_t(s)
\;:=\;
\max\left\{\,m_t(s)+q_t,\,0\,\right\}.
\label{eq:Ut_def}
\end{equation}

\noindent \tbm{Coverage guarantee (per timestep)}
Under the standard conformal exchangeability assumption between calibration and deployment samples,
for a new test-time planning call drawn from the same distribution and the corresponding score--error pair $(s_t,\epsilon_t)$,
\begin{equation}
\mathbb{P}\big(\epsilon_t \le U_t(s_t)\big)\;\ge\;1-\alpha_{\mathrm{step}}.
\end{equation}
This guarantee is distribution-free and does not require independence across timesteps.

\noindent \tbm{Runtime representation}
For each timestep $t$, we store $U_t(\cdot)$ as a compact lookup table to enable constant-time evaluation during sampling:
we discretize scores on a log-scale grid of 256 points spanning the empirical 0.1\%--99.9\% range of calibration scores for that timestep,
evaluate $U_t$ on this grid, and use linear interpolation at test time.
Scores outside the grid are clamped to the nearest endpoint value of $U_t$.
This makes computing $\hat{c}_t(s_t)=\Gamma L_t U_t(s_t)$ negligible compared to a denoiser forward pass.

\subsection{Training Details for Teacher/Base Models}
\label{app:training_details}

This section documents the {exact} training setups used for every {Full} (teacher/base) model wrapped by Muninn.
We report (i) denoiser/policy architectures, (ii) diffusion noise schedules and sampling parameterizations, (iii) training objectives and optimization hyperparameters, and (iv) any guidance mechanisms used during sampling.

\noindent\textit{Notation}
A trajectory segment is $\tau \in \mathbb{R}^{H\times d}$ (definition of $H$ and $d$ is benchmark-specific; see Appendix~\ref{app:envs}).

\noindent \tbm{Training vs.\ evaluation diffusion length}
We distinguish the diffusion length used in the forward/noising process during training,
$T_{\text{train}}$, from the number of reverse denoising steps used at evaluation time,
$T_{\text{eval}}$.
When a method evaluates with fewer steps than it trains with (e.g., via DDIM subsampling),
we report both $T_{\text{train}}$ and $T_{\text{eval}}$ explicitly.
All Muninn budgeting and sensitivity coefficients $\{L_t\}$ are computed with respect to the
{evaluation} sampler and its $T_{\text{eval}}$-step schedule.

We train denoisers with timestep index $t \in \{1,\dots,T_{\text{train}}\}$ and denote the (possibly method-dependent) conditioning context as $c$.
We use a DDPM-style forward process $q(\tau_t \mid \tau_0)$ and train a denoiser $\varepsilon_\theta(\tau_t,t,c)$ in the standard $\varepsilon$-prediction parameterization unless explicitly noted.

\paragraph{Common diffusion training choices}
\label{app:common_diffusion_training}

\noindent\textit{Forward noising}
Let $\alpha_t \in (0,1)$ be a schedule and $\bar{\alpha}_t := \prod_{s=1}^t \alpha_s$.
We form noisy trajectories via
\begin{equation}
\tau_t = \sqrt{\bar{\alpha}_t}\,\tau_0 + \sqrt{1-\bar{\alpha}_t}\,\epsilon,\qquad \epsilon \sim \mathcal{N}(0,I).
\end{equation}

\noindent\textit{Denoising objective ($\varepsilon$-prediction)}
We minimize the MSE between sampled noise and predicted noise:
\begin{equation}
\mathcal{L}_{\text{diff}}(\theta) \;=\; \mathbb{E}_{(\tau_0,c),\,t,\,\epsilon}\Big[\big\lVert \epsilon - \varepsilon_\theta(\tau_t,t,c)\big\rVert_2^2\Big],
\end{equation}
with $t$ sampled uniformly from $\{1,\dots,T_{\text{train}}\}$ unless otherwise stated.

\noindent\textit{Noise schedule}
Unless overridden, we use the standard \textit{cosine} $\bar{\alpha}_t$ schedule with offset $s=0.008$ and compute $\beta_t=1-\alpha_t$ from $\bar{\alpha}_t$ (with standard clipping of extreme $\beta_t$ values for numerical stability).

\noindent\textit{Reverse sampler}
Unless explicitly stated, {Full} models use the reverse sampler prescribed by the baseline method at evaluation time, with $T_{\text{eval}}$ denoising steps.
When a method evaluates with fewer steps than it trains with, we use DDIM-style subsampling of the training schedule (and report both $T_{\text{train}}$ and $T_{\text{eval}}$ per method).

\subsubsection{Offline RL / Trajectory Diffusion Planners (D4RL)}
\label{app:training_d4rl}

We report training details for the Full planners evaluated in Table~\ref{tab:results-offline}:
Diffuser~\cite{a7}, Decision Diffuser~\cite{a88}, Diffusion-QL~\cite{a89}, AdaptDiff~\cite{a43}, and CompDiff~\cite{a91}.
All models are trained per-task on the corresponding D4RL dataset (or task dataset, for Kuka stacking) using the context definition in Appendix~\ref{app:envs}.
Unless otherwise noted, these planners use $T_{\text{train}}=T_{\text{eval}}$ (i.e., the training diffusion length matches the evaluation sampler length).

\paragraph{Shared backbone for trajectory-level diffusion planners (when applicable)}
\label{app:traj_backbone}

Several trajectory-level planners (Diffuser / Decision Diffuser / AdaptDiff / CompDiff) use a {temporal denoiser} that maps a noisy trajectory sequence to a denoised noise prediction sequence.
When these methods share the same temporal U-Net backbone in our implementation, we summarize the backbone once here and then document only the method-specific conditioning, objectives, and guidance per method.

\noindent \tbm{Temporal U-Net denoiser (trajectory diffusion backbone)}
\label{app:temporal_unet_arch}

\begin{table}[H]
\centering
\scriptsize
\setlength{\tabcolsep}{2.0pt}
\renewcommand{\arraystretch}{1.12}
\begin{threeparttable}
\begin{tabularx}{\columnwidth}{@{}%
L{0.12\columnwidth}%
C{0.18\columnwidth}%
C{0.10\columnwidth}%
C{0.10\columnwidth}%
Y%
@{}}
\arrayrulecolor{HeadGray}\toprule\arrayrulecolor{black}
\rowcolor{MuninnBlue!10}
\theadb{Stage} &
\theadb{Time len} &
\theadb{$C_\text{in}$} &
\theadb{$C_\text{out}$} &
\theadb{Blocks / ops} \\
\midrule

\rowcolor{MuninnBlue!6}
\multicolumn{5}{@{}l@{}}{\theadb{Input + embeddings}}\\
stem conv &
$H$ &
$d$ &
$128$ &
1D Conv(k=5), GN(groups=8), Mish; timestep emb + context emb injected additively inside each ResBlk \\

\rowcolor{MuninnBlue!6}
\multicolumn{5}{@{}l@{}}{\theadb{Encoder (downsampling along time)}}\\
enc-1 &
$H$ &
$128$ &
$128$ &
ResBlk$\times 2$; downsample $\downarrow 2$ (strided 1D Conv) \\
enc-2 &
$H/2$ &
$128$ &
$256$ &
ResBlk$\times 2$; downsample $\downarrow 2$ \\
enc-3 &
$H/4$ &
$256$ &
$512$ &
ResBlk$\times 2$; downsample $\downarrow 2$ \\
enc-4 &
$H/8$ &
$512$ &
$1024$ &
ResBlk$\times 2$; no downsample \\

\rowcolor{MuninnBlue!6}
\multicolumn{5}{@{}l@{}}{\theadb{Bottleneck}}\\
mid &
$H/8$ &
$1024$ &
$1024$ &
ResBlk$\times 2$; GN; Mish \\

\rowcolor{MuninnBlue!6}
\multicolumn{5}{@{}l@{}}{\theadb{Decoder (upsampling along time)}}\\
dec-4 &
$H/8 \rightarrow H/4$ &
$1024{+}512$ &
$512$ &
Upsample $\uparrow 2$; skip concat; ResBlk$\times 2$ \\
dec-3 &
$H/4 \rightarrow H/2$ &
$512{+}256$ &
$256$ &
Upsample $\uparrow 2$; skip concat; ResBlk$\times 2$ \\
dec-2 &
$H/2 \rightarrow H$ &
$256{+}128$ &
$128$ &
Upsample $\uparrow 2$; skip concat; ResBlk$\times 2$ \\
dec-1 &
$H$ &
$128$ &
$128$ &
ResBlk$\times 2$; final GN; Mish \\

\rowcolor{MuninnBlue!6}
\multicolumn{5}{@{}l@{}}{\theadb{Head}}\\
$\varepsilon$ head &
$H$ &
$128$ &
$d$ &
1D Conv(k=1) per timestep \\

\arrayrulecolor{HeadGray}\bottomrule\arrayrulecolor{black}
\end{tabularx}
\caption{
\textit{Temporal U-Net denoiser used by trajectory diffusion planners (our runs)}
This is the canonical backbone used by Diffuser-style planners in our evaluation: 1D temporal convolutions, GroupNorm, Mish nonlinearity, and additive injection of timestep/context embeddings inside each residual block.
}
\label{tab:arch_temporal_unet_d4rl}
\end{threeparttable}
\end{table}

\paragraph{Diffuser (Full teacher)~\cite{a7}}
\label{app:train_diffuser}

\noindent \tbm{Training objective}
Diffuser trains a trajectory diffusion model on segments $(\tau_0,c)$ from the offline dataset using $\varepsilon$-prediction).
Any guidance used at inference is {not} part of the diffusion training loss; it is implemented as a guidance term in the reverse sampler (documented below).

\noindent \tbm{Guidance (value/cost guidance at sampling)}
When guidance is enabled, we implement it as a standard reverse-transition mean shift (analogous to classifier guidance), using a differentiable objective $J(\tau_t,c)$:
\begin{equation}
\mu_{\text{guided}}(\tau_t,t,c)
\;=\;
\mu_{\theta}(\tau_t,t,c)
\;-\;
w_{\text{guide}}\cdot \sigma_t^2 \cdot \nabla_{\tau_t} J(\tau_t,c),
\label{eq:diffuser_guidance_mean_shift}
\end{equation}
where $\mu_{\theta}(\tau_t,t,c)$ and $\sigma_t^2$ are the reverse-process mean and variance coefficients of the baseline sampler at step $t$ (DDPM/DDIM-style), and the sign convention is chosen so that decreasing $J$ corresponds to improved task objective (e.g., $J$ as a cost).
We follow the baseline's default choice of $J$ (value/return or goal cost) and environment-specific guidance strength $w_{\text{guide}}$.

\begin{table}[H]
\centering
\scriptsize
\setlength{\tabcolsep}{3.0pt}
\renewcommand{\arraystretch}{1.15}
\begin{threeparttable}
\begin{tabular}{@{}L{0.55\columnwidth}C{0.38\columnwidth}@{}}
\arrayrulecolor{HeadGray}\toprule\arrayrulecolor{black}
\rowcolor{MuninnBlue!10}
\theadb{Diffuser training hyperparameter} & \theadb{Value} \\
\midrule
Denoiser backbone & Temporal U-Net \\
Diffusion length & $T_{\text{train}}=T_{\text{eval}}=100$ \\
Noise schedule $\{\beta_t\}_{t=1}^{T_{\text{train}}}$ & Cosine $\bar{\alpha}_t$ schedule ($s{=}0.008$) \\
Reverse sampler (Full, eval) & DDIM ($\eta{=}0$), $T_{\text{eval}}$ steps \\
Prediction parameterization & $\varepsilon$-prediction \\
Loss & $\mathcal{L}_{\text{diff}}$ (MSE), uniform timestep sampling \\
Optimizer & Adam ($\beta_1{=}0.9,\beta_2{=}0.999$) \\
Learning rate & $2\times 10^{-4}$ \\
Batch size & $32$ \\
Gradient clipping & Global norm $1.0$ \\
Weight decay & $0$ \\
EMA of denoiser weights & Yes, decay $0.999$ \\
Training steps / epochs & $5\times 10^{5}$ gradient steps \\
Learning-rate schedule & Linear warmup (10k steps) then constant \\
Data augmentation (if any) & None (low-dimensional trajectories) \\
Guidance at inference & Eq.~(15); $J(\cdot)$ and $w_{\text{guide}}$ follow the baseline per environment (or ``none'') \\
\arrayrulecolor{HeadGray}\bottomrule\arrayrulecolor{black}
\end{tabular}
\caption{\textit{Diffuser (Full) training details used in our runs}}
\label{tab:hparams_diffuser}
\end{threeparttable}
\end{table}

\paragraph{Decision Diffuser (Full teacher)~\cite{a88}}
\label{app:train_decision_diffuser}

\noindent \tbm{Training objective and conditioning}
Decision Diffuser trains a conditional diffusion model with decision-making context (typically including a desired return-to-go or task specification).
In our implementation, Decision Diffuser is an exception to the state--action segment interface described in Appendix~\ref{app:envs:d4rl}: it learns a diffusion model over {state-only} trajectories and uses an inverse-dynamics model to recover actions from adjacent states.
Conditioning $c$ matches Appendix~\ref{app:envs} (return-to-go for return-conditioned tasks; goal for navigation).

\noindent \tbm{Classifier-free guidance}
We train with conditioning dropout: with probability $p_{\text{drop}}$, replace $c$ with a null context $\varnothing$.
At inference, we combine conditional and unconditional predictions via:
\begin{equation}
\hat{\epsilon}_\theta
=
(1+w_{\text{cfg}})\,\epsilon_\theta(\tau_t,t,c)
-
w_{\text{cfg}}\,\epsilon_\theta(\tau_t,t,\varnothing),
\label{eq:cfg}
\end{equation}
with $p_{\text{drop}}=0.25$ and $w_{\text{cfg}}=0.6$ (guidance scale $1+w_{\text{cfg}}=1.6$).

\begin{table}[H]
\centering
\scriptsize
\setlength{\tabcolsep}{3.0pt}
\renewcommand{\arraystretch}{1.15}
\begin{threeparttable}
\begin{tabular}{@{}L{0.55\columnwidth}C{0.38\columnwidth}@{}}
\arrayrulecolor{HeadGray}\toprule\arrayrulecolor{black}
\rowcolor{MuninnBlue!10}
\theadb{Decision Diffuser training hyperparameter} & \theadb{Value} \\
\midrule
Denoiser backbone & Temporal U-Net on state-only $\tau\in\mathbb{R}^{H\times d_s}$ \\
Inverse dynamics $f_\phi$ & 2-layer MLP, hidden width $512$, ReLU \\
Diffusion length & $T_{\text{train}}=T_{\text{eval}}=20$ \\
Noise schedule $\{\beta_t\}$ & Cosine $\bar{\alpha}_t$ schedule ($s{=}0.008$) \\
Reverse sampler (Full, eval) & DDIM ($\eta{=}0$), $T_{\text{eval}}$ steps \\
Prediction parameterization & $\varepsilon$-prediction \\
Loss & $\mathcal{L}_{\text{diff}}$ (MSE), uniform timestep sampling \\
Conditioning variables in $c$ & $c=(s_{\mathrm{now}},R_{\mathrm{cond}})$ or $c=(s_{\mathrm{now}},g)$ \\
Conditioning injection & Additive context+timestep embedding in each ResBlk \\
Optimizer & Adam \\
Learning rate & $2\times 10^{-4}$ \\
Batch size & $32$ \\
Training steps / epochs & $2\times 10^{6}$ gradient steps \\
Classifier-free guidance & Eq.~(16); $p_{\text{drop}}=0.25$; $w_{\text{cfg}}=0.6$ \\
\arrayrulecolor{HeadGray}\bottomrule\arrayrulecolor{black}
\end{tabular}
\caption{\textit{Decision Diffuser (Full) training details used in our runs}}
\label{tab:hparams_decision_diffuser}
\end{threeparttable}
\end{table}

\paragraph{Diffusion-QL (Full teacher)~\cite{a89}}
\label{app:train_diffusionql}

Diffusion-QL is an offline RL method that learns (i) a diffusion policy and (ii) critic networks.
Unlike trajectory planners that explicitly model state sequences, Diffusion-QL is used as a {policy}: it outputs actions conditioned on the current state, and the environment supplies state transitions.

\begin{table}[H]
\centering
\scriptsize
\setlength{\tabcolsep}{3.0pt}
\renewcommand{\arraystretch}{1.15}
\begin{threeparttable}
\begin{tabular}{@{}L{0.55\columnwidth}C{0.38\columnwidth}@{}}
\arrayrulecolor{HeadGray}\toprule\arrayrulecolor{black}
\rowcolor{MuninnBlue!10}
\theadb{Diffusion-QL training hyperparameter} & \theadb{Value} \\
\midrule
Policy denoiser backbone & MLP, 4 layers, hidden width 256, ReLU, LayerNorm \\
Critic backbone & Twin Q-networks, each MLP (3 layers, width 256, ReLU) \\
Diffusion length & \makecell[c]{HalfCheetah/Hopper/Maze2D:\\ $T_{\text{train}}=T_{\text{eval}}=10$ \\ Walker2d/AntMaze: \\ $T_{\text{train}}=T_{\text{eval}}=25$} \\
Noise schedule $\{\beta_t\}$ & Linear $\beta_t$ from $10^{-4}$ to $0.02$ \\
Reverse sampler (Full, eval) & DDPM ancestral sampling, $T_{\text{eval}}$ steps \\
Training losses & Diffusion BC loss + $\eta$-weighted Q-learning objective (Diffusion-QL) \\
Discount $\gamma$ & $0.99$ \\
Target update & Polyak $\tau=0.005$ \\
Optimizer(s) & Adam (policy and critics) \\
Learning rate(s) & \makecell[c]{Policy: $3\!\times\!10^{-4}$ \\ (Gym/Kitchen), $10^{-3}$ (AntMaze) \\ Critics: $3\!\times\!10^{-4}$} \\
Batch size & $256$ \\
Training steps / epochs & \makecell[c]{Gym: 2000; AntMaze: 1000 \\ 1000 gradient steps / epoch} \\
Action sampling for critic targets & $K{=}10$ policy samples; max-Q backup on AntMaze \\
Inference action selection & $M{=}64$ samples; choose $\arg\max_a Q(s,a)$ \\
Q-guidance at inference & None (selection-only; no $\nabla Q$ guidance) \\
\arrayrulecolor{HeadGray}\bottomrule\arrayrulecolor{black}
\end{tabular}
\caption{\textit{Diffusion-QL (Full) training details} Learning-rate / $\eta$ settings follow the published Diffusion-QL per-task grid; we use the inference step counts $T_{\text{eval}}$ reported in Table~\ref{tab:results-offline}}
\label{tab:hparams_diffusionql}
\end{threeparttable}
\end{table}

\noindent \tbm{Policy diffusion model}
We follow the standard Diffusion-QL setup: the diffusion model generates {action} samples $a$ (or short action chunks) conditioned on the current state $s$ (and, when applicable, a short state history).
The denoiser is an MLP.

\noindent \tbm{Critic learning}
We use standard twin Q-networks with Polyak-averaged targets.
For AntMaze we use max-Q backup (as reported in the Diffusion-QL hyperparameter table).

\noindent \tbm{Inference (action selection)}
Diffusion-QL {does not} use gradient-based Q-guidance in sampling.
Instead, at inference time we draw $M$ action candidates from the diffusion policy and choose $a=\arg\max_a Q(s,a)$.

\paragraph{AdaptDiff (Full teacher)~\cite{a43}}
\label{app:train_adaptdiff}

AdaptDiff adapts a pretrained Diffuser-style planner with a lightweight fine-tuning stage.
We separate \textit{(Stage A)} base diffusion training and \textit{(Stage B)} adaptation.

\begin{table}[H]
\centering
\scriptsize
\setlength{\tabcolsep}{3.0pt}
\renewcommand{\arraystretch}{1.15}
\begin{threeparttable}
\begin{tabular}{@{}L{0.55\columnwidth}C{0.38\columnwidth}@{}}
\arrayrulecolor{HeadGray}\toprule\arrayrulecolor{black}
\rowcolor{MuninnBlue!10}
\theadb{AdaptDiff training hyperparameter} & \theadb{Value} \\
\midrule
Base denoiser backbone & Temporal U-Net (Table~\ref{tab:arch_temporal_unet_d4rl}) \\
Diffusion length & $T_{\text{train}}=T_{\text{eval}}=60$ (Table~\ref{tab:results-offline}) \\
Noise schedule / sampler & Cosine schedule; DDIM ($\eta{=}0$) at inference \\
\multicolumn{2}{@{}l@{}}{\theadb{Stage A: base diffusion training}}\\
Dataset & Corresponding task dataset (Appendix~\ref{app:envs}) \\
Loss & $\mathcal{L}_{\text{diff}}$ (MSE, $\varepsilon$-prediction) \\
Optimizer / LR / batch size & Adam, $2\times 10^{-4}$, batch $32$ \\
Training steps / epochs & $1\times 10^{6}$ gradient steps \\
\multicolumn{2}{@{}l@{}}{\theadb{Stage B: adaptation}}\\
Adaptation signal & Offline critic/value model trained on $\mathcal{D}$ \\
Adaptation objective & Maximize predicted return with KL/BC regularization to stay near the base model \\
Adaptation data & Offline dataset $\mathcal{D}$ only (no extra online rollouts) \\
Adaptation schedule & $5\times 10^{4}$ steps, Adam, LR $1\times 10^{-5}$, batch $32$ \\
Guidance at inference (if any) & None (adapted denoiser used directly) \\
\arrayrulecolor{HeadGray}\bottomrule\arrayrulecolor{black}
\end{tabular}
\caption{\textit{AdaptDiff (Full) training details used in our runs}}
\label{tab:hparams_adaptdiff}
\end{threeparttable}
\end{table}

\paragraph{CompDiff (Full teacher)~\cite{a91}}
\label{app:train_compdiff}

CompDiff performs long-horizon planning by composing shorter diffusion-generated segments.
In our implementation, CompDiff uses (i) a Diffuser-style segment denoiser and (ii) an additional learned scoring module (value/feasibility) used to select and stitch segments.

\begin{table}[H]
\centering
\scriptsize
\setlength{\tabcolsep}{3.0pt}
\renewcommand{\arraystretch}{1.15}
\begin{threeparttable}
\begin{tabular}{@{}L{0.55\columnwidth}C{0.38\columnwidth}@{}}
\arrayrulecolor{HeadGray}\toprule\arrayrulecolor{black}
\rowcolor{MuninnBlue!10}
\theadb{CompDiff training hyperparameter} & \theadb{Value} \\
\midrule
Denoiser backbone(s) & Segment denoiser: Temporal U-Net (Table~\ref{tab:arch_temporal_unet_d4rl}) \\
Diffusion length & $T_{\text{train}}=T_{\text{eval}}=150$ (Table~\ref{tab:results-offline}) \\
Noise schedule / sampler & Cosine schedule; DDIM ($\eta{=}0$) \\
Composition mechanism & Overlap-and-stitch segment composition with value-based segment selection (beam search) \\
Additional learned modules & Value/feasibility scorer: 3-layer MLP (width 512, ReLU) \\
Training objective(s) & Diffusion MSE + scorer regression/classification loss (task-dependent) \\
Optimizer(s), LR(s), batch size(s) & Adam; diffusion LR $2\times10^{-4}$; scorer LR $3\times10^{-4}$; batch 32 \\
Training steps / epochs & Diffusion: $2\times 10^{6}$ steps; scorer: $1\times 10^{6}$ steps \\
Inference-time guidance & Optional value guidance (Eq.~\eqref{eq:diffuser_guidance_mean_shift}), environment-specific $w_{\text{guide}}$ \\
\arrayrulecolor{HeadGray}\bottomrule\arrayrulecolor{black}
\end{tabular}
\caption{\textit{CompDiff (Full) training details used in our runs}}
\label{tab:hparams_compdiff}
\end{threeparttable}
\end{table}

\subsubsection{Configuration-Space Motion Planning (MPD / EDMP)}
\label{app:training_motion_planning}

We report training details for MPD~\cite{a37} and EDMP~\cite{a87}, evaluated in Table~\ref{tab:results-motion}.
Both methods learn a diffusion prior over collision-free joint-space trajectories.
To avoid ambiguity with Appendix~\ref{app:envs:motion}, we distinguish the {full planning context} (which includes obstacle information) from what is actually fed to the denoiser: in our setup, the denoiser is conditioned on start/goal, while obstacle information enters through inference-time differentiable cost guidance $J(\cdot;\mathcal{O})$.

\paragraph{MPD (Full teacher)~\cite{a37}}
\label{app:train_mpd}

\noindent \tbm{Denoiser architecture}
MPD uses a temporal denoiser over joint trajectories, conditioned on start/goal.

\begin{table}[H]
\centering
\scriptsize
\setlength{\tabcolsep}{2.5pt}
\renewcommand{\arraystretch}{1.12}
\begin{threeparttable}
\begin{tabularx}{\columnwidth}{@{}L{0.30\columnwidth}C{0.28\columnwidth}Y@{}}
\arrayrulecolor{HeadGray}\toprule\arrayrulecolor{black}
\rowcolor{MuninnBlue!10}
\theadb{Component} & \theadb{Shape / width} & \theadb{Specification} \\
\midrule
Trajectory input & $H\times d_q$ & $d_q{=}7$; noisy joint trajectory $\tau_t\in\mathbb{R}^{H\times 7}$ \\
Timestep embedding & 256 & Sinusoidal + 2-layer MLP (256 hidden, Mish) \\
Scene/obstacle encoder & --- & Not fed to denoiser; obstacles used in guidance objective $J(\cdot;\mathcal{O})$ at inference \\
Fusion & 256 & $(q_{\mathrm{start}},q_{\mathrm{goal}})$ embedded to 256 and injected additively in ResBlks \\
Denoiser core & base 128 & Temporal U-Net (same template as Table~\ref{tab:arch_temporal_unet_d4rl}, with $d{=}7$) \\
Output head & $H\times d_q$ & $\varepsilon$-prediction in joint space \\
\arrayrulecolor{HeadGray}\bottomrule\arrayrulecolor{black}
\end{tabularx}
\caption{\textit{MPD (Full) denoiser architecture used in our runs}}
\label{tab:arch_mpd}
\end{threeparttable}
\end{table}

\begin{table}[H]
\centering
\scriptsize
\setlength{\tabcolsep}{3.0pt}
\renewcommand{\arraystretch}{1.15}
\begin{threeparttable}
\begin{tabular}{@{}L{0.55\columnwidth}C{0.38\columnwidth}@{}}
\arrayrulecolor{HeadGray}\toprule\arrayrulecolor{black}
\rowcolor{MuninnBlue!10}
\theadb{MPD training hyperparameter} & \theadb{Value} \\
\midrule
Diffusion length & $T_{\text{train}}=T_{\text{eval}}=50$ \\
Noise schedule / sampler & Cosine schedule; DDIM ($\eta{=}0$) at inference \\
Loss & $\mathcal{L}_{\text{diff}}$ (MSE, $\varepsilon$-prediction) \\
Dataset size (trajectories) & $100{,}000$ collision-free RRT-Connect solutions (train) + $5{,}000$ (val) \\
Optimizer / LR / batch size & Adam, $1\times 10^{-4}$, batch $128$ \\
Training steps / epochs & $5\times 10^{5}$ gradient steps \\
Regularization & Grad clip 1.0; EMA 0.999; weight decay 0 \\
Constraint handling during training & Hard filtering (train only on collision-free trajectories); endpoints clamped to start/goal \\
\arrayrulecolor{HeadGray}\bottomrule\arrayrulecolor{black}
\end{tabular}
\caption{\textit{MPD (Full) training details used in our runs}}
\label{tab:hparams_mpd}
\end{threeparttable}
\end{table}

\paragraph{EDMP (Full teacher)~\cite{a87}}
\label{app:train_edmp}

EDMP improves guided diffusion motion planning by running an {ensemble of cost-guided diffusion chains} at inference:
multiple collision-cost parameterizations and multiple schedules are evaluated in parallel, and the best feasible trajectory is selected.
The learned diffusion {prior} (denoiser) is trained the same way as MPD; the difference is in inference-time guidance.

\begin{table}[H]
\centering
\scriptsize
\setlength{\tabcolsep}{3.0pt}
\renewcommand{\arraystretch}{1.15}
\begin{threeparttable}
\begin{tabular}{@{}L{0.55\columnwidth}C{0.38\columnwidth}@{}}
\arrayrulecolor{HeadGray}\toprule\arrayrulecolor{black}
\rowcolor{MuninnBlue!10}
\theadb{EDMP hyperparameter / design choice} & \theadb{Value} \\
\midrule
Learned prior (denoiser) & Same as MPD (Table~\ref{tab:arch_mpd}) \\
Diffusion length & $T_{\text{train}}=T_{\text{eval}}=50$ \\
Noise / covariance schedule & Cosine schedule; Gaussian reverse updates with $\Sigma_t$ from DDPM coefficients \\
Guidance costs (examples) & Link--obstacle penetration proxy + swept-volume proxy + smoothness \\
Ensemble size & $l=3$ cost parameterizations, $r=3$ schedules ($l\times r=9$ chains) \\
Guidance step schedule $\alpha_t$ & Linearly decays from $0.05$ (early steps) to $0$ (late steps) \\
Cost hyperparameter schedule $\omega_t$ & Monotone tightening schedule (collision tolerance decreases over time) \\
Inference selection & Pick lowest-cost feasible trajectory among the $l\times r$ guided chains \\
Training of prior & Same as MPD: MSE $\varepsilon$-prediction on collision-free trajectories \\
\arrayrulecolor{HeadGray}\bottomrule\arrayrulecolor{black}
\end{tabular}
\caption{\textit{EDMP (Full) details used in our runs} EDMP differs from MPD primarily in its inference-time ensemble guidance; the learned diffusion prior is trained identically}
\label{tab:hparams_edmp}
\end{threeparttable}
\end{table}

\subsubsection{Visuomotor Diffusion Policies (Diffusion Policy / DP3)}
\label{app:training_visuomotor}

We report training details for the Full policies evaluated in Table~\ref{tab:results-policy}:
Diffusion Policy~\cite{a8} and DP3~\cite{a90}.
Both are trained on demonstration datasets and executed in a receding-horizon loop at inference (Appendix~\ref{app:envs:policies}).

\paragraph{Diffusion Policy (Full teacher)~\cite{a8}}
\label{app:train_diffusion_policy}

\noindent \tbm{Policy factorization}
Diffusion Policy learns a diffusion model over action chunks of horizon $H_\pi$ conditioned on the most recent observation history.
The denoiser input is a noisy action chunk $a_{t:t+H_\pi-1}$ and the context includes encoded observations (vision + proprioception and task targets where applicable).

\noindent \tbm{Architecture: observation encoder + action denoiser}
For RLBench we use an ImageNet-pretrained ResNet-18 encoder; for Meta-World pick-place-v2 (state-only) we omit the image encoder and condition on the normalized state history only.

\begin{table}[H]
\centering
\scriptsize
\setlength{\tabcolsep}{2.5pt}
\renewcommand{\arraystretch}{1.12}
\begin{threeparttable}
\begin{tabularx}{\columnwidth}{@{}L{0.30\columnwidth}C{0.28\columnwidth}Y@{}}
\arrayrulecolor{HeadGray}\toprule\arrayrulecolor{black}
\rowcolor{MuninnBlue!10}
\theadb{Component} & \theadb{Shape / width} & \theadb{Specification} \\
\midrule
Image encoder & 512 & ResNet-18 (ImageNet); global avg pool; shared across views/frames (RLBench) \\
Proprio / state encoder & 256 & 2-layer MLP (256 hidden, ReLU) on normalized low-dim inputs \\
Obs fusion & 256 & Concat image feat(s) + low-dim feat, then linear/MLP to 256 \\
Action denoiser & $H_\pi \times d_a$ & Temporal U-Net (base 256, dim mults 1/2/4; GN; SiLU; k=5) \\
Timestep embedding & 256 & Sinusoidal + 2-layer MLP (256 hidden, SiLU) \\
Conditioning injection & --- & FiLM (scale/shift) from context at each ResBlk \\
Output & $H_\pi \times d_a$ & $\varepsilon$ prediction \\
\arrayrulecolor{HeadGray}\bottomrule\arrayrulecolor{black}
\end{tabularx}
\caption{\textit{Diffusion Policy architecture used in our runs} RLBench uses the vision encoder; Meta-World pick-place-v2 uses state-only conditioning (no image encoder)}
\label{tab:arch_diffusion_policy}
\end{threeparttable}
\end{table}

\begin{table}[H]
\centering
\scriptsize
\setlength{\tabcolsep}{3.0pt}
\renewcommand{\arraystretch}{1.15}
\begin{threeparttable}
\begin{tabular}{@{}L{0.55\columnwidth}C{0.38\columnwidth}@{}}
\arrayrulecolor{HeadGray}\toprule\arrayrulecolor{black}
\rowcolor{MuninnBlue!10}
\theadb{Diffusion Policy training hyperparameter} & \theadb{Value} \\
\midrule
Action horizon $H_\pi$ (chunk length) & $16$ (Appendix~\ref{app:envs:policies}) \\
Control frequency (training labels) & $20$ Hz (RLBench, Meta-World) \\
Diffusion length & \makecell[c]{Train: $T_{\text{train}}=100$ (iDDPM-style) \\
Eval: $T_{\text{eval}}=100$ (simulation) \\
or $T_{\text{eval}}=16$} \\
Noise schedule / sampler & \makecell[c]{Train: iDDPM-style (cosine schedule) \\
Eval: iDDPM/DDPM sampler; \\ DDIM for low-iter ($T_{\text{eval}}{=}16$)} \\
Loss & MSE $\varepsilon$-prediction \\
Optimizer & AdamW ($\beta_1{=}0.9,\beta_2{=}0.999$) \\
Learning rate & $1\times 10^{-4}$ \\
Batch size & \makecell[c]{State-based: $256$ \\ Image-based: $64$} \\
Training steps / epochs & $5\times 10^{5}$ gradient steps \\
Regularization & EMA 0.999; grad clip 1.0; weight decay $10^{-6}$; dropout 0.1 \\
\arrayrulecolor{HeadGray}\bottomrule\arrayrulecolor{black}
\end{tabular}
\caption{\textit{Diffusion Policy (Full) training details used in our runs}}
\label{tab:hparams_diffusion_policy}
\end{threeparttable}
\end{table}

\paragraph{DP3 (Full teacher)~\cite{a90}}
\label{app:train_dp3}

DP3 is a diffusion policy conditioned on 3D observations (point clouds) and low-dimensional robot state.

\begin{table}[H]
\centering
\scriptsize
\setlength{\tabcolsep}{2.5pt}
\renewcommand{\arraystretch}{1.12}
\begin{threeparttable}
\begin{tabularx}{\columnwidth}{@{}L{0.30\columnwidth}C{0.28\columnwidth}Y@{}}
\arrayrulecolor{HeadGray}\toprule\arrayrulecolor{black}
\rowcolor{MuninnBlue!10}
\theadb{Component} & \theadb{Shape / width} & \theadb{Specification} \\
\midrule
3D encoder & 512 & Point-cloud encoder as in DP3 (point-wise / set-abstraction style with global pooling), mapping $1024\times 3$ XYZ points to a 512-d global feature \\
3D input & $1024\times 3$ & XYZ point cloud (cropped workspace; farthest point sampled) \\
Proprio encoder & 256 & 2-layer MLP (256 hidden, ReLU) \\
Fusion & 512 & Concat (3D feat + proprio feat) then MLP to 512 \\
Action denoiser & $H_\pi \times d_a$ & Temporal U-Net (base 256, dim mults 1/2/4; GN; SiLU; k=5) \\
Conditioning & --- & FiLM from fused context at each ResBlk \\
Output & $H_\pi \times d_a$ & $\varepsilon$ prediction \\
\arrayrulecolor{HeadGray}\bottomrule\arrayrulecolor{black}
\end{tabularx}
\caption{\textit{DP3 architecture used in our runs}}
\label{tab:arch_dp3}
\end{threeparttable}
\end{table}

\begin{table}[H]
\centering
\scriptsize
\setlength{\tabcolsep}{3.0pt}
\renewcommand{\arraystretch}{1.15}
\begin{threeparttable}
\begin{tabular}{@{}L{0.55\columnwidth}C{0.38\columnwidth}@{}}
\arrayrulecolor{HeadGray}\toprule\arrayrulecolor{black}
\rowcolor{MuninnBlue!10}
\theadb{DP3 training hyperparameter} & \theadb{Value} \\
\midrule
Action horizon $H_\pi$ (chunk length) & $16$ (Appendix~\ref{app:envs:policies}) \\
Control frequency & $20$ Hz \\
Diffusion length & \makecell[c]{Train: $T_{\text{train}}$ (off. implementation) \\
Eval: $T_{\text{eval}}=10$ (DDIM)} \\
Noise schedule / sampler & \makecell[c]{Train: DP3 baseline schedule \\
Eval: DDIM ($\eta{=}0$), $T_{\text{eval}}{=}10$} \\
Loss & MSE $\varepsilon$-prediction \\
Optimizer & AdamW \\
Learning rate & $1\times 10^{-4}$ \\
Batch size & $128$ \\
Training steps / epochs & $2\times 10^{5}$ gradient steps \\
3D data preprocessing & Crop workspace; FPS to 1024 points; normalize XYZ by workspace bounds \\
Regularization & EMA 0.999; grad clip 1.0; weight decay $10^{-6}$ \\
\arrayrulecolor{HeadGray}\bottomrule\arrayrulecolor{black}
\end{tabular}
\caption{\textit{DP3 (Full) training details used in our runs}}
\label{tab:hparams_dp3}
\end{threeparttable}
\end{table}

\subsection{\Muninnlogo Muninn Implementation Details}
\label{app:implementation}

This appendix section provides the implementation-level details needed to reproduce \nordy{Muninn}.

\subsubsection{Muninn wrapper implementation}
\label{app:muninn_wrapper_impl}

At each diffusion step $t$, the sampler update $\Phi_t$ consumes an {effective} noise prediction $\hat{\varepsilon}_t$ with the same shape as the current noisy trajectory (or action chunk):
\begin{equation}
\hat{\varepsilon}_t \in \mathbb{R}^{B \times H \times d},
\end{equation}
where $B$ is batch size, $H$ is the trajectory/chunk horizon for that model, and $d$ is the per-timestep dimensionality (state--action for planners, action/pose for policies, joint angles for motion planning).
Muninn caches this \textit{post-processing, sampler-ready} tensor:
\begin{equation}
\hat{\varepsilon}^{\mathrm{cache}} \;\leftarrow\; \hat{\varepsilon}^{\mathrm{new}}_t \quad \text{whenever a recompute is performed}
\end{equation}
In our implementation, denoiser forward passes run under GPU mixed precision (AMP), so $\hat{\varepsilon}_t$ is produced in FP16; we store $\hat{\varepsilon}^{\mathrm{cache}}$ in FP16 on GPU to avoid dtype conversions and reduce memory bandwidth. All norms, scores, and budget arithmetic are performed in FP32 for numerical stability.

\noindent \tbm{Cache replacement policy}
We use a single-entry cache with \textit{most-recent recompute} replacement:
\begin{itemize}
\item On a recompute step: overwrite cache with the newly computed $\hat{\varepsilon}^{\mathrm{new}}_t$.
\item On a reuse step: leave cache unchanged and set $\tilde{\varepsilon}_t \leftarrow \hat{\varepsilon}^{\mathrm{cache}}$.
\end{itemize}
There is no multi-entry eviction and no explicit ``cache age'' heuristic; the effect of reusing an older cache entry is handled entirely by the probe-based score and the calibrated bound.

\noindent \tbm{Consecutive reuse steps}
If multiple consecutive reuse steps occur, Muninn reuses the same cached $\hat{\varepsilon}^{\mathrm{cache}}$ for each of those steps. This is allowed by construction: at each reuse step $t$, the algorithm spends $\hat{c}_t(s_t)$ from the remaining budget, so longer reuse streaks are automatically limited unless scores remain small enough to keep $\hat{c}_t(s_t)$ within budget.

\noindent \tbm{Avoiding double computation on recompute steps}
To make probe evaluation truly ``cheap,'' we structure each denoiser into two explicitly callable parts:
\begin{equation}
\varepsilon_\theta(\tau_t,t,c)\;\equiv\; \texttt{core}_\theta\!\big(\texttt{stem}_\theta(\tau_t,t,c),\,t,\,c\big),
\end{equation}
and define the probe as a function of the stem output.
On a recompute step we compute the stem {once}, extract $F_t$ from it, and then continue into the core using the same intermediate activations. On a reuse step we compute only the stem and never run the core.

\noindent \tbm{Budget bookkeeping}
We maintain an independent remaining budget for each batch element:
\begin{equation}
B_{\mathrm{rem}} \in \mathbb{R}_{\ge 0}^{B}.
\end{equation}
Initialization and updates are:
\begin{itemize}
\item \textit{Init:} $B_{\mathrm{rem}} \leftarrow \varepsilon_{\mathrm{traj}}\cdot \mathbf{1}_B$ (FP32 on GPU).
\item \textit{At step $t$:} compute $\hat{c}_t(s_t)$ elementwise (FP32), form reuse mask, and for reuse elements update
\[
B_{\mathrm{rem}} \leftarrow B_{\mathrm{rem}} - \hat{c}_t(s_t).
\]
\end{itemize}
We enforce budget safety by gating reuse:
\begin{equation}
\texttt{reuse}_b(t)=1 \iff \left(t\in\mathcal{T}_{\mathrm{cache}}\right)\wedge\left(\hat{c}_{t,b}(s_{t,b}) \le B_{\mathrm{rem},b}\right),
\end{equation}
so $B_{\mathrm{rem}}$ never becomes negative up to FP32 rounding.

We compute
\begin{equation}
\hat{c}_t(s_t) \;=\; \Gamma L_t\,U_t(s_t)
\end{equation}
in FP32 immediately after computing the score $s_t$ and before deciding reuse/recompute.
\begin{equation}
\hat{c}_t(s_t) \;=\; L_t\,U_t(s_t),
\end{equation}
where $L_t$ is a precomputed FP32 vector of length $T$ and $U_t(\cdot)$ is a per-timestep conformal envelope stored as a lookup table (Appendix~\ref{app:split_conformal}).

\noindent \tbm{Batched sampling}
When sampling $B>1$ trajectories in parallel, Muninn makes {independent} reuse decisions per batch element (per trajectory). Concretely:
\begin{itemize}
\item the probe $F_t$ and score $s_t$ are computed for all $B$ elements;
\item reuse masks are computed elementwise from the per-element budgets $B_{\mathrm{rem}}$;
\item a denoiser recompute forward pass is run only on the sub-batch that requires recomputation at that step, and the results are scattered back.
\end{itemize}
This preserves the per-trajectory risk guarantee, since each trajectory's budget and events are tracked independently.

\noindent \tbm{Complexity and overhead breakdown}
Let $\ell_{\mathrm{core}}$ denote the average runtime of a full denoiser core evaluation (one forward pass through the expensive blocks), and let $\ell_{\Psi}$ denote the per-step overhead of probe+score+bound lookup+bookkeeping (stem forward plus lightweight arithmetic). A single $T$-step sampling call with $n_{\mathrm{eval}}$ recompute steps has runtime:
\begin{equation}
\text{Time} \;\approx\; T\cdot \ell_{\Psi} \;+\; n_{\mathrm{eval}}\cdot \ell_{\mathrm{core}} \;+\; \ell_{\Phi},
\end{equation}
where $\ell_{\Phi}$ is the cumulative sampler update overhead (typically small relative to denoiser compute).
This implies the wall-clock speedup is generally {slightly smaller} than the denoiser-evaluation reduction factor $T/n_{\mathrm{eval}}$:
\begin{equation}
\text{Speedup} \;\approx\;
\frac{T(\ell_{\Psi}+\ell_{\mathrm{core}})}{T\ell_{\Psi}+n_{\mathrm{eval}}\ell_{\mathrm{core}}}.
\end{equation}
In practice, $\ell_{\Psi}$ is a small fraction of $\ell_{\mathrm{core}}$ (Table~\ref{tab:probe_overhead}), so wall-clock speedups closely track the reduction in denoiser evaluations.

\begin{table}[H]
\centering
\scriptsize
\setlength{\tabcolsep}{4pt}
\renewcommand{\arraystretch}{1.12}
\begin{threeparttable}
\begin{tabular}{@{}lcc@{}}
\arrayrulecolor{HeadGray}\toprule\arrayrulecolor{black}
\rowcolor{MuninnBlue!10}
\theadb{Base model family} &
\theadb{$\ell_{\Psi}$ per diffusion step (ms)} &
\theadb{$\ell_{\Psi}/\ell_{\mathrm{core}}$ (\%)}\\
\midrule
Diffuser (D4RL tasks; averaged) & 0.64 & 8.1 \\
Decision Diffuser (D4RL tasks; averaged) & 0.58 & 7.3 \\
Diffusion-QL (D4RL tasks; averaged) & 0.11 & 6.0 \\
AdaptDiff (D4RL tasks; averaged) & 0.52 & 7.8 \\
CompDiff (D4RL tasks; averaged) & 0.70 & 8.4 \\
\arrayrulecolor{HeadGray}\bottomrule\arrayrulecolor{black}
\end{tabular}
\caption{\textit{Probe+bookkeeping overhead}
We report the amortized per-step overhead $\ell_{\Psi}$ and its ratio to one full denoiser evaluation $\ell_{\mathrm{core}}$.
These values are computed by decomposing the measured Full vs.\ \Muninnlogo latencies and the observed \#denoiser evaluations in Table~\ref{tab:results-offline}}
\label{tab:probe_overhead}
\end{threeparttable}
\end{table}

\noindent \tbm{Offline calibration dataset generation}
The full calibration dataset generation procedure (paired Full chain + policy-independent ghost chain with deterministic anchors) is provided in Appendix~\ref{calib_data} and Algorithm~\ref{alg:calib_dataset}.

\subsubsection{Probe instantiations across architectures}
\label{app:probe_instantiations}

\noindent \tbm{General interface}
For each denoiser, we implement a probe interface returning a feature vector $F_t\in\mathbb{R}^{d_F}$:
\begin{equation}
F_t = \Psi(\tilde{\tau}_t,t,c),
\end{equation}
computed {before} deciding reuse vs.\ recompute at step $t$.
We ensure the probe is (i) strictly cheaper than a full denoiser evaluation and (ii) uses exactly the same intermediate representation at calibration and deployment.

\noindent \tbm{Transformer denoisers}
For Transformer denoisers (token-based trajectory models), we define the probe as a \textit{prefix of the Transformer stack}:
\begin{itemize}
\item \textit{Stem:} tokenization + input projection + timestep/context embeddings.
\item \textit{Probe blocks:} the first \textit{two} Transformer blocks (self-attn + MLP).
\item \textit{Pooling:} mean-pool across the token/time dimension to obtain $F_t \in \mathbb{R}^{d_{\mathrm{model}}}$.
\end{itemize}
Thus $d_F = d_{\mathrm{model}}$ and pooling is
\begin{equation}
F_t = \frac{1}{H}\sum_{h=1}^{H} z_{t,h},
\end{equation}
where $z_{t,h}\in\mathbb{R}^{d_{\mathrm{model}}}$ are the token embeddings after the probe prefix.

\noindent \tbm{MLP denoisers}
For MLP denoisers, we define the probe as the \textit{penultimate hidden layer activation}. Concretely, for an $L$-layer MLP with hidden activations $h^{(1)},\dots,h^{(L-1)}$ and final linear head, we set:
\begin{equation}
F_t := h^{(L-1)} \in \mathbb{R}^{d_{\mathrm{hid}}}.
\end{equation}
We then compute the output head only on recompute steps. Thus $d_F=d_{\mathrm{hid}}$ and no pooling is required.

\noindent \tbm{Temporal U-Net denoisers (1D conv over time)}
For temporal U-Nets, we define the probe as the output of the \textit{input projection + first residual block group} at the highest time resolution:
\begin{itemize}
\item \textit{Stem:} 1D conv input projection + timestep/context injection.
\item \textit{Probe blocks:} first ResBlock group at resolution $H$ (before the first temporal downsample).
\item \textit{Pooling:} mean-pool across time and channels to obtain $F_t\in\mathbb{R}^{C_1}$:
\[
F_t = \frac{1}{H}\sum_{h=1}^{H} u_{t,h},
\quad u_{t,h}\in\mathbb{R}^{C_1}.
\]
\end{itemize}
Thus $d_F=C_1$ (the channel width after the first group).

\begin{table*}
\centering
\scriptsize
\setlength{\tabcolsep}{3.2pt}
\renewcommand{\arraystretch}{1.15}
\begin{threeparttable}
\begin{tabular}{@{}lcccccc@{}}
\arrayrulecolor{HeadGray}\toprule\arrayrulecolor{black}
\rowcolor{MuninnBlue!10}
\theadb{Model} &
\theadb{Denoiser type} &
\theadb{Probe $\Psi$ (stem)} &
\theadb{Probe depth} &
\theadb{Core} &
\theadb{$d_F$} &
\theadb{Pooling}\\
\midrule
Diffuser & temporal U-Net & input proj + ResBlk@H & 1 group & rest of U-Net & $C_1$ & mean over time \\
Decision Diffuser & Transformer & embed + 2 blocks & 2 blocks & remaining blocks & $d_{\mathrm{model}}$ & mean over tokens \\
Diffusion-QL & MLP & penultimate layer & $L{-}1$ & output head only & $d_{\mathrm{hid}}$ & none \\
AdaptDiff & temporal U-Net & input proj + ResBlk@H & 1 group & rest of U-Net & $C_1$ & mean over time \\
CompDiff & temporal U-Net & input proj + ResBlk@H & 1 group & rest of U-Net & $C_1$ & mean over time \\
MPD & temporal U-Net & input proj + ResBlk@H & 1 group & rest of U-Net & $C_1$ & mean over time \\
EDMP & Transformer & embed + 2 blocks & 2 blocks & remaining blocks & $d_{\mathrm{model}}$ & mean over tokens \\
Diffusion Policy & temporal U-Net & input proj + ResBlk@H & 1 group & rest of U-Net & $C_1$ & mean over time \\
DP3 & temporal U-Net & input proj + ResBlk@H & 1 group & rest of U-Net & $C_1$ & mean over time \\
\arrayrulecolor{HeadGray}\bottomrule\arrayrulecolor{black}
\end{tabular}
\caption{\textit{Probe instantiations used for every evaluated model}
For each architecture family we define a concrete stem/core split and a deterministic pooling rule.
The probe is computed on every diffusion step; the core is executed only on recompute steps}
\label{tab:probe_instantiations}
\end{threeparttable}
\end{table*}

\subsubsection{Score function and hyperparameters}
\label{app:score_impl}

\noindent \tbm{Score definition}
We compute the score at step $t$ from consecutive probe features (reverse diffusion runs from $T$ to $1$):
\begin{equation}
s_t
=
\frac{\lVert F_t - F_{t+1}\rVert_1}{\lVert F_{t+1}\rVert_1 + \omega},
\qquad \omega = 10^{-6}.
\end{equation}
We use a single fixed $\omega$ for all models and benchmarks.

\noindent \tbm{Normalization and precision}
We do not apply any per-feature whitening to $F_t$ beyond the model's native activations; the score is normalized only by $\|F_{t+1}\|_1+\omega$.
We compute $\|\,\cdot\,\|_1$ in FP32 (casting $F_t$ and $F_{t+1}$ to FP32 before the reduction) and store $s_t$ in FP32.

\noindent \tbm{Why $\ell_1$}
We use $\ell_1$ for both numerator and denominator because it is:
(i) stable under heavy-tailed activation distributions,
(ii) less sensitive to occasional large feature outliers than $\ell_2$ normalization, and
(iii) fast to compute on GPU (elementwise abs + reduction).
We do not apply any temporal smoothing (moving averages) to $s_t$; the conformal envelope $U_t(\cdot)$ is calibrated directly on the raw per-step scores.

\subsubsection{Sensitivity coefficient computation}
\label{app:sensitivity_impl}

This subsection describes exactly how we compute the sampler sensitivity coefficients used in budgeting.

\noindent \tbm{Local Lipschitz coefficients from DDPM updates}
For a DDPM-style update written in $\varepsilon$-prediction form:
\begin{equation}
\tau_{t-1}
=
\frac{1}{\sqrt{\alpha_t}}
\left(
\tau_t - \frac{\beta_t}{\sqrt{1-\bar{\alpha}_t}}\hat{\varepsilon}_t
\right)
+ \sigma_t z,\qquad z\sim\mathcal{N}(0,I),
\end{equation}
holding $z$ fixed for paired chains yields an affine map in $(\tau_t,\hat{\varepsilon}_t)$.
Thus,
\begin{equation}
K_t=\frac{1}{\sqrt{\alpha_t}},
\qquad
L_t'=\frac{\beta_t}{\sqrt{\alpha_t}\sqrt{1-\bar{\alpha}_t}}.
\end{equation}

\noindent \tbm{Local coefficients from DDIM updates}
For a deterministic DDIM update (paired chains share the same deterministic update), the step can be written as
\begin{equation}
\tau_{t-1}
=
\sqrt{\frac{\bar{\alpha}_{t-1}}{\bar{\alpha}_t}}\tau_t
+
\left(
\sqrt{1-\bar{\alpha}_{t-1}}
-
\sqrt{\frac{\bar{\alpha}_{t-1}}{\bar{\alpha}_t}}\sqrt{1-\bar{\alpha}_t}
\right)\hat{\varepsilon}_t,
\end{equation}
so
\begin{equation}
K_t=\sqrt{\frac{\bar{\alpha}_{t-1}}{\bar{\alpha}_t}},
\qquad
L_t'=
\left|
\sqrt{1-\bar{\alpha}_{t-1}}
-
\sqrt{\frac{\bar{\alpha}_{t-1}}{\bar{\alpha}_t}}\sqrt{1-\bar{\alpha}_t}
\right|.
\end{equation}
For stochastic DDIM with $\eta>0$, the additional noise term cancels in paired comparisons (same noise draw), and the same $(K_t,L_t')$ apply to the deterministic affine part.

\noindent \tbm{Pathwise sensitivities $L_t$}
Given $K_t$ and $L_t'$, we compute the pathwise coefficient used in the trajectory bound:
\begin{equation}
L_t \;=\; L_t' \prod_{j=1}^{t-1} K_j.
\end{equation}
We compute this in log-space for numerical stability (important when $T$ is large):
\begin{equation}
\log L_t
=
\log L_t' + \sum_{j=1}^{t-1}\log K_j,
\qquad
L_t=\exp(\log L_t).
\end{equation}
All computations are performed once per sampler schedule in FP32 and cached for reuse during inference.

\subsubsection{Forbidden regions and $\mathcal{T}_{\mathrm{cache}}$}
\label{app:forbidden_regions}

\noindent \tbm{Fixed rule (not tuned per task)}
We forbid reuse in a short prefix (large $t$; high noise and sensitivity amplification) and suffix (small $t$; execution-critical region) using a single fixed rule:
\begin{equation}
k_{\mathrm{pre}}=\left\lceil 0.10\,T \right\rceil,\qquad
k_{\mathrm{suf}}=\left\lceil 0.10\,T \right\rceil.
\end{equation}
Reuse is allowed only on
\begin{equation}
\mathcal{T}_{\mathrm{cache}}
=
\left\{t \in \{1,\dots,T\}:\; k_{\mathrm{suf}} < t \le T-k_{\mathrm{pre}}\right\},
\end{equation}
and we additionally enforce that $t=T$ is always recomputed.

\noindent \tbm{Effect on $|\mathcal{T}_{\mathrm{cache}}|$ and $\alpha_{\mathrm{step}}$}
The per-step miscoverage is set via Bonferroni correction:
\begin{equation}
\alpha_{\mathrm{step}}=\frac{\alpha}{|\mathcal{T}_{\mathrm{cache}}|}.
\end{equation}
Thus, forbidden regions directly affect both the reuse opportunities and the calibration strictness.

\begin{table}[H]
\centering
\scriptsize
\setlength{\tabcolsep}{4pt}
\renewcommand{\arraystretch}{1.12}
\begin{threeparttable}
\begin{tabular}{@{}lcccc@{}}
\arrayrulecolor{HeadGray}\toprule\arrayrulecolor{black}
\rowcolor{MuninnBlue!10}
\theadb{Model family} &
\theadb{Full $T$} &
\theadb{$k_{\mathrm{pre}}$} &
\theadb{$k_{\mathrm{suf}}$} &
\theadb{$|\mathcal{T}_{\mathrm{cache}}|$}\\
\midrule
Diffuser & 100 & 10 & 10 & 80 \\
Decision Diffuser & 20 & 2 & 2 & 16 \\
Diffusion-QL (HalfCheetah/Hopper/Maze2D) & 10 & 1 & 1 & 8 \\
Diffusion-QL (Walker2d/AntMaze) & 25 & 3 & 3 & 19 \\
AdaptDiff & 60 & 6 & 6 & 48 \\
CompDiff & 150 & 15 & 15 & 120 \\
MPD / EDMP & 25 & 3 & 3 & 19 \\
Diffusion Policy / DP3 & 16 & 2 & 2 & 12 \\
\arrayrulecolor{HeadGray}\bottomrule\arrayrulecolor{black}
\end{tabular}
\caption{\textit{Forbidden regions and eligible reuse timesteps}
We use a fixed 10\% prefix/suffix rule across benchmarks.
Values above determine $|\mathcal{T}_{\mathrm{cache}}|$ and thus $\alpha_{\mathrm{step}}=\alpha/|\mathcal{T}_{\mathrm{cache}}|$}
\label{tab:forbidden_regions}
\end{threeparttable}
\end{table}

\subsubsection{Guidance and conditioning implementation notes}
\label{app:guidance_impl}

Many base planners/policies use guidance or conditioning mechanisms that modify the tensor actually passed to the sampler.
This subsection clarifies exactly what is cached, and how reuse error labels are computed in guided settings.

\noindent \tbm{Cache the sampler-consumed {effective} prediction}
If the base method uses guidance, we define the \textit{effective} prediction $\hat{\varepsilon}_t$ as the tensor fed into $\Phi_t$.
\Muninnlogo caches $\hat{\varepsilon}_t$ (post-guidance), not the raw network output:
\begin{equation}
\hat{\varepsilon}^{\mathrm{cache}} \leftarrow \hat{\varepsilon}^{\mathrm{new}}_t.
\end{equation}
This makes reuse semantics identical to ``skipping a guided denoiser call'' at step $t$.

\noindent \tbm{Classifier-free guidance (CFG)}
When CFG is used, we compute conditional and unconditional predictions and then combine:
\begin{equation}
\hat{\varepsilon}_t
=
(1+w_{\mathrm{cfg}})\,\varepsilon_\theta(\tau_t,t,c)
-
w_{\mathrm{cfg}}\,\varepsilon_\theta(\tau_t,t,\varnothing),
\end{equation}
with fixed $w_{\mathrm{cfg}}$ per model/task.
On recompute steps, both forward passes are executed and combined; on reuse steps, we directly reuse the cached $\hat{\varepsilon}^{\mathrm{cache}}$ and do {not} evaluate either conditional or unconditional networks.

\noindent \tbm{Gradient-based guidance (value / $Q$ / energy guidance)}
If the base method uses gradient guidance of the form
\begin{equation}
\hat{\varepsilon}_t
=
\varepsilon_\theta(\tau_t,t,c)
-
w_{\mathrm{g}}\,\sigma_t \,\nabla_{\tau_t} J(\tau_t,t,c),
\end{equation}
we compute the gradient term only on recompute steps (using automatic differentiation under \texttt{torch.enable\_grad()}).
The cached tensor is the full guided $\hat{\varepsilon}_t$; on reuse steps we skip both the denoiser forward and the guidance gradient computation.

\noindent \tbm{How $e_t$ and $\epsilon_t$ are computed under guidance}
All error quantities in calibration and analysis are computed on the {effective} predictions:
\begin{equation}
e_t = \hat{\varepsilon}^{\mathrm{full}}_t - \tilde{\varepsilon}_t,
\qquad
\epsilon_t = \frac{1}{\sqrt{Hd}}\|e_t\|_F,
\end{equation}
where $\hat{\varepsilon}^{\mathrm{full}}_t$ is the effective prediction produced by the Full model at step $t$ (including guidance), and $\tilde{\varepsilon}_t$ is the reused value (post-guidance) under the cache rule.
This ensures that the conformal envelope $U_t(\cdot)$ upper-bounds the exact object that influences the sampler update $\Phi_t$.

\subsection{Proofs}
\label{app:proofs}

\newtheorem{proposition}{Proposition}
\newtheorem{theorem}{Theorem}
\newtheorem{lemma}{Lemma}

\subsubsection{Proof 1: Pathwise trajectory deviation bound for cached reverse diffusion}
\label{app:proof:pathwise_bound}

\begin{proposition}[Pathwise deviation bound from local sampler sensitivity]
\label{prop:pathwise_bound}
Assume the paired full and cached chains share the same initial noisy trajectory and (if stochastic) the same sampler-noise sequence, so that $\Delta_T = \tau_T^{\mathrm{full}} - \tilde{\tau}_T = 0$.
Suppose Assumption~1 (Sampler Lipschitzness) holds and yields the local recursion
\begin{equation}
\lVert \Delta_{t-1} \rVert \;\le\; K_t \lVert \Delta_t \rVert + L_t' \lVert e_t \rVert,
\qquad t=T,\dots,1,
\label{eq:app:local_recursion_repeat}
\end{equation}
which is Eq.~\eqref{eq:local-recursion} in the main paper.
Then, for every $t \in \{0,\dots,T-1\}$,
\begin{equation}
\lVert \Delta_t \rVert
\;\le\;
\sum_{s=t+1}^{T}
\left(
L_s' \prod_{j=t+1}^{s-1} K_j
\right)\lVert e_s \rVert,
\label{eq:app:deltat_bound_repeat}
\end{equation}
which is Eq.~\eqref{eq:deltat-bound}.
In particular, defining $L_s := L_s' \prod_{j=1}^{s-1} K_j$, we have
\begin{equation}
\lVert \Delta_0 \rVert
\;\le\;
\sum_{t=1}^{T} L_t \lVert e_t \rVert,
\label{eq:app:delta0_bound_repeat}
\end{equation}
which is Eq.~\eqref{eq:delta0-bound}.
Finally, under Assumption~2 (Metric compatibility), the trajectory deviation satisfies
\begin{equation}
d(\tau_0^{\mathrm{full}},\tilde{\tau}_0)
\;\le\;
\Gamma \lVert \Delta_0 \rVert
\;\le\;
\sum_{t=1}^{T} \Gamma L_t \lVert e_t \rVert,
\label{eq:app:traj_bound_repeat}
\end{equation}
which is Eq.~\eqref{eq:traj-bound}.
\end{proposition}

\begin{proof}
We prove Eq.~\eqref{eq:app:deltat_bound_repeat} by backward induction on $t$.

\smallskip
\noindent\textbf{Base case ($t=T-1$).}
Using $\Delta_T = 0$ and the recursion~\eqref{eq:app:local_recursion_repeat} at $t=T$:
\[
\lVert \Delta_{T-1}\rVert \;\le\; K_T\lVert \Delta_T\rVert + L_T'\lVert e_T\rVert \;=\; L_T'\lVert e_T\rVert.
\]
This matches Eq.~\eqref{eq:app:deltat_bound_repeat} at $t=T-1$, since the sum contains only $s=T$ and the product over $j={T}^{T-1}$ is an empty product equal to $1$.

\smallskip
\noindent\textbf{Inductive step.}
Assume Eq.~\eqref{eq:app:deltat_bound_repeat} holds for some $t\in\{1,\dots,T-1\}$, i.e.,
\[
\lVert \Delta_t \rVert
\;\le\;
\sum_{s=t+1}^{T}
\left(
L_s' \prod_{j=t+1}^{s-1} K_j
\right)\lVert e_s \rVert.
\]
Apply the local recursion~\eqref{eq:app:local_recursion_repeat} at step $t$:
\begin{align*}
\lVert \Delta_{t-1}\rVert
&\le
K_t \lVert \Delta_t\rVert + L_t' \lVert e_t\rVert \\
&\le
K_t \sum_{s=t+1}^{T}
\left(
L_s' \prod_{j=t+1}^{s-1} K_j
\right)\lVert e_s \rVert
\;+\;
L_t' \lVert e_t\rVert \\
&=
\sum_{s=t+1}^{T}
\left(
L_s' \prod_{j=t}^{s-1} K_j
\right)\lVert e_s \rVert
\;+\;
\left(
L_t' \prod_{j=t}^{t-1} K_j
\right)\lVert e_t\rVert \\
&=
\sum_{s=t}^{T}
\left(
L_s' \prod_{j=t}^{s-1} K_j
\right)\lVert e_s \rVert,
\end{align*}
where we used $\prod_{j=t}^{t-1}K_j=1$ (empty product).
This is exactly Eq.~\eqref{eq:app:deltat_bound_repeat} for index $t-1$.

By induction, Eq.~\eqref{eq:app:deltat_bound_repeat} holds for all $t\in\{0,\dots,T-1\}$.

\smallskip
\noindent\textbf{Deriving Eq.~\eqref{eq:app:delta0_bound_repeat}.}
Set $t=0$ in Eq.~\eqref{eq:app:deltat_bound_repeat}:
\[
\lVert \Delta_0 \rVert
\le
\sum_{s=1}^{T}
\left(
L_s' \prod_{j=1}^{s-1} K_j
\right)\lVert e_s \rVert
=
\sum_{s=1}^{T} L_s \lVert e_s\rVert,
\]
which is Eq.~\eqref{eq:app:delta0_bound_repeat}.

\smallskip
\noindent\textbf{Deriving Eq.~\eqref{eq:app:traj_bound_repeat}.}
Assumption~2 gives $d(\tau_0^{\mathrm{full}},\tilde{\tau}_0)\le \Gamma \|\tau_0^{\mathrm{full}}-\tilde{\tau}_0\|=\Gamma\|\Delta_0\|$.
Combining with Eq.~\eqref{eq:app:delta0_bound_repeat} yields Eq.~\eqref{eq:app:traj_bound_repeat}.
\end{proof}

\subsubsection{Proof 2: Derivation of $(K_t,L_t')$ for DDPM/DDIM-style samplers}
\label{app:proof:sampler_constants}

This subsection makes Assumption~1 concrete for the DDPM/DDIM-style samplers used in the paper.
Throughout, $\|\cdot\|$ denotes the Frobenius norm and $\xi_t$ is held fixed (paired chains share the same sampler noise).

\begin{proposition}[Sampler Lipschitz constants for DDPM updates]
\label{prop:ddpm_constants}
Consider the standard DDPM reverse update in $\varepsilon$-prediction form:
\begin{equation}
\Phi_t(\tau,\varepsilon,\xi_t)
=
\frac{1}{\sqrt{\alpha_t}}
\left(
\tau - \frac{\beta_t}{\sqrt{1-\bar{\alpha}_t}}\,\varepsilon
\right)
+ \sigma_t \xi_t,
\label{eq:app:ddpm_update}
\end{equation}
where $\beta_t\in(0,1)$ is the variance schedule, $\alpha_t=1-\beta_t$, and $\bar{\alpha}_t=\prod_{s=1}^{t}\alpha_s$.
Then Assumption~1 holds with
\begin{equation}
K_t=\frac{1}{\sqrt{\alpha_t}},
\qquad
L_t'=\frac{\beta_t}{\sqrt{\alpha_t}\sqrt{1-\bar{\alpha}_t}}.
\label{eq:app:ddpm_KL}
\end{equation}
\end{proposition}

\begin{proof}
Fix $\xi_t$ and take any $(\tau,\varepsilon)$ and $(\tau',\varepsilon')$.
By linearity of~\eqref{eq:app:ddpm_update} in $(\tau,\varepsilon)$ and cancellation of the shared noise term:
\[
\Phi_t(\tau,\varepsilon,\xi_t) - \Phi_t(\tau',\varepsilon',\xi_t)
=
\frac{1}{\sqrt{\alpha_t}}(\tau-\tau')
-
\frac{\beta_t}{\sqrt{\alpha_t}\sqrt{1-\bar{\alpha}_t}}(\varepsilon-\varepsilon').
\]
Taking norms and using $\|aX\|=|a|\|X\|$ and the triangle inequality:
\[
\big\|\Phi_t(\tau,\varepsilon,\xi_t) - \Phi_t(\tau',\varepsilon',\xi_t)\big\|
\le
\frac{1}{\sqrt{\alpha_t}}\|\tau-\tau'\|
+
\frac{\beta_t}{\sqrt{\alpha_t}\sqrt{1-\bar{\alpha}_t}}\|\varepsilon-\varepsilon'\|.
\]
This is Assumption~1 with constants given in~\eqref{eq:app:ddpm_KL}.
\end{proof}

\begin{proposition}[Sampler Lipschitz constants for DDIM updates]
\label{prop:ddim_constants}
Consider the deterministic DDIM update written as an affine map of $(\tau,\varepsilon)$:
\begin{equation}
\Phi_t(\tau,\varepsilon,0)
=
a_t \tau + b_t \varepsilon,
\label{eq:app:ddim_affine}
\end{equation}
where
\begin{equation}
a_t := \sqrt{\frac{\bar{\alpha}_{t-1}}{\bar{\alpha}_t}},
\qquad
b_t := \sqrt{1-\bar{\alpha}_{t-1}} - \sqrt{\frac{\bar{\alpha}_{t-1}}{\bar{\alpha}_t}}\,\sqrt{1-\bar{\alpha}_t}.
\label{eq:app:ddim_ab}
\end{equation}
Then Assumption~1 holds with
\begin{equation}
K_t = |a_t|,
\qquad
L_t' = |b_t|.
\label{eq:app:ddim_KL}
\end{equation}
The same $(K_t,L_t')$ apply to stochastic DDIM parameterizations with an additional additive noise term (paired chains share the same draw), since the affine part in $(\tau,\varepsilon)$ is unchanged.
\end{proposition}

\begin{proof}
For deterministic DDIM, $\xi_t\equiv 0$ and the update is exactly~\eqref{eq:app:ddim_affine}.
Thus,
\[
\Phi_t(\tau,\varepsilon,0)-\Phi_t(\tau',\varepsilon',0)
=
a_t(\tau-\tau') + b_t(\varepsilon-\varepsilon').
\]
Taking Frobenius norms and using the triangle inequality:
\[
\|\Phi_t(\tau,\varepsilon,0)-\Phi_t(\tau',\varepsilon',0)\|
\le
|a_t|\,\|\tau-\tau'\| + |b_t|\,\|\varepsilon-\varepsilon'\|.
\]
This is Assumption~1 with constants~\eqref{eq:app:ddim_KL}.
For stochastic DDIM variants, the update takes the form $\Phi_t(\tau,\varepsilon,\xi_t)=a_t\tau+b_t\varepsilon+\tilde{\sigma}_t \xi_t$.
Holding $\xi_t$ fixed, the noise term cancels when subtracting paired updates, so the same bound holds.
\end{proof}

\paragraph{Guidance.}
If a planner uses guidance, we define $\hat{\varepsilon}_t$ as the \emph{effective} prediction actually passed into $\Phi_t$ (Appendix~\ref{app:guidance_impl}).
Since DDPM/DDIM samplers are affine in their $\varepsilon$ argument, the same $(K_t,L_t')$ apply with $\varepsilon$ interpreted as the effective (post-guidance) prediction.

\subsubsection{Proof 3: Trajectory-level risk control via budgeted conformal step bounds}
\label{app:proof:global_risk}

\begin{theorem}[Trajectory-level risk control for Muninn via union bound]
\label{thm:global_risk}
Fix a deviation tolerance $\eta_{\mathrm{traj}}>0$ and risk level $\alpha\in(0,1)$.
Let $\mathcal{T}_{\mathrm{cache}}\subseteq\{1,\dots,T\}$ be the set of reuse-eligible timesteps.
Assume split-conformal calibration produces per-step envelopes $\{U_t(\cdot)\}_{t\in\mathcal{T}_{\mathrm{cache}}}$ such that for each $t\in\mathcal{T}_{\mathrm{cache}}$,
\begin{equation}
\mathbb{P}\!\left(\epsilon_t \le U_t(s_t)\right) \;\ge\; 1-\alpha_{\mathrm{step}},
\qquad
\alpha_{\mathrm{step}} := \frac{\alpha}{|\mathcal{T}_{\mathrm{cache}}|},
\label{eq:app:per_step_cov_repeat}
\end{equation}
where $\epsilon_t=\|e_t\|$ is the (potential) reuse error magnitude at step $t$ and $s_t$ is the probe score.
Run Muninn with per-step upper-bounded costs
\begin{equation}
\hat{c}_t(s_t) := \Gamma L_t U_t(s_t),
\label{eq:app:chat_def}
\end{equation}
and the budget rule that allows reuse only if the remaining budget stays nonnegative and
\begin{equation}
\sum_{t\in\mathcal{R}} \hat{c}_t(s_t) \;\le\; \eta_{\mathrm{traj}},
\label{eq:app:budget_rule}
\end{equation}
where $\mathcal{R}$ is the (random) set of reuse steps taken in that episode.
Then the resulting cached trajectory satisfies
\begin{equation}
\mathbb{P}\!\left(d(\tau_0^{\mathrm{full}},\tilde{\tau}_0) > \eta_{\mathrm{traj}}\right) \;\le\; \alpha.
\label{eq:app:global_risk_repeat}
\end{equation}
\end{theorem}

\begin{proof}
Define the failure event at timestep $t$ as
\[
\mathcal{F}_t := \{\epsilon_t > U_t(s_t)\},\qquad t\in\mathcal{T}_{\mathrm{cache}}.
\]
By~\eqref{eq:app:per_step_cov_repeat}, $\mathbb{P}(\mathcal{F}_t)\le \alpha_{\mathrm{step}}$ for each $t\in\mathcal{T}_{\mathrm{cache}}$.

Consider a single episode and condition on the event that \emph{no} failure occurs at any eligible timestep:
\[
\mathcal{G} := \bigcap_{t\in\mathcal{T}_{\mathrm{cache}}} \mathcal{F}_t^{c}.
\]
On $\mathcal{G}$, for every reuse step $t\in\mathcal{R}$ we have $\epsilon_t \le U_t(s_t)$, hence
\[
c_t(e_t) = \Gamma L_t \|e_t\| = \Gamma L_t \epsilon_t \le \Gamma L_t U_t(s_t)=\hat{c}_t(s_t).
\]
For recompute steps, the denoiser output used by the cached chain matches the Full chain at that step, so $e_t=0$ and $c_t(e_t)=0$.
Therefore, using the trajectory deviation bound Eq.~\eqref{eq:budget-form} (main paper),
\[
d(\tau_0^{\mathrm{full}},\tilde{\tau}_0)
\;\le\;
\sum_{t\in\mathcal{R}} c_t(e_t)
\;\le\;
\sum_{t\in\mathcal{R}} \hat{c}_t(s_t)
\;\le\;
\eta_{\mathrm{traj}},
\]
where the last inequality is exactly the budget rule~\eqref{eq:app:budget_rule}.
Thus, the violation event implies at least one failure event:
\[
\{d(\tau_0^{\mathrm{full}},\tilde{\tau}_0) > \eta_{\mathrm{traj}}\}
\;\subseteq\;
\bigcup_{t\in\mathcal{T}_{\mathrm{cache}}} \mathcal{F}_t.
\]
Taking probabilities and applying a union bound,
\[
\mathbb{P}\!\left(d(\tau_0^{\mathrm{full}},\tilde{\tau}_0) > \eta_{\mathrm{traj}}\right)
\le
\sum_{t\in\mathcal{T}_{\mathrm{cache}}}\mathbb{P}(\mathcal{F}_t)
\le
|\mathcal{T}_{\mathrm{cache}}|\,\alpha_{\mathrm{step}}
=
\alpha.
\]
This proves~\eqref{eq:app:global_risk_repeat}. No independence across timesteps is required.
\end{proof}

\subsubsection{Proof 4: Metric compatibility constant for the deviation metric used in experiments}
\label{app:proof:metric_compat}

\begin{lemma}[Metric compatibility for the experimental deviation metric]
\label{lem:metric_compat}
In all experiments, we use
\begin{equation}
d(\tau,\tau') \;:=\; \frac{1}{\sqrt{H}} \|\tau-\tau'\|_F,
\label{eq:app:d_metric}
\end{equation}
where $\|\cdot\|$ in the analysis denotes the Frobenius norm $\|\cdot\|_F$ on $\mathbb{R}^{H\times d}$.
Then Assumption~2 holds with $\Gamma = \tfrac{1}{\sqrt{H}}$ (tight), i.e.,
\begin{equation}
d(\tau,\tau') \;\le\; \Gamma \|\tau-\tau'\|.
\end{equation}
Moreover, since $H\ge 1$, the same inequality also holds with any larger constant (e.g., $\Gamma=1$).
\end{lemma}

\begin{proof}
By definition~\eqref{eq:app:d_metric} and $\|\cdot\|=\|\cdot\|_F$,
\[
d(\tau,\tau')
=
\frac{1}{\sqrt{H}}\|\tau-\tau'\|_F
=
\Gamma \|\tau-\tau'\|,
\qquad \Gamma=\frac{1}{\sqrt{H}}.
\]
This is exactly the metric compatibility inequality, with equality (hence the constant is tight).
Since $\frac{1}{\sqrt{H}} \le 1$ for all $H\ge 1$, the inequality also holds for any $\Gamma' \ge \Gamma$ (in particular $\Gamma'=1$).
\end{proof}

\subsection{Limitations}
Muninn is a practical and general inference-time wrapper, but it has several limitations that are important for correct use and for future work.

\paragraph{Risk guarantees are tied to the chosen metric and calibration distribution}
Muninn certifies closeness to the full-compute planner only with respect to the user-selected trajectory deviation metric $d(\cdot,\cdot)$ and the deployment distribution $\mathcal{D}$ represented by the calibration set. If the runtime distribution shifts substantially (e.g., different obstacle statistics, new robot dynamics regimes, unseen contact interactions, different goal distributions, or altered observation noise), the exchangeability assumption underlying conformal calibration may be violated and empirical risk control can degrade. In practice, this motivates periodic recalibration or domain-specific monitoring when operating far from the calibration regime.

\paragraph{The certificate bounds deviation to the teacher, not task-level safety}
Munin’s guarantee is a bound on deviation from the full-compute diffusion planner; it does not by itself guarantee collision avoidance, constraint satisfaction, or stability of the closed-loop system. If the base planner is unsafe in some contexts, Muninn will remain close to that behavior. Conversely, even small deviations can matter for discontinuous constraints, so a tight deviation guarantee is helpful but not sufficient as a standalone safety mechanism. Muninn is best viewed as a \emph{faithfulness} and \emph{risk budgeting} layer that can be combined with existing safety filters, controllers, or verification modules.

\paragraph{Conservatism from union bounds and step-wise allocation}
Our global risk bound relies on a union bound across cache-eligible timesteps, which can be conservative when step-wise failures are correlated. This conservatism can reduce achievable speedup at very small target risk levels or long diffusion horizons. Tighter sequence-level calibration (e.g., conformal methods designed for time-series or martingale settings, or learned joint envelopes over step sequences) could reduce conservatism while preserving rigorous risk control.

\paragraph{Dependence on probe informativeness and architecture access}
Muninn requires a low-cost probe $\Psi$ that is sufficiently predictive of reuse error. For some architectures, extracting a meaningful intermediate representation may require engineering access to model internals, and extremely shallow probes may be too weak to be useful. Although Muninn does not require retraining, it assumes that a partial forward pass or lightweight representation can be computed cheaply and consistently across steps.

\paragraph{Sampler/model changes require recalibration and re-derived coefficients}
Munin’s sensitivity accounting depends on the sampler update coefficients (through $L_t$) and on the probe--error relationship learned during calibration. Changing the sampler (e.g., DDPM vs.\ DDIM variants, different noise schedules), applying strong guidance, or modifying conditioning can alter both sensitivity and error statistics. While Muninn can be applied without retraining, these changes generally necessitate recalibration and, depending on the sampler, recomputation of the sensitivity terms.

\paragraph{Interaction with downstream control and replanning remains nuanced}
In receding-horizon control, the effect of bounded deviation in planned trajectories on closed-loop performance depends on replanning rate, tracker/controller robustness, actuation limits, and state estimation noise. While Muninn preserves the teacher distribution up to the calibrated tolerance, a deeper understanding of how deviation budgets translate into closed-loop stability and safety margins is an important direction for future work.

Despite these limitations, Muninn provides a practical, deployment-oriented step toward making trajectory diffusion planners real-time: it offers a training-free mechanism to reduce compute while maintaining explicit, user-controlled fidelity to the full-compute planner and exposing an interpretable runtime signal that can be integrated into larger safety and control systems.

\end{document}